\documentclass{article}
\usepackage[utf8]{inputenc}
\usepackage[T1]{fontenc}
\usepackage{main}
\usepackage{microtype}
\usepackage{graphicx}
\usepackage{times}
\usepackage{amsmath}
\usepackage{amssymb}
\usepackage{amsthm}
\usepackage{bm}
\usepackage{mathtools}
\usepackage{enumitem}
\usepackage{booktabs}
\usepackage{xcolor}
\usepackage{natbib}
\definecolor{mydarkblue}{rgb}{0,0.08,0.45}
\usepackage[colorlinks=true,linkcolor=mydarkblue,citecolor=mydarkblue,filecolor=mydarkblue,urlcolor=mydarkblue]{hyperref}
\usepackage{fancyhdr}

\usepackage{comment}
\usepackage{dsfont}
\usepackage{float}
\usepackage{multirow}
\usepackage{makecell}
\usepackage{hhline}
\usepackage{colortbl}
\usepackage[linesnumbered,ruled,vlined]{algorithm2e}
\usepackage[noend]{algorithmic}

\definecolor{algcomment}{RGB}{95,110,125}

\definecolor{yjc}{RGB}{225,0,100}
\definecolor{lxs}{RGB}{138,43,226}
\definecolor{jcg}{RGB}{100,160,0}
\definecolor{adam}{RGB}{0,0,200}
\definecolor{own_pink}{RGB}{217,25,169}
\definecolor{own_blue}{RGB}{0,100,223}

\definecolor{gred}{RGB}{250, 210, 207}
\definecolor{coolblue1}{rgb}{0.91, 0.94, 0.98}
\definecolor{coolblue2}{rgb}{0.76, 0.85, 0.94}
\definecolor{coolblue3}{rgb}{0.54, 0.72, 0.87}
\definecolor{coolblue4}{rgb}{1, 1, 1}

\usepackage{xspace}
\usepackage{cleveref}
\usepackage[]{multicol, multirow}
\usepackage[most]{tcolorbox}
\usepackage{etoc}

\tcbuselibrary{listingsutf8}
\tcbuselibrary{listingsutf8,breakable}

\newtcolorbox[auto counter]{observation}[1][]{
  colback=black!5!white,
  colframe=black!70!white,
  fonttitle=\bfseries,
  title=Observation~\thetcbcounter,
  enhanced,
  boxrule=0.6pt,
  left=1mm,right=1mm,top=1mm,bottom=1mm,
  #1
}

\newtcolorbox[auto counter]{takeaway}[1][]{
  colback=teal!3!white,
  colframe=teal!55!black,
  fonttitle=\bfseries,
  title=Takeaway~\thetcbcounter,
  enhanced,
  boxrule=0.5pt,
  left=1mm,right=1mm,top=1mm,bottom=1mm,
  #1
}

\newtcolorbox[auto counter]{practicalguidance}[1][]{
  colback=cyan!3!white,
  colframe=cyan!60!black,
  fonttitle=\bfseries,
  title=Practical Guidance~\thetcbcounter,
  enhanced,
  boxrule=0.5pt,
  left=1mm,right=1mm,top=1mm,bottom=1mm,
  #1
}

\newtcolorbox[auto counter]{discussion}[1][]{
  colback=violet!4!white,
  colframe=violet!60!black,
  fonttitle=\bfseries,
  title=Discussion~\thetcbcounter,
  enhanced,
  boxrule=0.5pt,
  left=1mm,right=1mm,top=1mm,bottom=1mm,
  #1
}

\newcommand{\defn}{\coloneqq}

\newcommand{\ror}{\sigma} 
\newcommand{\unb}{\cU} 

\newcommand{\LRQFTRL}{{\sf L-Robust-Q-FTRL}\xspace}
\newcommand{\OLRQFTRL}{{\sf Online-L-Robust-Q-FTRL}\xspace}

\newcommand{\cA}{\mathcal{A}}

\newcommand{\cD}{\mathcal{D}}

\newcommand{\cP}{\mathcal{P}}

\newcommand{\cS}{{\mathcal{S}}}

\newcommand{\cU}{\mathcal{U}}

\newcommand{\cX}{\mathcal{X}}
\newcommand{\cY}{\mathcal{Y}}

\newcommand{\mymid}{\,|\,} 

\usepackage{scalerel,stackengine}
\stackMath
\newcommand\reallywidehat[1]{%
\savestack{\tmpbox}{\stretchto{%
  \scaleto{%
    \scalerel*[\widthof{\ensuremath{#1}}]{\kern-.6pt\bigwedge\kern-.6pt}%
    {\rule[-\textheight/2]{1ex}{\textheight}}
  }{\textheight}%
}{0.5ex}}%
\stackon[1pt]{#1}{\tmpbox}%
}
\newcommand\reallywidecheck[1]{%
\savestack{\tmpbox}{\stretchto{%
  \scaleto{
    \scalerel*[\widthof{\ensuremath{#1}}]{\kern-.6pt\bigwedge\kern-.6pt}%
    {\rule[-\textheight/2]{1ex}{\textheight}}
  }{\textheight}%
}{0.5ex}}%
\stackon[1pt]{#1}{\scalebox{-1}{\tmpbox}}%
}

\allowdisplaybreaks[4]

\definecolor{algogray}{RGB}{130,130,130}
\definecolor{mypurple}{RGB}{120,40,180}
\definecolor{thmblue}{RGB}{40,80,180}
\definecolor{deforange}{RGB}{220,130,40}
\usepackage{etoolbox}

\tcbset{themedboxbase/.style={
  enhanced, breakable,
  boxrule=0.35pt, arc=1mm,
  coltitle=black, fonttitle=\small\sffamily\bfseries,
  attach boxed title to top left={xshift=4mm,yshift*=-1.2mm},
  boxsep=1.5mm, top=1.5mm, bottom=1.5mm, left=1mm, right=1mm,
  before skip=10pt, after skip=10pt
}}

\newtcolorbox{algorithmbox}[1]{%
  themedboxbase,
  colframe=algogray,
  colbacktitle=algogray!20!white,
  colback=algogray!5!white,
  boxed title style={sharp corners, boxrule=0pt,
    top=1pt, bottom=0.5pt, left=4mm, right=4mm,
    borderline={0.5pt}{0pt}{algogray!40!white}},
  title={\textbf{#1}}
}

\tcbset{thmboxstyle/.style={
  themedboxbase,
  colframe=mypurple,
  colbacktitle=mypurple!15!white,
  colback=mypurple!5!white,
  boxed title style={sharp corners, boxrule=0pt,
    top=1pt, bottom=0.5pt, left=4mm, right=4mm,
    borderline={0.5pt}{0pt}{mypurple!40!white}}
}}

\tcbset{theoremboxstyle/.style={
  themedboxbase,
  colframe=thmblue,
  colbacktitle=thmblue!15!white,
  colback=thmblue!5!white,
  boxed title style={sharp corners, boxrule=0pt,
    top=1pt, bottom=0.5pt, left=4mm, right=4mm,
    borderline={0.5pt}{0pt}{thmblue!40!white}}
}}

\tcbset{defboxstyle/.style={
  themedboxbase,
  colframe=deforange,
  colbacktitle=deforange!15!white,
  colback=deforange!5!white,
  boxed title style={sharp corners, boxrule=0pt,
    top=1pt, bottom=0.5pt, left=4mm, right=4mm,
    borderline={0.5pt}{0pt}{deforange!40!white}}
}}

\newtcolorbox{responsebox}[2][]{
    breakable, enhanced,
    colback=cyan!5, colframe=cyan!50!blue,
    coltext=black, coltitle=white,
    fonttitle=\bfseries\rmfamily,
    arc=1mm, boxrule=1pt,
    width=0.95\linewidth, center,
    title=#2, #1
}

\theoremstyle{plain}
\newtheorem{lemma}{Lemma}

\newtheorem{theorem}{\textbf{Theorem}}

\newtheorem{definition}{Definition}
\newtheorem{corollary}{Corollary}[theorem]
\theoremstyle{remark}

\renewenvironment{lemma}[1][]{%
  \refstepcounter{lemma}%
  \begin{tcolorbox}[thmboxstyle, title={\textbf{Lemma~\thelemma\ifblank{#1}{}{\quad(#1)}}}]%
}{\end{tcolorbox}}

\renewenvironment{definition}[1][]{%
  \refstepcounter{definition}%
  \begin{tcolorbox}[defboxstyle, title={\textbf{Definition~\thedefinition\ifblank{#1}{}{\quad(#1)}}}]%
}{\end{tcolorbox}}

\title{Taming the Curses of Multiagency in Robust Markov Games with Large State Space through Linear Function Approximation}

\author{
    \textbf{Jingchu Gai}\thanks{Machine Learning Department, Carnegie Mellon University.} \quad
    \textbf{Laixi Shi}\thanks{Department of Electrical and Computer Engineering, Johns Hopkins University, MD 21286, USA.} \\
    CMU Machine Learning Department \quad Johns Hopkins University \\
    \texttt{jgai@andrew.cmu.edu/laixis@jhu.edu}
}

\begin{document}
\raggedbottom
\sloppy

\maketitle
\thispagestyle{fancy}
\fancyhead{}
\lhead{\includegraphics[height=0.5cm]{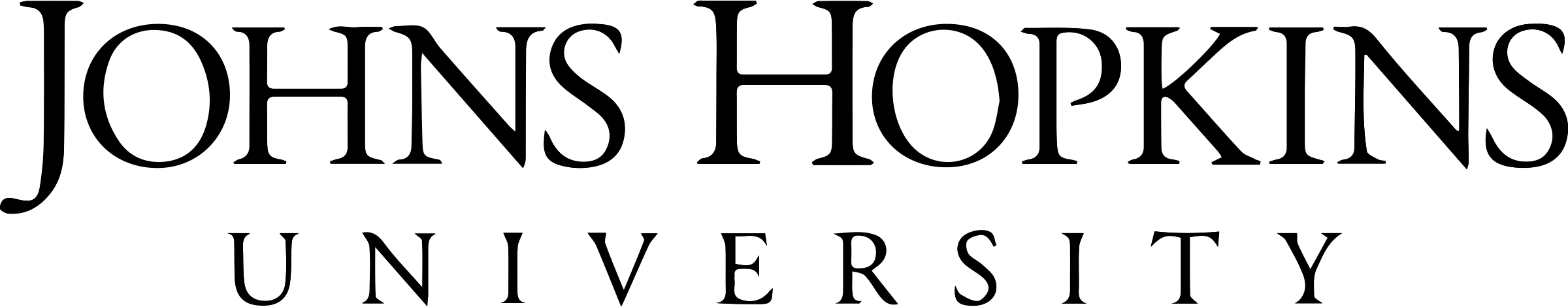}}
\rhead{%
  \raisebox{-0.1cm}{\includegraphics[height=0.8cm]{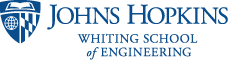}}%
}
\fancyfoot{}
\cfoot{\thepage}
\renewcommand{\headrulewidth}{0pt}
\renewcommand{\footrulewidth}{0pt}
\setlength{\headheight}{12pt}
\addtolength{\topmargin}{00pt}
\setlength{\headsep}{3mm}
\addtolength{\topmargin}{-0.8cm}
\pagestyle{plain}

\begin{abstract}
    Multi-agent reinforcement learning (MARL) holds great potential but faces robustness challenges due to environmental uncertainty. To address this, distributionally robust Markov games (RMGs) optimize worst-case performance when the environment deviates from the nominal model within a uncertainty set. Beyond robustness, an equally urgent goal for MARL is data efficiency---sampling from vast state and action spaces that grow exponentially with the number of agents potentially leads to the curse of multiagency. However, current provably data-efficient algorithms for RMGs are limited to tabular settings with finite state and action spaces, which are only computationally manageable for small-scale problems, leaving RMGs with large-scale (or infinite) state spaces largely unexplored. The only existing work beyond tabular settings focuses on linear function approximation (LFA) for a restrictive class of RMGs using vanish minimal value assumption and still suffers from sample complexity with the curse of multiagency. In this work, we focuses on general RMGs with LFA. For uncertainty sets defined by total variation distance, we develop provably data-efficient algorithms that break the curse of multiagency in both the generative model setting and a newly proposed online interactive setting. To our knowledge, our results are the first to break the curse of multiagency of sample complexity for RMGs with large (possibly infinite) state spaces, regardless of the uncertainty set construction.
\end{abstract}

\setcounter{tocdepth}{2}
\tableofcontents

\section{Introduction}

Multi-agent reinforcement learning (MARL), formulated as Markov games (MGs) for strategic games, provides a powerful framework for sequential decision-making problems involving multiple agents interacting in a shared environment. MARL has demonstrated strong empirical success across a wide range of domains, including fine-tuning multi-LLM agent systems \cite{park2025maporl}, autonomous driving \citep{zhou2020smarts}, games \citep{silver2017mastering}, and ecosystem protection \citep{fang2015security}. Due to the inherently conflicting objectives among agents, solution concepts for MGs are typically defined in terms of equilibrium notions: joint strategies under which no rational agent has an incentive to deviate unilaterally, assuming all other agents keep their strategies fixed. Widely studied examples include Nash equilibria (NE) \citep{nash1951non,shapley1953stochastic}, correlated equilibria (CE), and coarse correlated equilibria (CCE) \citep{aumann1987correlated,moulin1978strategically}.

Robustness becomes a central challenge in MARL, as learned solutions can be highly sensitive to subtle perturbations in human behavior and environmental dynamics, potentially leading to substantial degradation in both individual-agent performance and overall return \citep{balaji2019deepracer,zhang2020robust,zeng2022resilience,yeh2021sustainbench}. This motivates the development of algorithms that remain reliable under uncertainty—particularly in high-stakes or safety-critical applications—and has spurred growing interest in robust MARL through the lens of distributionally robust Markov games (RMGs) \citep{kardecs2011discounted,zhang2020robust,shi2024sample}. Unlike standard MGs, which assume a fixed environment, RMGs explicitly pursue robustness by optimizing worst-case performance over environments within a prescribed uncertainty set.

Beyond robustness, an equally critical goal in MARL is data efficiency. In extensive real-world applications, agents must sample from environments with vast state and action spaces. This challenge is further amplified in MARL by the exponential growth of the joint action space as the number of agents increases, leading to potentially exploding sample complexity—a phenomenon known as the {\em curse of multiagency}. Consequently, substantial research effort has focused on designing provably sample-efficient algorithms not only for RMGs \citep{blanchet2024double,shi2024breakcursemultiagencyrobust} but also for broader MARL problems \citep{wang2023breaking,jin2021v}.

However, current sample-efficient algorithms are often limited to tabular RMGs with finite state and action spaces. In practice, (multi-agent) RL problems inherently involve large-scale or even continuous and infinite-dimensional state spaces—e.g., transportation with autonomous vehicles \cite{hua2025multi}, robotics and swarm manipulation \cite{orr2023multi-new}, water and energy systems \cite{jia2025recent}, and financial markets \cite{mohl2025jaxmarl}. Scaling learning algorithms to such large-scale state spaces poses significant data and computational challenges. To address this, function approximation—such as deep neural networks—is widely employed in RL \citep{silver2016mastering,vinyals2019grandmaster,wang2023breaking,cui2023breaking}. Nevertheless, incorporating function approximation into provably sample-efficient solutions for RMGs remains largely unexplored. The only existing work \cite{zheng2025distributionally} studies linear function approximation in an online setting, focusing on a restrictive class of RMGs \citep{wang2023breaking,jin2019provablyefficientreinforcementlearning} and requiring an additional vanishing minimal value assumption to obtain meaningful guarantees; moreover, its sample complexity still suffers from the curse of multiagency. This raises an open question:
\begin{center}
    {\em  Can the curse of multiagency be tamed in robust Markov games with large-scale state spaces? }
\end{center}
To address this gap, to the best of our knowledge, this work provides the first provably sample-efficient algorithm for RMGs with large-scale (or infinite) state spaces that avoids the curse of multiagency. In particular, we focus on linear function approximation. Let $S$ and $A_1, \ldots, A_n$ denote the cardinality of the state space and the action spaces of $n$ agents, respectively; let $d$ denote the feature dimension of the linear function class for each state-action pair (see \eqref{eq:linaer-mg-assumption} for details).
Let $H$ denote the horizon (number of steps per episode), and let $T$ denote the total number of episodes in the online learning process. The main contributions are summarized as below:

\begin{itemize}

\item We focus on \emph{distributionally robust linear Markov games} (R-LMGs), i.e., RMGs with linear function approximation. Under the data collection mechanism with a generative model, we develop a provable algorithm that achieves an $\varepsilon$-approximate CCE for R-LMGs using total variation uncertainty set, with sample complexity at most $\tilde{\mathcal{O}}(H^9 d^3/\epsilon^4)$ (cf.~Theorem~\ref{thm:main_generative_model}). To the best of our knowledge, this is the first provable sample-complexity guarantee for R-LMGs under the generative model setting, irrespective of the uncertainty-set formulation; moreover, it effectively overcomes the curse of multi-agency. Technically, unlike tabular RMGs, large-scale (possibly infinite) state spaces make exhaustive sampling over all state--action pairs infeasible. We therefore employ a infinite-to-finite reduction to sample a carefully chosen finite subset and generalize to the full state space.

\item Beyond the generative model setting, we consider a more practical data-collection regime—an online interactive protocol in which agents sample Markovian trajectories from transition kernels within the uncertainty set of R-LMGs, rather than solely from the nominal kernel. In this setting, we propose a sample-efficient algorithm that attains sublinear regret $\tilde{\mathcal{O}}\!\left(d\,\max_{i\in[n]} A_i\, H^2 \sqrt{T}\right)$ (cf.~Theorem~\ref{thm:main_online}), thereby breaking the curse of multiagency. Technically, since the true robust value function is unknown, we cannot sample directly from the worst-case transition kernel and must instead approximate the adversarial environment. To this end, beyond the optimistic value estimates commonly used in standard linear MGs for policy updates, we introduce a new sampling strategy that leverages an additional sequence of pessimistic robust value estimates to approximate the adversarial model, combined with a hybrid sampling scheme to establish our guarantees.

\end{itemize}
\begin{table}[t]
\begin{center}
\resizebox{\textwidth}{!}{
\begin{tabular}{c|c|c|c|c|c}
\hline \hline

Data Mechanism & Algorithm & Robust MG Model  & Uncertainty Set & \makecell{Break \\ Curse} & Sample Complexity/ $\mathrm{Regret}(T)$ \tabularnewline
\hline \hline

\multirow{8}{*}{Generative Model} 

& P$^2$MPO \vphantom{$\frac{1^{7}}{1^{7^{7}}}$} & \multirow{2}{*}{Tabular} & \multirow{2}{*}{$(s,\boldsymbol{a})$-rectangularity} & \multirow{2}{*}{No} & \multirow{2}{*}{$S^4 \left(\prod_{i=1}^n A_i\right)^3 \frac{H^4}{\varepsilon^2}$ } \tabularnewline
& \citep{blanchet2024double} & & & & \tabularnewline
\cline{2-6}

& DR-NVI \vphantom{$\frac{1^{7}}{1^{7^{7}}}$} & \multirow{2}{*}{Tabular} & \multirow{2}{*}{$(s,\boldsymbol{a})$-rectangularity} & \multirow{2}{*}{No} & \multirow{2}{*}{$\frac{SH^3 \prod_{i=1}^n A_i}{\varepsilon^2} \min \Big\{H, \frac{1}{\min_i \sigma_i}\Big\}$} \tabularnewline
& \citep{shi2024sample} & & & & \tabularnewline
\cline{2-6}

& Robust-Q-FTRL \vphantom{$\frac{1^{7}}{1^{7^{7}}}$} & \multirow{2}{*}{Tabular} & fictitious & \multirow{2}{*}{\textcolor{red}{Yes}} & \multirow{2}{*}{
$\frac{SH^6 \sum_{i=1}^n A_i}{\varepsilon^4} \min \Big\{H, \frac{1}{\min_i \sigma_i}\Big\}$} \tabularnewline
& \citep{shi2024breakcursemultiagencyrobust} & & $(s,a_i)$-rectangularity & & \tabularnewline
\hhline{~-----}

& \cellcolor{cyan!10} \textbf{This Work} \vphantom{$\frac{1^{7}}{1^{7^{7}}}$} & \cellcolor{cyan!10} \textcolor{red}{Linear}, & \cellcolor{cyan!10} fictitious & \cellcolor{cyan!10} & \cellcolor{cyan!10} \tabularnewline
& \cellcolor{cyan!10} (Generative) & \cellcolor{cyan!10} $d = S \max_{i \in[n]} A_i$ & \cellcolor{cyan!10} $(s,a_i)$-rectangularity & \cellcolor{cyan!10}\multirow{-2}{*}{\textcolor{red}{Yes}} & \cellcolor{cyan!10} \multirow{-2}{*}[0.2em]{$\frac{H^9d^3}{\varepsilon^4}$} \tabularnewline

\hline \hline

\multirow{6}{*}{Online Interaction} 

& \multirow{2}{*}{\makecell[c]{DRMG \\ \citep{farhat2025online}}} \vphantom{$\frac{1^{7}}{1^{7^{7}}}$} & \multirow{2}{*}{Tabular} & \multirow{2}{*}{$(s,\boldsymbol{a})$-rectangularity} & \multirow{2}{*}{No} & \multirow{2}{*}{$\tilde{\mathcal{O}} \big( \sqrt{\min \{ \sigma_{\min}^{-1}, H \} H S  (\prod_{i\in[n]} A_i)  T} \big),$} \tabularnewline
& & & & & \tabularnewline
\cline{2-6}

& DR-CCE-LSI \vphantom{$\frac{1^{7}}{1^{7^{7}}}$} & \textcolor{red}{Linear}, & \multirow{2}{*}{$(s,\boldsymbol{a})$-rectangularity} & \multirow{2}{*}{No} & \multirow{2}{*}{ $ d_{\mathsf{curse}} H \min\!\big\{ H, \tfrac{1}{\min\{\sigma_i\}} \big\} \,  \sqrt{T}$  } \tabularnewline
& \citep{zheng2025distributionally} & $d_{\mathsf{curse}} = S\prod_{i\in[n]}A_i$ & & & \tabularnewline
\hhline{~-----}

& \cellcolor{cyan!10} \textbf{This Work} \vphantom{$\frac{1^{7}}{1^{7^{7}}}$} & \cellcolor{cyan!10} \textcolor{red}{Linear}, & \cellcolor{cyan!10} fictitious & \cellcolor{cyan!10} & \cellcolor{cyan!10} \tabularnewline
& \cellcolor{cyan!10} (Online) & \cellcolor{cyan!10} $d = S \max_{i \in[n]} A_i$ & \cellcolor{cyan!10} $(s,a_i)$-rectangularity & \cellcolor{cyan!10}\multirow{-2}{*}{\textcolor{red}{Yes}} & \cellcolor{cyan!10} \multirow{-2}{*}[0.2em]{$d \max_{i\in[n]}A_i\, H^2    \sqrt{T}$} \tabularnewline

\hline \hline
\end{tabular}
}
\end{center}
\caption{The existing results of sample complexity in the generative model setting (or regret in the online setting) for finding an $\varepsilon$-approximate equilibrium in finite-horizon, multi-agent, general-sum robust Markov games (RMGs), omitting logarithmic factors. To our knowledge, this work provides the first algorithms that break the curse of multiagency for large (possibly infinite) state spaces using linear function approximation.}
\label{tab:prior-work-double-line}
\end{table}

\section{Related work}

\paragraph{Distributionally Robust RL.}
Standard reinforcement learning (RL) methods can be highly sensitive to
perturbations of environment dynmaics (transition kernels and reward function)
\citep{zhang2020robust,mahmood2018benchmarking}. To enhance robustness, 
distributionally robust RL (robust RL) seeks policies that optimize worst-case
performance over a prescribed uncertainty set by incorporating
distributionally robust optimization (DRO) into sequential decision making.
DRO has also been extensively studied in supervised learning
\citet{rahimian2019distributionally,gao2020finite,bertsimas2018data,
duchi2018learning,blanchet2019quantifying,gao2023distributionally}.

Distributionally robust RL was first studied in the single-agent setting
\citep{nilim2005robust,iyengar2005robust,badrinath2021robust,zhou2021finite,
shi2022distributionally,shi2023curious}. More recently, DRO formulations have been incorporated to multi-agent RL and formulates robust Markov
games (RMGs) \citet{zhang2020robust,kardecs2011discounted,ma2023decentralized,
blanchet2023double,shi2024sample}. Existing sample-complexity guarantees
for algorithms of RMGs largely focus on tabular settings with finite state
and action spaces \citet{ma2023decentralized,blanchet2024double,
shi2024sample,shi2024breakcursemultiagencyrobust}, while comparatively few
results address large-scale or infinite state spaces. To our knowledge, the only existing work \cite{zheng2025distributionally} studied RMGs under linear function approximation. \cite{zheng2025distributionally} focuses on a restrictive class of RMGs \citep{zheng2025distributionally,lu2024distributionally} that requires an additional vanishing minimal value assumption, which is a more restrictive subclass of RMGs than the setting considered in this work; moreover, its sample complexity still suffers from the curse of multiagency when transferred to the tabular setting. In this work, we focus on general RMGs under independent per-agent linear function approximation modeling with no further assumption on the structure of the value function, and achieve sample complexity that breaks the curse of multi-agency. See the discussion after Theorem~\ref{thm:main_online} for a detailed comparison.



\paragraph{Multi-Agent RL with Function Approximation.}
Linear function approximation (LFA) was first studied in the single-agent setting \citep{jin2020provably,yang2020reinforcement}. Building on these techniques, \citet{xie2020learning} and \citet{chen2024unifiedalgorithmsrldecisionestimation} investigated Markov games under centralized LFA, where a central function class models value functions across the joint actions of all agents. This setting makes circumventing the curse of multi-agency challenging. To address this issue, \citet{wang2023breaking,cui2023breaking} introduced independent per-agent linear function classes and correspondingly developed provably efficient algorithms that break the curse of multi-agency; further, \citet{dai2024refinedsamplecomplexitymarkov} obtained optimal dependence on the action-set size via sharper analysis. Beyond LFA, several works have studied Markov games with general function approximation in both centralized settings \citep{huang2022towards,ni2022representation,xiong2022self,chen2022almost,zhan2023decentralized,jin2022power} and per-agent settings \citep{wang2023breaking}. In this paper, we extend independent per-agent LFA to a popular robust counterpart of standard MGs---robust linear MGs (R-LMGs). Compared to standard MGs, in both generative and online settings, we must handle the additional statistical error induced by the high nonlinearity of R-LMGs with respect to the nominal transition kernel. In the online interactive setting in particular, additional challenges arise from the need to sample from the approximate adversarial environment and from the fact that the LFA assumption applies solely to the nominal kernel; kernels within the uncertainty set do not necessarily preserve this linear structure. To address this, we introduce a \emph{hybrid sampling strategy} that uses a sequence of pessimistic robust value estimates, allowing us to estimate the nominal transition kernel satisfying LFA through ridge regression and then apply the pessimistic estimates to approximate the adversarial model for subsequent robust objective estimation.

\paragraph{Breaking the Curse of Multiagency in Multi-Agent RL.}
In multi-agent reinforcement learning (MARL), the joint action space grows
exponentially with the number of agents, making it crucial to develop
algorithms whose sample complexity does not scale exponentially in the
number of agents---often referred to as breaking the curse of multi-agency. This challenge is ubiquitous in MARL and has attracted substantial interest.
\citet{song2021can,rubinstein2017settling,bai2020provable} show that learning
a Nash equilibrium in general Markov games can require sample complexity
exponential in the number of agents. This has motivated work on alternative
solution concepts, such as correlated equilibrium (CE) and coarse correlated
equilibrium (CCE), for which non-exponential guarantees are possible. For standard MGs, \citet{daskalakis2023complexity} provide a complexity
bound of $\widetilde{O}(H^{11}S^{3}\max_{i\in[n]}A_i/\epsilon^3)$ for learning
a CCE, and \citet{jin2021v,song2021can,li2022minimax} also propose algorithms with non-exponential sample complexity, both for tabular MGs. Beyond the tabular setting,
\citet{cui2023breaking} and \citet{wang2023breaking} develop algorithms that
break the curse of multi-agency under independent per-agent LFA. Beyond standard Markov games,
\citet{song2021can,alatur2024independentpolicymirrordescent,
dong2024convergencenashequilibriumnoregret} study Markov potential games and
give algorithms that learn Nash equilibria with sample complexity that does
not scale exponentially with the number of agents.

For robust Markov games, breaking the curse remains largely underexplored. The only existing results that breaks the curse of multi-agency 
\citet{shi2024breakcursemultiagencyrobust,jiao2024minimax} are restricted to tabular RMGs with R-contamination or TV distance. To the best of our knowledge, this is the first work to break the curse of multi-agency for large-scale (continuous) state space, in both generative model and online data collection setting.

\section{Problem Formulation}
\subsection{Background: standard linear Markov games}
\noindent\textbf{Finite-horizon multi-agent general-sum Markov games (MGs).} We consider an episodic finite-horizon general-sum MG represented by the tuple $(H, \mathcal{S}, \{\mathcal{A}_i\}_{i \in [n]}, P, \{r_i\}_{i \in [n]})$, where $H$ denotes the length of each episode and $\mathcal{S}$ represents the state space. The joint action space for all agents is defined as $\mathcal{A} = \mathcal{A}_1 \times \cdots \times \mathcal{A}_n$, where, for each $i \in [n]$, $\mathcal{A}_i$ is a finite action set of agent $i$ with cardinality $A_i$. We use $\mathbf{a} \in \mathcal{A}$ to represent the joint action of all agents, and $\mathbf{a}_{-i}$ to denote the joint action of all agents except for agent $i$. 
Additionally, $P = \{P_h\}_{h \in [H]}$ is a collection of transition kernels, where $P_h(\cdot \mid s, \mathbf{a})$ gives the distribution of the next state when agents take action $\mathbf{a}$ at state $s$ and step $h$. For each agent $i$, $r_i = \{r_{i,h}\}_{h \in [H]}$ represents the reward functions, where agent $i$ receives $r_{i,h}(s, \mathbf{a}) \in [0, 1]$ when all agents take action $\mathbf{a}$ at state $s$ and step $h$.

To continue, we focus on Markov policies, where the action selection rule depends only on the current state, independent of the previous trajectory. Specifically, for each agent $i \in [n]$, actions are chosen according to the policy $\pi_i = \{\pi_{i,h} : \mathcal{S} \rightarrow \Delta(\mathcal{A}_i)\}_{1 \leq h \leq H}$. Here, $\pi_{i,h}(a_i \mymid s)$ denotes the probability of selecting action $a_i \in \mathcal{A}_i$ in state $s$ at time step $h$.The joint policy for all agents is represented as $\pi = \pi_1 \times \pi_2 \times \cdots \times \pi_n : \mathcal{S} \times [H] \rightarrow \Delta(\mathcal{A})$. Additionally, we define $\pi_{-i} = \otimes_{j \neq i} \pi_j$ to denote the joint policy of all agents except for agent $i$.
For any given joint policy $\pi$ and transition kernel $P$ in a Markov game, the long-term cumulative reward for the $i$-th agent can be described by the value function $V_{i,h}^{\pi, P}: \mathcal{S} \rightarrow \mathbb{R}$ (or the Q-function $Q_{i,h}^{\pi, P}: \mathcal{S} \times \mathcal{A} \rightarrow \mathbb{R}$) as follows: for all $(h, s, \mathbf{a}) \in [H] \times \mathcal{S} \times \mathcal{A}$,
\begin{equation}
\label{eq:value-function-defn}
\begin{aligned}
&\textstyle Q_{i,h}^{\pi, P}(s, \mathbf{a}):=\mathbb{E}_{\pi,P}\big[\sum_{t=h}^{H}  r_{i,t}\big(s_{t}, \mathbf{a}_{t}\big)\mid s_{h}=s, \mathbf{a}_h = \mathbf{a}\big],\quad \textstyle V_{i,h}^{\pi, P}(s) :=\mathbb{E}_{\pi,
P}\big[\sum_{t=h}^{H} r_{i,t}\big(s_{t}, \mathbf{a}_{t}\big)\mid s_{h}=s\big]. 
\end{aligned}
\end{equation}
Here, the expectation is taken over the trajectory $\{(s_t, \mathbf{a}_t)\}_{h \leq t \leq H}$ generated by following the joint policy $\pi$ under the transition kernel $P$.

\paragraph{Linear Markov games: MGs with per-agent independent linear function approximation.} In this work, we focus on MGs with linear function approximation (LFA), abbreviated by linear MGs later, assuming that each agents' transitions and reward functions can be described by independent linear functions, as proposed by \citet{cui2023breaking,wang2023breaking,dai2024refinedsamplecomplexitymarkov}.
Specifically, in any linear MG with transition kernel $P =P^0$, for each agent $i \in [n]$, there exists a known $d$-dimensional feature mapping $\phi_i: \mathcal{S} \times \mathcal{A}_i \rightarrow \mathbb{R}^d$, such that for all  joint policy $\pi: [H] \times \mathcal{S} \mapsto \Delta(\mathcal{A})$, as well as   $(i,h,s,a_i) \in [n] \times [H] \times \mathcal{S}\times\mathcal{A}_i$, we have
\begin{equation}    
  \label{eq:linaer-mg-assumption}
\begin{aligned}
    &\mathbb{E}_{a_{-i} \sim \pi_{-i}(s)}\left[P_h^0(\cdot \mymid s, (a_i,\mathbf{a}_{-i}))\right] = \phi_i(s,a_i)^\top \mu_{i,h}^{\pi_{-i}},\quad  \mathbb{E}_{a_{-i} \sim \pi_{-i}(s)}\left[r_{i,h}(s,(a_i,\mathbf{a}_{-i}))\right] = \phi_i(s,a_i)^\top \theta_{i,h}^{\pi_{-i}},
\end{aligned}
\end{equation}
where $\mu_{i,h}^{\pi_{-i}} = (\mu_{i,h}^{\pi_{-i},(1)}, \ldots, \mu_{i,h}^{\pi_{-i},(d)})^\top$ consists of $d$-rows, where each row $\mu_{i,h}^{\pi_{-i},(j)}$ is a unknown signed measure over $\mathcal{S}$. $\theta_{i,h}^{\pi_{-i}} \in \mathbb{R}^d$ is the unknown reward parameter. We assume the standard normalization bounds:  $\lVert \theta_{i,h}^{\pi_{-i}} \rVert_2 \leq d$ and $\lVert \mu_{i,h}^{\pi_{-i}}(s) \rVert_2 \leq \sqrt{d}$ for all $(s,h) \in \mathcal{S} \times [H]$.
Here, $\mu_{i,h}^{\pi_{-i}}(s) \in \mathbb{R}^d$ denotes the feature vector corresponding to state $s$ over time step $h$.

\paragraph{Remark.} The above assumption on the linear structure (cf.~\eqref{eq:linaer-mg-assumption}) is strictly weaker than that in the closely related prior work \cite[Assumption 3.1]{zheng2025distributionally}, which also studies robust linear MGs but adopts a centralized linear structure. Specifically, \citet{zheng2025distributionally} requires the transition kernel $P_h(s^\prime \mymid s,\mathbf{a})$ for each state-joint action $(s,\mathbf{a})$ to admit a linear representation proportional to a feature vector $\phi(s,\mathbf{a})$.

\subsection{Distributionally robust linear MGs}
To consider robustness for linear MGs, we focus on a distributionally robust linear Markov game (R-LMG), characterized by the tuple $\{ \mathcal{S}, \{\mathcal{A}_i\}_{i \in [n]}, \{\mathcal{U}_{\rho}^{\sigma_i}(P^0, \cdot)\}_{ i \in[n]}, r, H \}$. The key difference between an R-LMG and a standard linear MG is that, in an R-LMG, the transition kernel for each agent $i$ can vary within its own uncertainty set $\mathcal{U}_{\rho}^{\sigma_i}(P^0,\cdot)$ around a nominal transition kernel $P^0$. Note that the nominal transition $P^0$ and $r$ satisfy the linear function approximation described in \eqref{eq:linaer-mg-assumption}, while other transitions within the uncertainty set $\mathcal{U}_{\rho}^{\sigma_i}(P^0, \cdot)$ do not need to satisfy LFA or any other modeling assumptions. We shall specify the uncertainty set $\mathcal{U}_{\rho}^{\sigma_i}(P^0,\cdot)$ with $\rho$ and $\sigma_i$ shortly. 
\paragraph{Fictitious uncertainty set.} Robust MGs (and also robust linear MGs) form a broad class due to the flexibility
in specifying uncertainty sets that capture different notions of model
mismatch and shifts. For a generic uncertainty set $\mathcal{U}_{\rho}^{\sigma_i}(P^0,\cdot)$,
$\rho$ represents a divergence function that quantifies the discrepancy between transition
kernels, and $\sigma_i\ge 0$ controls the uncertainty radius for agent $i$.
Various choices of the divergence function $\rho$ can be considered, including 
but not limited to $f$-divergence family (such as total variation and Kullback-Leibler (KL) divergence), Wasserstein distance \citet{xu2023improved}, and learning-based approximation \citet{ding2023seeing}. In this work, we adopt
the fictitious uncertainty sets $\mathcal{U}_{\rho}^{\sigma_i}(P^0,\pi)$
from \citet{shi2024breakcursemultiagencyrobust}, which model each agent as addressing and responding to the aggregate uncertainty induced by all other agents, rather than decomposing uncertainty into agent-by-agent individual uncertainty. This aligns with realistic human decision-making under uncertainty, as observed in behavioral economics \citep{goeree2005regular,sandomirskiy2024narrow}. 

To specify, we firstly introduce some useful notations.  For all $(s, \mathbf{a}, h) \in \mathcal{S} \times \mathcal{A} \times [H]$, let $P_{h,s,\mathbf{a}}^0 \in \Delta(\mathcal{S})$ represent the nominal transition kernel such that, for all $s^\prime \in \mathcal{S}$, we have $P_{h,s,\mathbf{a}}^0(s^\prime) = P_h^0(s^\prime \mid s, \mathbf{a})$. Additionally, let $\mathcal{A}_{-i} := \prod_{j\neq i} A_j$. And we define the expected nominal transition kernel conditioned on agent $i$ selecting action $a_i \in \mathcal{A}_i$, while the other agents follow a conditional policy, as described below:
\begin{align}
\label{eq:basic_transition}
  \textstyle P_{h,s,a_i}^{\pi_{-i}}(s^\prime):=\sum_{\mathbf{a}_{-i}\in\mathcal{A}_{-i}}\frac{\pi_h(a_i, \mathbf{a}_{-i}\mid s)}{\pi_{i,h}(a_i\mid s)}P_{h,s,\mathbf{a}}^0(s^\prime),\quad\forall s^\prime\in\mathcal{S}.
\end{align}

Armed with the notation, for any joint policy $\pi: \mathcal{S} \times [H] \to \Delta(\mathcal{A})$ and uncertainty levels $\{\sigma_i\}_{1 \leq i \leq n}$, the uncertainty set is defined as $\{\mathcal{U}_{\rho}^{\sigma_i}(P^0, \pi)\}_{i \in [n]}$ as $\mathcal{U}_{\rho}^{\sigma_i}(P^0, \pi) = \otimes \mathcal{U}_{\rho}^{\sigma_i}(P_{h,s,a_i}^{\pi_{-i}})$, where
\begin{align}\label{eq:fictitious-uncertainty-set}
  \mathcal{U}_{\rho}^{\sigma_i}(P_{h,s,a_i}^{\pi_{-i}}) := \big\{ P \in \Delta(\mathcal{S}) : \rho(P, P_{h,s,a_i}^{\pi_{-i}}) \leq \sigma_i \big\}.
\end{align}

\noindent\textbf{Robust value functions.} In a robust linear Markov game, each agent optimizes its worst-case performance. For a joint policy $\pi: \mathcal{S} \times [H] \to \Delta(\mathcal{A})$, we define agent $i$'s robust value function $\{V_{i,h}^{\pi,\sigma_i}\}_{h\in[H]}$ and robust Q-function $\{Q_{i,h}^{\pi,\sigma_i}\}_{h\in[H]}$: for all $(i, h, s, a_i) \in [n] \times [H] \times \mathcal{S} \times \mathcal{A}_i$,
\begin{equation}
  \label{eq:value-function-defn-robust}
\begin{aligned}
     \textstyle Q_{i,h}^{\pi,\sigma_i}(s, a_i) :=  \inf_{P \in \mathcal{U}_{\rho}^{\sigma_i}(P^0,\pi)} Q_{i,h}^{\pi,P}(s, a_i), \quad \textstyle V_{i,h}^{\pi,\sigma_i}(s) := \inf_{P \in \mathcal{U}_{\rho}^{\sigma_i}(P^0,\pi)} V_{i,h}^{\pi,P}(s).
\end{aligned}
\end{equation}
Notably, unlike in \eqref{eq:value-function-defn}, the $Q$-function for agent $i$ is evaluated based solely on its individual action $a_i \in \mathcal{A}_i$, rather than the joint action $\mathbf{a} \in \mathcal{A}$ of all agents.
We further define the robust best-response policy as the corresponding
benchmark value function. For each agent $i \in [n]$, given the
policies of the other $n-1$ agents, define for all $(h,s)\in[H]\times\mathcal{S}$,
\begin{equation}
  \label{eq:defn-optimal-V}
\begin{aligned}
  \textstyle V_{i,h}^{\star,\pi_{-i},\sigma_i}(s)
  := \max_{\pi_i':\,\mathcal{S}\times[H]\to\Delta(\mathcal{A}_i)}
  V_{i,h}^{\pi_i'\times \pi_{-i},\,\sigma_i}(s).
\end{aligned}
\end{equation}
Any maximizer $\pi_i'$ is referred to as a robust best response of agent $i$.


\paragraph{Solution concepts: robust equilibria.}
Analogous to standard MGs, we introduce two types of equilibria as the learning objectives in R-LMGs, the robust variants --- robust Nash equilibrium (NE) and robust coarse correlated equilibrium (CCE):
\begin{itemize}
  \item Robust Nash equilibrium (Robust NE): A product policy $\pi=\pi_1\times\cdots\times\pi_n:
  \mathcal{S}\times[H]\to\prod_{i=1}^n\Delta(\mathcal{A}_i)$ is a robust NE if
  \begin{equation}
    \label{eq:defn-robust-Nash-E}
    V_{i,1}^{\pi,\sigma_i}(s)=V_{i,1}^{\star,\pi_{-i},\sigma_i}(s),
    \quad \forall (i,s)\in [n] \times \mathcal{S}.
  \end{equation}
In other words, fixing other agents' policy $\pi_{-i}$, agent $i$ cannot improve its worst-case outcome (i.e., the robust value function) by unilaterally changing its own policy.
   \item Robust coarse correlated equilibrium (Robust CCE): A distribution over joint product policies
   $\xi:=\{\xi_h\}_{h\in[H]}:[H]\to\Delta(\mathcal{S}\to\prod_{i\in[n]}\Delta(\mathcal{A}_i))$
   is a robust CCE if, for all $(i,s)\in[n]\times\mathcal{S}$,
  \begin{align*}
      &\textstyle \mathbb{E}_{\pi \sim \xi} \big[  V_{i,1}^{\pi, \sigma_i}(s) \big]
      \geq \max_{\pi_i^\prime: \mathcal{S} \times [H] \to \Delta(\mathcal{A}_i)} \mathbb{E}_{\pi \sim \xi} \big[  V_{i,1}^{\pi^\prime \times \pi_{-i}, \sigma_i}(s) \big].
    \end{align*}
    Intuitively, when the joint policy $\pi$ is drawn from $\xi$, no agent can increase its expected robust value by committing to an alternative policy. 
\end{itemize}


\noindent Building on the definition of the fictitious uncertainty set, the following theorem establishes the existence of both a robust NE and a robust CCE in robust linear MGs (R-LMGs). The proof is deferred to Appendix~\ref{sec:proof-eq:ne-existence}.
\begin{tcolorbox}[thmboxstyle, title={\textbf{Existence of robust NE}}]
    \begin{theorem}
    \label{eq:ne-existence}
      For any robust linear MG $\{ \mathcal{S}, \{\mathcal{A}_i\}_{1 \leq i \leq n}, \{\mathcal{U}_{\rho}^{\sigma_i}(P^0, \cdot)\}_{1 \leq i \leq n}, r, H \}$ with linear approximation (cf.~\eqref{eq:linaer-mg-assumption}) and fictitious uncertainty sets, there exists at least one robust NE, which further implies the existence of at least one robust CCE.
    \end{theorem}
\end{tcolorbox}

 This theorem extends existing robust-NE existence results in \cite{shi2024breakcursemultiagencyrobust} for robust MGs with finite state spaces to settings with possibly infinite state spaces, by combining Glicksberg’s fixed-point theorem \citep{glicksberg1952further} with additional analytic arguments.

Additionally, for any joint policy $\pi: \mathcal{S} \times [H] \to \Delta(\mathcal{A})$, the robust value function can be represented by the following equation, which we refer to as the robust Bellman equation:
\begin{align}\label{eq:robust-bellman-equation}
  \textstyle V^{\pi,\sigma_i}_{i,h}(s) = \mathbb{E}_{a_i \sim \pi_{i,h}(s)} \big[ r_{i,h}^{\pi_{-i}}(s, a_i) + \inf_{\mathcal{U}_{\rho}^{\sigma_i}(P_{h,s,a_i}^{\pi_{-i}})} P V_{i,h+1}^{\pi,\sigma_i}\big],
\end{align}
where we define $r_{i,h}^{\pi_{-i}}(s, a_i)=\mathbb{E}_{a_{-i}\sim\pi_{-i,h}(s)}\left[r_{i,h}(s,\mathbf{a})\right]$ for all $(s,a_i)\in\mathcal{S}\times\mathcal{A}_i$.

\paragraph{Specification of the uncertainty set: total variation (TV).}
In this work, we consider the total variation (TV) distance as the divergence function $\rho$, motivated by its simplicity, practical utility \citep{szita2003varepsilon, lee2021optidice, pan2023adjustable}, and potential theoretical advantages \citep{shi2023curious, shi2024sample}. Namely,  for any two distributions $P, P^\prime \in \Delta(\mathcal{S})$, the TV distance between $P$ and $P^\prime$ is given by:
\begin{align}
\textstyle \left\lVert P - P^\prime \right\rVert_{d_{\mathsf{TV}}} = \frac{1}{2}\sum_{s \in \mathcal{S}} \left| P(s) - P^\prime(s) \right|.
\end{align}
For simplicity, we denote $\mathcal{U}^{\sigma_i}(\cdot) \defn \mathcal{U}_{d_{\mathsf{TV}}}^{\sigma_i}(\cdot)$ throughout the rest of the paper.

\section{Robust Linear MGs with a Generative Model}

\label{sec:generative_model}
In both single-agent and multi-agent RL, three data-collection mechanisms are widely used: the generative model setting \citep{zhang2020model-based,li2022minimax}, the offline setting \citep{blanchet2023double}, and the online setting \citep{jin2018q}. In this section, we begin with the fundamental generative model setting, in which we assume access to a generative model or simulator that provides samples $s^\prime \sim P_h^0(\cdot \mid s, \mathbf{a})$ from the nominal environment $P^0$ for any given state, joint action, and time-step tuple $(s, \mathbf{a}, h) \in \mathcal{S} \times \mathcal{A} \times [H]$.

To learn equilibria in MARL, we focus on finding an {\em $\varepsilon$-approximate CCE} using as few samples as possible. This is motivated by the fact that computing a Nash equilibrium is generally intractable \citep{daskalakis2013complexity}, while finding an exact robust equilibrium is computationally hard and often impractical \citep{aghassi2006robust,wiesemann2013robust}. Mathematically, our goal is to find a distribution   $\xi=\{\xi_h\}_{h\in[H]}: [H] \rightarrow \Delta(\cS \rightarrow  \prod_{i\in[n]} \Delta(\cA_i))$ over policies such that 
\begin{align}
\label{eq:appximate_CCE}
	 \textstyle\max_{s\in \cS, i\in[n]} \big\{ \mathbb{E}_{\pi\sim \xi} \big[V_{i,1}^{\star,\pi_{-i}, \sigma_i }(s)\big] - \mathbb{E}_{\pi\sim \xi}\left[V_{i,1}^{\pi,\sigma_i}(s)\right] \big\} \leq \varepsilon.
\end{align}

\subsection{Algorithm Design}\label{sec:generative-algo-design}
To solve robust linear MGs (R-LMGs), we propose \LRQFTRL (outlined in Algorithm~\ref{alg:linear_generative_model}). Motivated by Robust-Q-FTRL algorithm developed for robust MGs in tabular setting \citep{shi2024breakcursemultiagencyrobust}, it incorporates designs tailored to linear function approximation (LFA) for large-scale possibly infinite state space, through an infinite-to-finite sampling design and leveraging ridge regression to estimate the transition probabilities and reward functions.
Due to the space limit, we will introduce this key module of infinite-to-finite sampling  for nominal model estimation, which differs fundamentally from the tabular robust MG setting. The details of the entire algorithm are postponed to the Appendix~\ref{sec:generative-algo-details}.

\noindent\textbf{Nominal model estimation: constructing finite set reflecting (infinite) state space.}
Unlike the tabular case, where the state space is finite and relatively small, sampling costs become prohibitive when enumerating all state-action pairs in a large-scale (or infinite) state space. To address this challenge, we employ the following lemma, a widely-used tool for linear Markov games \citep{wang2023breaking}:

\begin{lemma}
    \label{lm:generative_model_supporting_set}
    Let $\mathcal{X} \subset \mathbb{R}^d$ be a compact, full-dimensional set \footnote{Here, full-dimensional implies that the linear span of $\mathcal{X}$ covers the entire space, i.e., $\text{span}(\mathcal{X}) = \mathbb{R}^d$. Namely, there exists a distribution $\rho$ over $\mathcal{X}$ so that the second moment matrix $\Sigma = \mathbb{E}_{x\sim \rho}[xx^\top] $ is positive definite and invertible. See \citet{boyd2004convex}.} There exists a distribution $\rho$ supported on at most $d(d+1)/2$ points in $\mathcal{X}$ such that, defining $\Sigma = \mathbb{E}_{x \sim \rho} [xx^\top]$, we have $\Sigma$ is invertible and  $\lVert x \rVert^2_{\Sigma^{-1}} \leq d$ for all $x \in \mathcal{X}$.
\end{lemma}
  The proof of Lemma~\ref{lm:generative_model_supporting_set} can be found in \citet{lattimore2020bandit}. Without loss of generality, we assume that the feature set of each agent $i$, denoted by $\mathcal{X}_i \defn \{\phi_i(s, a_i)\}_{(s, a_i) \in \mathcal{S} \times \mathcal{A}_i}$, is full-dimensional. If this is not the case, the model parameters can be represented using a lower-dimensional linear parameterization (i.e., a smaller $d$). Applying Lemma~\ref{lm:generative_model_supporting_set}, we can construct a support set of state-action pairs $\mathcal{Y}_i \subset \mathcal{S} \times \mathcal{A}_i$ with $|\mathcal{Y}_i| \leq d(d+1)/2$, and a distribution $\rho_i$ over $\mathcal{Y}_i$ such that $\max_{(s, a_i) \in \mathcal{S} \times \mathcal{A}_i} \lVert \phi_i(s, a_i) \rVert_{\Sigma_i^{-1}}^2 \leq d$,
where $\Sigma_i = \mathbb{E}_{(s, a_i) \sim \rho_i} [\phi_i(s, a_i) \phi_i(s, a_i)^\top]$. Given the finite set $\mathcal{Y}_i$, we query the generative model $\lceil N \rho_i(s, a_i) \rceil$ times for each $(s, a_i) \in \mathcal{Y}_i$. This yields a total sample size of at most $N + d(d+1)/2$. Crucially, this sampling design relies only on the finite set $\mathcal{Y}_i$, making the process feasible even when the original state-action space $\cX_i$ is infinite.

Armed with the constructed finite set $\cY_i$ for each agent $i$, we estimate nominal transitions and rewards via ridge regression. \LRQFTRL then follows an adversarial online learning procedure similar to Robust-Q-FTRL \citep{li2022minimax}, using backward recursion from $h=H$ to $h=1$. At each step $h$, the algorithm executes a $K$-iteration loop that samples around size $N$ based on $\cY_i$, estimates the nominal model, and updates the robust Q-function via the dual form of robust Bellman equation \citep{iyengar2005robust} (cf.~\eqref{eq:q_function_generative_model} in Appendix~\ref{sec:generative-algo-details}). After completing the $K$ iterations, it constructs an optimistic robust value function by adding a bonus term $\beta_{i,h}$. See Appendix~\ref{sec:generative-algo-details} for more details.
\subsection{Theoretical Guarantees}
 We provide the theoretical guarantee for the sample complexity of \LRQFTRL in Algorithm~\ref{alg:linear_generative_model}.
		  
\begin{tcolorbox}[thmboxstyle, title={\textbf{Theoretical Guarantee for Algorithm~\ref{alg:linear_generative_model}}}]
          \begin{theorem}
		  \label{thm:main_generative_model}
			Consider any robust linear MG $\{ \cS, \{\cA_i\}_{1 \leq i \leq n},\{\cU^{\sigma_i}(P^0,\cdot)\}_{1 \leq i \leq n}, r, H \}$ with uncertainty level $\{\sigma_i\}_{i\in[n]}$ and some $\delta \in (0,1)$. Algorithm~\ref{alg:linear_generative_model} can output an $\varepsilon$-robust CCE $\xi$, i.e., $\max_{1 \leq i \leq n} \{ \mathbb{E}_{\pi\sim \xi} [ V_{i,1}^{\star,\pi_{-i}, \sigma_i }(s)] - \mathbb{E}_{\pi\sim \xi}[ V_{i,1}^{\pi,\sigma_i}(s)] \} \leq \varepsilon$
			with probability at least $1 - \delta$, provided that $N \geq CH^4d^3/\varepsilon^2$ and $K \geq CH^4/\varepsilon^2$, where $C$ is a positive constant. Namely, the requirement of the total number of samples during the learning process is
			\begin{align*}
			  N_{\mathsf{all}} = HK(N+d(d+1)/2) \geq O\left(H^9d^3/\varepsilon^4 \right).
			\end{align*}
		  \end{theorem}
\end{tcolorbox}
The implications of the above theorem are as below:
\paragraph{Breaking the curse of multi-agency in large state spaces.}
To the best of our knowledge, this provides the first provable sample complexity guarantees for robust linear MGs using a generative model, irrespective of the uncertainty set formulation. To illustrate the implications of our results, we consider the reduction to the tabular setting, which corresponds to $d = S \max_{i \in [n]} A_i$. Substituting this into Theorem~\ref{thm:main_generative_model} yields a complexity of $O\big(H^9 S^3 \max_{i \in [n]} (A_i)^3/\varepsilon^4\big)$, which depends polynomially on all salient parameters. This demonstrates that \LRQFTRL is a sample-efficient algorithm for robust linear MGs, effectively overcoming the curse of multi-agency in the generative model setting. Comparisons with existing results in the tabular setting are summarized in Table~\ref{tab:prior-work-double-line}.
		  
		\paragraph{Technical insights.}
		 In contrast to robust Markov games in the tabular setting, estimating the transition model and reward function in large-scale (possibly infinite) state spaces presents a challenge, since sampling every state-action pair—as done in tabular robust Markov games—is infeasible. To address this, we construct a specific finite subset of the state-action space to sample from; this enables us to obtain sufficient model information for the entire space, followed by ridge regression to estimate the parameters. On the other hand, compared to standard linear MGs, the robust value function in R-LMGs is highly nonlinear with respect to the nominal transition kernel. As a result, the statistical error to be controlled is no longer a linear aggregation of the estimation errors across $K$ iterations, which typically cancel out in standard linear MGs. We address this challenge by drawing $N$ samples in each iteration (compared to a single sample in standard linear MGs) and employing a tailored decomposition of the robust value function to control the auxiliary statistical error, ultimately achieving a sharp sample complexity guarantee.

		  \section{Robust Linear MGs with Online Adversarial Environment}
		  \label{sec:online}
    

          In this section, we move beyond the generative model setting and focus on learning R-LMGs in a more practical {\em online interactive setting}, where agents interact with the environment and collect Markovian trajectories, reflecting many real-world applications \citep{mnih2015human,vinyals2019grandmaster}. Notably, the online setting in RMGs can be substantially more challenging than the generative model setting, due to the need to balance exploration and exploitation and to cope with statistical dependence along sampled trajectories.

In contrast to the generative model setting discussed in Section~\ref{sec:generative_model}, the online setting concerns not only the final output policy but also the performance throughout the entire exploration process, relative to the best possible performance in hindsight. Specifically, over a $T$-round online process, the goal is to minimize regret incurred along the way. Let $\xi^t : [H] \to \Delta\big(\mathcal{S} \to \Delta(\prod_{i \in [n]} \mathcal{A}_i)\big)$ denote the distribution over product policies executed by the algorithm at round $t \in [T]$. The regret is then defined as
    \begin{align}\label{eq:regret-def}
	  &\text{Regret}(T)
	   \defn \textstyle\max_{ i\in[n]}\sum_{t=1}^T \big\{ \mathbb{E}_{\pi\sim \xi^t} \big[ V_{i,1}^{\star,\pi_{-i}, \sigma_i }(s_1)\big] - \mathbb{E}_{\pi\sim \xi^t}\big[ V_{i,1}^{\pi,\sigma_i}(s_1)\big] \big\}.
	\end{align}

		  \subsection{Interactive Data Collection within the Uncertainty Set}


Many online interaction paradigms in practical studies allow agents to interact not only with the nominal environment but also with perturbed environments for adversarial training aimed at safety and sim-to-real transfer \citep{akkaya2019solving,ding2020learning}. For instance, to address the sim-to-real gap in robotics \citep{akkaya2019solving} and safety in autonomous driving \citep{ding2020learning}, the training environment is explicitly generated using approximate worst-case parameters or scenarios, actively configured to expose the agent's vulnerabilities. This interaction paradigm is also reflected in recent theoretical works: \citet{jackson2024discovering} and \citet{zhang2021robust} optimize policies against estimated adversarial environments, while \citet{ren2022distributionally} approximate worst-case models within a Wasserstein ball centered on nominal estimates.

Motivated by such widely used online interaction paradigms, we propose an practically meaningful online protocol in which agents may sample from transition kernels within the prescribed uncertainty set, rather than only from the nominal kernel.
Specifically, we can sample from the nominal environment and also utilize the following sampling mechanism interacting with an "adversarial environment" inside the uncertainty set, defined as:
\begin{definition}[\textbf{Interaction with an Adversarial Environment}]\label{def:interaction}
  For any agent $i \in [n]$, given a collection of vectors $\{V_h: \mathcal{S} \rightarrow [0,H] \}_{h \in [H]}$ (which need not be the robust value functions) and a joint policy $\pi$, we are allowed to sample trajectories  $(s_1, \mathbf{a}_1, r_1(s_1, \mathbf{a}_1),\ldots, \mathbf{a}_{H-1}, r_{H-1}(s_1, \mathbf{a}_1), s_H)$, where the transitions following $s_{h+1} \sim P_{i,h,s_h,a_i}^{\pi_{-i}, V_{h+1}}(\cdot)$ with $P_{i,h,s_h,a_i}^{\pi_{-i},V_{h+1}} \defn \textstyle\arg\min_{P \in \mathcal{U}^{\sigma_i}(P_{h,s_h,a_i}^{\pi_{-i}})} P V_{h+1}.$
\end{definition}
Namely, we are allowed to interact with and sample from the {\em still unknown} approximate adversarial environment induced by the input approximate robust value function vectors $\{V_h\}_{h\in[H]}$. This environment lies within the fictitious uncertainty set $\mathcal{U}^{\sigma_i}(P_{h,s_h,a_i}^{\pi_{-i}})$ (see \eqref{eq:fictitious-uncertainty-set}).

\paragraph{Comparison to prior online interaction mechanisms.}
Prior works on robust RL with online data collection \citep{lu2024distributionally,zheng2025distributionally} consider agents to interact with \emph{only the fixed nominal environment}.  However, this mechanism is inherently ill-posed: even for a special subclass of tabular robust Markov games (single-agent robust Markov decision processes), achieving sublinear regret is generally impossible in the worst case, as noted by \citet[Theorem~3.2]{lu2024distributionally}. Consequently, analyses in this setting are typically established under additional assumptions on the RMGs considered, namely, that the robust value function satisfies a vanishing minimality condition \citep{lu2024distributionally,zheng2025distributionally}.

\paragraph{Challenges.} We emphasize that the proposed online interaction setting introduces \emph{unique} challenges beyond those in standard linear MGs and in existing online interaction settings considered for  R-LMGs.
    
{\em 1. No degeneration to standard linear Markov games.} Even if we can query a worst-case transition kernel with respect to a given value function, without access to the ground-truth robust value function, it is impossible to identify and sample from the exact worst-case transition kernel within the uncertainty set; therefore, the problem does not degenerate to solving a standard linear MG. The key challenge is therefore to approximate the worst-case environment using progressively more accurate estimates of the robust value function.

 {\em 2. The sampled adversarial environment can be non-linear.} Prior work \cite{zheng2025distributionally} focuses on sampling only in the nominal transition environment satisfying linear function approximation, where the samples from the nominal model can be directly leveraged for estimating a linear model. While in our problems, we sample from both the nominal kernel and online adversarial environments lacking the linearity assumption. Consequently, the collected trajectories may be generated by a potentially non-linear model, which requires extra care in model estimation and in the design of subsequent procedures based on these samples.


\subsection{Algorithm Design}
To solve R-LMGs in the proposed online interactive setting, we develop \OLRQFTRL (summarized in Algorithm~\ref{alg:lin_robust_Q_FTRL}), which achieves an $\varepsilon$-approximate CCE for general R-LMGs (cf.~\eqref{eq:appximate_CCE}) and simultaneously minimizes regret. Compared to the generative model setting in Section~\ref{sec:generative_model}, the main challenge is to design a sampling procedure that mimics drawing Markovian trajectories under the exact worst-case environment using only an estimated robust value function. Before presenting the full algorithmic pipeline, we first describe one key component: the sampling strategy employed by \OLRQFTRL.  This strategy differs substantially from those used in the generative model setting and from those for standard linear Markov games \citep{wang2023breaking,jin2021v}, which serve as inspiration for our approach.

\paragraph{Collecting online trajectories in the approximate adversarial environment.}
At the $t$-th online sampling round, for $t \in [T]$, we run a $K$-iteration procedure at each time step $h \in [H]$. For a fixed round $t \in [T]$, step $h \in [H]$, and iteration $k$, let $\pi_h^{t,k}$ denote the current learned policy, and apply the hybrid sampling scheme in Algorithm~\ref{alg:multi-sampling}. Concretely, for each agent $i$, we collect $N$ samples by interacting with an approximate adversarial environment for the first $h-1$ steps, and then sampling from the nominal environment only at the final step $h$. The approximate adversarial environment in Algorithm~\ref{alg:multi-sampling} is specified by querying an additional collection of pessimistic robust value estimates $\{\underline{V}_j^l\}_{(j,l)\in[H]\times[t-1]}$. These auxiliary pessimistic estimates are designed to handle robust linear Markov games and are crucial for sampling under adversarial dynamics, a difficulty that does not arise in standard linear Markov games \citep{wang2023breaking,jin2021v}.

Specifically, for a given agent $i$, each $l \in [t-1]$, and each step $j \in [h-1]$, the agents first select a joint policy $\pi$ uniformly from the set $\{\otimes_{i' \in [n]} \pi_{i',j}^{l,k}\}_{k \in [K]}$ of $K$ policies from round $l$. For the rest of the paper, we denote $a_{i,h}$ as the action of the $i$-th agent at time step $h$ for all $(i,h) \in [n] \times [H]$. They then sample an action $\mathbf{a}_j \sim \pi(\cdot \mid s_j)$ (Line~\ref{line:select_policy_and_action}) and observe the next state $s_{j+1}$ from the adversarial kernel $P_{i,j,s_j,a_{i,j}}^{\pi_{-i}, \underline{V}_{i,j+1}^l}(\cdot)$, as in Definition~\ref{def:interaction}. At step $h$, we switch strategies: the opponents' joint policy $\pi_{-i}$ is selected uniformly from $\{\otimes_{i' \neq i} \pi_{i',h}^{l,k}\}_{k \in [K]}$ and their actions are sampled as $\mathbf{a}_{-i,h} \sim \pi_{-i}(\cdot \mid s_h)$, while agent $i$ samples $a_{i,h}$ uniformly from its action space (Line~\ref{line:select_opponents_policy_and_action}). Finally, we draw $s_{h+1}$ from the nominal kernel $P_{h,s_h,a_{i,h}}^{\pi_{-i}}(\cdot)$ (Line~\ref{line:collect_state_sample_from_nominal_environment}). We repeat this procedure $\lceil N/t \rceil$ times for each $l \in [t-1]$, yielding a dataset $\cD_i$ containing approximately $N$ samples, denoted by
$\{(s_1^{(m)}, \mathbf{a}_1^{(m)}, \mathbf{r}_1^{(m)}, \ldots, \mathbf{a}_h^{(m)}, \mathbf{r}_h^{(m)}, s_{h+1}^{(m)})\}_{m =1}^{|\cD_i|}$.

\refstepcounter{algocf}
\begin{algorithmbox}{Algorithm~\thealgocf\quad Hybrid-Sampling}\label{alg:multi-sampling}
    \begin{algorithmic}[1] 
        \STATE \textbf{Input:} Agent $i$, step $h$, episode $t$, distribution $\xi_h^l$, values $\{V_j^l\}_{(j,l)\in[H]\times[t-1]}$.

        \STATE Repeat for each $l \in [t-1]$, iterating $\lceil N/t \rceil$ times:
        \STATE \COMMENT{Sample from the approximate adversarial environment for the first $h-1$ steps}
        \FOR{$j\in[h-1]$}
        \STATE \COMMENT{Random policy draw from historical FTRL iterates at step $j$.}
            \STATE Select $\pi\sim\text{Unif}\big(\{\otimes_{i^\prime\in[n]}\pi_{i^\prime,j}^{l,k}\}_{k\in[K]}\big)$, take action $\mathbf{a}_j\sim \pi(\cdot \mymid s_j)$ \label{line:select_policy_and_action}
            
            \STATE Collect state sample $s_{j+1}\sim P_{i,j,s_j,a_{i,j}}^{\pi_{-i},V_{j+1}^l}(\cdot)$ (see Definition~\ref{def:interaction}))

        \ENDFOR
        \STATE \COMMENT{At step $h$, explore $a_i$ uniformly; opponents from past policies.}
        \STATE Select $\pi_{-i}\sim\text{Unif}(\{\otimes_{i^\prime\neq i}\pi_{i^\prime,h}^{l,k}\}_{k\in[K]})$, $a_{-i,h}\sim \pi_{-i}(s_h)$, $a_i\sim\text{Unif}(\mathcal{A}_i)$. \label{line:select_opponents_policy_and_action}
       
        \STATE Collect $s_{h+1}\sim P_{h,s_h,a_{i,h}}^{\pi_{-i}}(\cdot)$ (defined in \eqref{eq:basic_transition}).  \label{line:collect_state_sample_from_nominal_environment}
        
        \STATE \textbf{Output:} Samples are collected in $\cD_i = \{(s_1^{(m)}, \mathbf{a}_1^{(m)}, \mathbf{r}_1^{(m)}, \ldots, \mathbf{a}_h^{(m)}, \mathbf{r}_h^{(m)}, s_{h+1}^{(m)})\}_{m =1}^{|\cD_i|}$.
    \end{algorithmic}

\end{algorithmbox}

The central mechanism of Algorithm~\ref{alg:lin_robust_Q_FTRL} is to construct both optimistic and pessimistic estimates of the robust value function for subseqeunt procedures. Using the pessimistic estimate, the algorithm collects online trajectory samples from an adversarial environment, while using the optimistic estimate to update the policy. We now describe the algorithm in detail.
At each online round $t \in [T]$, the algorithm learns a distribution $\xi^t$ over policies by running $K$ iterations at each time step $h \in [H]$. For any fixed round $t$ and step $h \in [H]$, the $k$-th iteration proceeds as follows:

\begin{itemize}[topsep=0pt,leftmargin=10pt]
  \setlength{\itemsep}{-3pt}

    
 \item \textbf{Estimate nominal model.}
 In the generative-model setting of Section~\ref{sec:generative-algo-design}, one actively queries a pre-designed per-agent state-action set $\mathcal{Y}_i$ from Lemma~\ref{lm:generative_model_supporting_set}. In contrast, the online setting's visited state-action pairs are determined entirely by the sampled trajectories in $\mathcal{D}_i$ and need not---nor can they---be chosen as \emph{a prior}. We therefore directly apply ridge regression to the trajectory samples in $\mathcal{D}_i$ collected by $\textsf{Hybrid-Sampling}$ to estimate the transition probabilities and the reward function:
    \begin{subequations}
      \label{eq:ridge_regression_online}
      \begin{align}
           &\textstyle \Lambda_{i,h}^{k,t} = \sum_{m=1}^{|\cD_i|} \phi_i(s_h^{(m)}, a_{i,h}^{(m)}) \phi_i(s_h^{(m)}, a_{i,h}^{(m)})^\top + \lambda I,\quad
        \textstyle\widehat{\theta}_{i,h}^{k,t} = (\Lambda_{i,h}^{k,t})^{-1} \sum_{m=1}^{|\cD_i|} \phi_i(s_h^{(m)}, a_{i,h}^{(m)}) r_{i,h}^{(m)} ,\\
        &\textstyle\widehat{\mu}_{i,h}^{k,t} = (\Lambda_{i,h}^{k,t})^{-1} \sum_{m=1}^{|\cD_i|} \phi_i(s_h^{(m)}, a_{i,h}^{(m)}) \delta_{s_{h+1}^{(m)}},
      \end{align}
where $\Lambda_{i,h}^{k,t}\in\mathbb{R}^{d\times d}, \phi_i(\cdot,\cdot) \in\mathbb{R}^d, \widehat{\mu}_{i,h}^{k,t}\in\mathbb{R}^d\rightarrow\Delta(\mathcal{S})$. Here, we denote $\delta_s$ as a measure over state space $\cS$ so that the measure on state $s$ is $1$, and the measure on other states is $0$, for all $s\in\cS$. $\lambda$ is the regularization parameter in ridge regression satisfying $\lambda \geq 1$. 
    For all $(s, a_i) \in \mathcal{S} \times \mathcal{A}_i$, we compute the estimated reward function and transition probabilities as: 
    \begin{align}
      r_{i,h}^{k,t}(s, a_i) = \phi_i(s, a_i)^\top \widehat{\theta}_{i,h}^{k,t},\quad 
      P_{i,h}^{k,t}(\cdot \mid s, a_i) =\phi_i(s, a_i)^\top \widehat{\mu}_{i,h}^{k,t}.
    \end{align}
    \end{subequations}

    \item \textbf{Update robust-Q-function and policy.} Different from generative model setting, we now compute both the optimistic estimation and pessimistic estimation of the current-step Q-function. Specifically, for all $(s,a_i)\in\mathcal{S}\times\mathcal{A}_i$, we compute $\overline{q}_{i,h}^{k,t}(s,a_i)$ and $\underline{q}_{i,h}^{k,t}(s,a_i)$ as follows:
    \begin{subequations}
      \label{eq:q_function_oinline}
          \begin{align}
           &\textstyle \overline{q}_{i,h}^{k,t}(s, a_i) = r_{i,h}^{k,t}(s, a_i)   + \textstyle\max_{\alpha \in \big[\overline{V}_{\text{min}}, \overline{V}_{\text{max}}\big]}\big\{
              P_{i,h}^{k,t}(\cdot \mymid s,a_i) \big[\overline{V}_{i,h+1}^t\big]_{\alpha}-\sigma_i \big(\alpha-\min_{s^\prime} \big[ \overline{V}_{i,h+1}^t\big]_{\alpha}\left(s^\prime\right) \big)\big\},\\
        &\textstyle \underline{q}_{i,h}^{k,t}(s, a_i) = r_{i,h}^{k,t}(s, a_i) + \textstyle\max_{\alpha \in \big[\underline{V}_{\text{min}}, \underline{V}_{\text{max}}\big]} \big\{
        P_{i,h}^{k,t}(\cdot \mymid s,a_i)\big[\underline{V}_{i,h+1}^{t}\big]_{\alpha}- \sigma_i\big(\alpha - \min_{s'}\big[\underline{V}_{i,h+1}^{t}\big]_{\alpha}\left(s'\right)\big)\big\}.
          \end{align}
      \end{subequations}
    where we define $\overline{V}_{\text{min}}=\min_s \overline{V}_{i,h+1}^{t}(s)$, $\overline{V}_{\text{max}}=\max_s \overline{V}_{i,h+1}^{t}(s)$, $\underline{V}_{\text{min}}=\min_s \underline{V}_{i,h+1}^{t}(s)$ and $\underline{V}_{\text{max}}=\max_s \underline{V}_{i,h+1}^{t}(s)$.
   After calculating the \( q \)-function, we employ the Follow-the-Regularized-Leader (FTRL) strategy to update the learner's policy, i,e., 
\begin{align*}
   \forall (i, s, a_i) \in [n] \times  \mathcal{S} \times \mathcal{A}_i:  \pi_{i,h}^{k+1,t}(a_i \mid s) = \textstyle\frac{\exp(\eta_{k+1} \sum_{l=1}^{k} \overline{q}_{i,h}^{l,t}(s, a_i))}{\sum_{a_i \in \mathcal{A}_i} \exp(\eta_{k+1} \sum_{l=1}^{k} \overline{q}_{i,h}^{l,t}(s, a_i))}.
\end{align*}
\item {\bf Estimate robust value function.} 
    Subsequently, we compute the optimistic and pessimistic estimates of the value function as follows:
    \begin{equation}
      \label{eq:estimate_value_function}
      \begin{aligned}
          \overline{V}_{i,h}^t(s)& = \textstyle\min\big\{ \sum_{k=1}^{K} \frac{1}{K}\big<\pi_{i,h}^{k,t}(\cdot\mid s),\overline{q}_{i,h}^{k,t}(s,\cdot)\big>
          + \beta_{i,h,1}^{t}(s),~H-h+1\big\}\\
          \underline{V}_{i,h}^t(s)&=\textstyle\max\big\{\sum_{k=1}^{K} \frac{1}{K}\big<\pi_{i,h}^{k,t}(\cdot\mid s),\underline{q}_{i,h}^{k,t}(s,\cdot)\big>
           - \beta_{i,h,2}^{t}(s),0\big\}.
      \end{aligned}
  \end{equation}
  Here, for all $(i,h,t,s) \in [n]\times[H]\times[T]\times\mathcal{S}$, the bonus terms are defined as
  \begin{align}
    \textstyle\beta_{i,h,1}^{t}(s)
    =&\textstyle\max_{a_i\in\mathcal{A}_i}\sum_{k=1}^K\frac{1}{K}
    \big\lVert\phi_i(s,a_i)\big\rVert_{\big(\Lambda_{i,h}^{k,t}\big)^{-1}}
    \big(4H\sqrt{d\ln(NH+1)+\ln\big(\frac{3TNHnK}{\delta}\big)}+2H\sqrt{d}\big)\notag\\
    &\textstyle+\frac{1}{N}+2H\sqrt{\frac{\ln A_i}{K}}, \label{eq:online_optimistic_bonus} \\
    \textstyle\beta_{i,h,2}^{t}(s)
    =&\textstyle\sum_{k=1}^K\frac{1}{K}\mathbb{E}_{a_i\sim\pi_{i,h}^{k,t}(s)}
    \big\lVert\phi_i(s,a_i)\big\rVert_{\big(\Lambda_{i,h}^{k,t}\big)^{-1}}
    \big(4H\sqrt{d\ln(NH+1)+\ln\big(\frac{3TNHnK}{\delta}\big)}+2H\sqrt{d}\big)+\frac{1}{N}.  \label{eq:online_pessimistic_bonus}
  \end{align}

\end{itemize}

\refstepcounter{algocf}
\begin{algorithmbox}{Algorithm~\thealgocf\quad Online-Robust-$Q$-Linear-FTRL}\label{alg:lin_robust_Q_FTRL}

    \begin{algorithmic}[1]
        \STATE \textbf{Input:} Iterations $T$, learning rate $\{\eta_{t+1}\}$.
        \FOR{$t=1,2,\ldots, T$}
            \STATE \textbf{Initialization:} $\overline{V}_{i,H+1}^t=\underline{V}_{i,H+1}^t=q_{i, H+1}^t=0$, $\pi_{i,h}^1(a_i | s)=1/A_i$ for all $i,h,s,a_i$.
            \FOR{$h = H, H-1, \ldots, 1$ and $k=1, \ldots, K$}
                \STATE \COMMENT{Backward procedure over $[H]$ with $K$ inner FTRL steps;}
                \FOR{$i = 1, 2, \ldots, n$}
                \STATE \COMMENT{Collect samples from the approximate adversarial environment}
                    \STATE Collect $N$ samples with $\textsf{Hybrid-Sampling}(i,h,t,\{\xi^l\}_{l<t},\{\underline{V}_{i,j}^l\}_{(j,l)\in[H]\times[t-1]})$ in Algo~\ref{alg:multi-sampling}.
                    \STATE \COMMENT{Estimate the transition probabilities, reward function, and current step Q-function}
                    \STATE Estimate $r_{i,h}^{k,t}, P_{i,h}^{k,t}$ via \eqref{eq:ridge_regression_online} and $q$-function $\overline{q}_{i,h}^{k,t}$ and $\underline{q}_{i,h}^{k,t}$ via \eqref{eq:q_function_oinline}.
                    \STATE \COMMENT{Update policy using FTRL with \emph{optimistic} $\overline{q}$ only.}
                    \STATE Update policy for all $(s,a_i)$:  $\pi_{i,h}^{k+1,t}(a_i|s)=\frac{\exp(\eta_{k+1}\sum_{l=1}^k\overline{q}_{i,h}^{l,t}(s,a_i))}{\sum_{a_i'\in\mathcal{A}_i}\exp(\eta_{k+1}\sum_{l=1}^k\overline{q}_{i,h}^{l,t}(s,a_i'))}$
                   
                \ENDFOR
            \ENDFOR
            \STATE Update $\overline{V}_{i,h}^t(s), \underline{V}_{i,h}^t(s)$ via \eqref{eq:estimate_value_function}.
            \STATE Let $\xi^t$ be the uniform distribution over product policies $\{\pi_h^{k,t}\}_{(h,k)\in[H]\times[K]}$.
        \ENDFOR
        \STATE \textbf{Output:} All distributions $\{\xi^t\}_{t\in[T]}$ over product policies.
    \end{algorithmic}
\end{algorithmbox}

  \subsection{Theoretical Guarantee}


\begin{tcolorbox}[thmboxstyle, title={\textbf{Theoretical Guarantee for Algorithm~\ref{alg:lin_robust_Q_FTRL}}}]
  \begin{theorem}
  \label{thm:main_online}
 Consider any R-LMG $\{ \mathcal{S}, \{\mathcal{A}_i\}_{1 \leq i \leq n}, \{\mathcal{U}^{\sigma_i}(P^0,\cdot)\}_{1 \leq i \leq n}, r, H \}$ with a fixed initial state $s_1 \in \cS$, uncertainty levels $\{\sigma_i\}_{i \in [n]}$, and fix $\delta \in (0,1)$. Let $N = K = T$. Then, with probability at least $1-\delta$, the regret of Algorithm~\ref{alg:lin_robust_Q_FTRL} satisfies
\begin{align*}
\textstyle \text{Regret}(T)\leq \tilde{\mathcal{O}}\big(H^2 d \max_{i\in[n]}A_i  \sqrt{T}\big).
\end{align*}
And if the online iteration time T satisfies $T > H^4 d^2 \max_{i \in [n]} (A_i)^2 / \epsilon^2$, then Algorithm~\ref{alg:lin_robust_Q_FTRL} outputs an $\varepsilon$-CCE.
\end{theorem}

 
  \end{tcolorbox}

\noindent\textbf{Breaking the curse of multiagency and comparison to prior work.}
To illustrate the implications of our results, we consider the specialization to the tabular setting, which corresponds to setting the feature dimension as $d = S \max_{i \in [n]} A_i$. In this case, Theorem~\ref{thm:main_online} yields the regret bound
$\tilde{\mathcal{O}}\!\Big( H^2 S \max_{i \in [n]} (A_i)^2 \sqrt{T} \Big)$,
which is polynomial in all salient parameters. This shows that Linear-Robust-Q-FTRL breaks the curse of multiagency in the proposed online interactive setting.

A closely related work is \citet{zheng2025distributionally}, which also studies robust linear MGs, but under a different online interactive setting: agents interact only with the fixed nominal transition kernel. \citet{zheng2025distributionally} focuses on a restrictive problem class in which the robust value function is required to satisfy a \emph{vanishing minimal value} condition. Moreover, under some additional assumptions in \cite[Corollary~5.3]{zheng2025distributionally}, \citet{zheng2025distributionally} obtains a sublinear regret bound
$\mathrm{Regret}(T) \leq \tilde{\mathcal{O}}\!\Big( H \min\!\big\{ H, \tfrac{1}{\min\{\sigma_i\}} \big\} \, d_{\mathsf{curse}} \sqrt{T} \Big)$,
where, in the tabular reduction, the feature dimension becomes $d_{\mathsf{curse}} = S \prod_{i=1}^n A_i$. Consequently, the regret bound in \cite[Corollary~5.3]{zheng2025distributionally} scales as
$\tilde{\mathcal{O}}\big(H \min\!\big\{ H, \tfrac{1}{\min\{\sigma_i\}} \big\} \, S \,\textcolor{red}{\prod_{i=1}^n A_i}\, \sqrt{T}\big)$,
which grows exponentially in the number of agents through the joint action space $\prod_{i=1}^n A_i$. In contrast, Linear-Robust-Q-FTRL guarantees
$\tilde{\mathcal{O}}\!\Big( H^2 S \,\textcolor{red}{\max_{i \in [n]} (A_i)^2}\, \sqrt{T} \Big)$, thereby avoiding exponential dependence on the number of agents---a benefit that is retained even in the tabular setting.

\begin{tcolorbox}[defboxstyle, title={\textbf{Technical Contributions}}]
  \begin{itemize}[topsep=0pt,leftmargin=17pt]
	\setlength{\itemsep}{-3pt}

\item \textbf{Constructing an adversarial environment using a pessimistic value function.}
As discussed above, we cannot sample directly from the true worst-case adversarial transition dynamics because the ground-truth robust value function is unknown. To approximate that adversarial environment, in addition to the optimistic robust value estimates typically used to guide policy updates in standard (linear) MGs \citep{wang2023breaking,jin2021v}, we introduce a sequence of \emph{pessimistic} robust value estimates. These pessimistic estimates are fed into the hybrid sampling module (cf.~Algorithm~\ref{alg:multi-sampling}) to obtain a tight lower bound on the robust value function under the true adversarial environment.

\item \textbf{Maintaining nominal-model estimation via hybrid-environment sampling.}
A key technical challenge is that the linear function approximation assumption applies only to the nominal kernel; kernels in the uncertainty set need not preserve this structure, making direct estimation of the worst-case kernel intractable. To address this, we continue to estimate the nominal model while preserving the state--action occupancy distribution induced by the worst-case environment. Specifically, we use a hybrid sampling mechanism that rolls out Markovian trajectories under the approximate worst-case environment for the first $h-1$ steps and queries the nominal kernel only at the final step $h$. This design allows us to estimate the nominal kernel while preserving the state-action occupancy distribution induced by the worst-case environment up to step $h-1$, which play a crucial role in subsequent estimation of the value function corresponding to the worst-case environment. 

  \end{itemize}
 \end{tcolorbox}


\section{Conclusion}
In this work, we develop provably data-efficient algorithms for robust linear Markov games (R-LMGs) for large-scale, possibly infinite state spaces under linear function approximation. Focusing on a fictitious uncertainty set defined via total variation, we study both the generative-model setting and a more practical online interactive data-collection setting. In both settings, to the best of our knowledge, we provide the first provably sample-efficient algorithms for R-LMGs that overcome the curse of multi-agency, regardless of the uncertainty set formulation. Notably, the proposed online interactive setting is practically meaningful while also preserving well-posedness for achieving sublinear regret. On the technical side, we introduce a new hybrid sampling mechanism to address the challenge of sampling from an approximate adversarial environment in the online setting, which may be of independent interest for future research.

\section*{Acknowledgments}
The authors thank Prof. Yuejie Chi for her guidance throughout the process and Prof. Eric Mazumdar for very helpful discussions. Laixi Shi acknowledges funding support from MERL.

\bibliographystyle{unsrtnat}
\bibliography{bibfileRL,bibfileDRO,bibfileGame}

\newpage
\appendix
\clearpage

\refstepcounter{algocf}
\begin{algorithmbox}{Algorithm~\thealgocf\quad \LRQFTRL}\label{alg:linear_generative_model}
    \begin{algorithmic}[1]
        \STATE \textbf{Input:} The support set $\{\mathcal{Y}_i\}_{i\in[n]}$ and the sampling distribution $\{\rho_i\}_{i\in[n]}$; number of iterations $K$.
        \STATE \textbf{Initialization:} $\widehat{V}_{i,H+1}(s)=Q_{i, h}^0(s,a_i)=0$, $\pi_{i,h}^1(a_i \mid s)=1/A_i$ for all $(i, h,s,a_i) \in [n] \times [H] \times \mathcal{S} \times \mathcal{A}_i$.
        \FOR{$h = H, H-1, \ldots, 1$, and $k\in[K]$}
             \FOR{$i = 1, 2, \ldots, n$}
                 \STATE \COMMENT{Sampling from the nominal model based on the finite set $\mathcal{Y}_i$ and probability distribution $\rho_i$.}
                 \FOR{ $(s,a_i)\in\mathcal{Y}_i$ and $u = 1,2,\cdots \lceil N\rho_i(s,a_i)\rceil$}
                    \STATE Sample $\mathbf{a}^{(u)}(s,a_i)= [a_{j}(s,a_i) ]_{j\in[n]}$ with
                    $a_{j}(s,a_i) \overset{\text{ind.}}{\sim} \pi_{j, h}(\cdot \mid s) ~(j\neq i),\quad
                        a_{i}(s,a_i) = a_i.$ \label{line:sample_nominal_model-1}
                    \STATE Sample from the generative model:
                    \begin{align*}
                        &r_{i,h}^{(u)} = r_{i,h}(s, \mathbf{a}^{(u)}(s,a_i)), \quad
                        s^{(u)}_{h+1} \sim P_h( \cdot \mid s, \mathbf{a}^{(u)}(s,a_i)).
                    \end{align*}
                    \STATE \COMMENT{Observe reward and next state from the nominal model $P_h$.}
                    \STATE Insert the data tuple $\{s_h^{(m)} = s, a_{i,h}^{(m)} = a_i, r_{i,h}^{(m)} = r_{i,h}^{(u)}, s_{h+1}^{(m)} = s^{(u)}_{h+1}\}$ into $\mathcal{D}_i$.  \label{line:sample_nominal_model-2}
                 \ENDFOR
                 \STATE Estimate model parameters $\{r_{i,h}^k, P_{i,h}^k\}$ with $\mathcal{D}_i$ via ridge regression using \eqref{eq:ridge_regression_generative_model}.
                 \STATE \COMMENT{Ridge regression for estimating nominal reward and transition in the feature space.}
                 \STATE Compute $q$-function $q_{i,h}^k(s,a_i)$ with \eqref{eq:q_function_generative_model}.
                 \STATE \COMMENT{Follow the Regularized Leader (FTRL) algorithm to update the policy.}
                 \STATE Update policy for all $(s,a_i)\in\mathcal{S}\times\mathcal{A}_i$ with
                 \begin{align}\label{eq:policy_update_generative_model}
                    \pi_{i,h}^{k+1}(a_i\mid s)=\frac{\exp (\eta_{k+1}\sum_{l=1}^kq_{i,h}^l(s,a_i))}{\sum_{a_i\in\mathcal{A}_i}\exp (\eta_{k+1}\sum_{l=1}^kq_{i,h}^l(s,a_i))}
                 \end{align}
                 
             \ENDFOR
             \STATE \COMMENT{Update optimistic robust values with bonus $\beta_{i,h}$ after $K$ rounds at step $h$.}
             \STATE Construct an upper bound $\{\widehat{V}_{i,h}\}$ for value functions using \eqref{eq:line-number-policy-update_generative_model}.
             
        \ENDFOR
        \STATE \textbf{Output:} Product policies $\{\pi_h^k = (\pi_{1,h}^k \times \cdots \times \pi_{n,h}^k)\}_{h \in [H], k \in [K]}$ and a probability distribution $\xi = \{\xi_h\}_{h \in [H]}$ over product policies so that $\xi_h(\pi_h^k) = 1/K$ for all $h \in [H]$.
    \end{algorithmic}
\end{algorithmbox}

\section{Algorithm Design with a Generative Model}\label{sec:generative-algo-details}
Due to the space limit, we postpone the entire pipleline of Algorithm~\ref{alg:linear_generative_model} in this section as follows.


For initialization, \LRQFTRL set the robust value function, robust Q-function $\widehat{V}_{i,H+1}(s)=Q_{i, h}^0(s,a_i)=0$, and the policy $\pi_{i,h}^1(a_i\mymid s)=1/A_i$ for all $(i,s) \in [n] \times \cS$.
Then \LRQFTRL performs a backward recursive learning path, learning recursively from the final time step $h=H$ to $h=1$. At each time step $h\in[H]$, a $K$-iterations online learning process is applied. At each interaction iteration $k\in[K]$, parallelly for each agent $i\in[n]$, we execute the following steps:
\begin{itemize}

  \item {\bf Sampling $\mathcal{D}_i$ from the nominal model.} For each $(s,a_i) \in \mathcal{Y}_i$, we collect $\lceil N\rho_i(s,a_i)\rceil$ samples from the generative model (cf.~line~\ref{line:sample_nominal_model-1} to \ref{line:sample_nominal_model-2}) and yield the dataset $\cD_i =\big\{s_h^{(m)}, a_{i,h}^{(m)}, r_{i,h}^{(m)}, s_{h+1}^{(m)} \big\}_{m \in \mathcal{D}_i}$
\item  { \bf Estimating the nominal model.} Armed with the dataset $\cD_i$, we estimate the model parameters $\{r_{i,h}^k, P_{i,h}^k\}$ via ridge regression as below:
    \begin{subequations}
      \label{eq:ridge_regression_generative_model}
      \begin{align}
        &\Lambda_{i,h}^k = \sum_{m\in\mathcal{D}_i} \phi_i(s_h^{(m)}, a_{i,h}^{(m)}) \phi_i(s_h^{(m)}, a_{i,h}^{(m)})^\top + \lambda I,\quad
        \widehat{\theta}_{i,h}^k = (\Lambda_{i,h}^k)^{-1} \sum_{m\in\mathcal{D}_i} \phi_i(s_h^{(m)}, a_{i,h}^{(m)}) r_{i,h}^{(m)} ,\\
        &\widehat{\mu}_{i,h}^k = (\Lambda_{i,h}^k)^{-1} \sum_{m\in\mathcal{D}_i} \phi_i(s_h^{(m)}, a_{i,h}^{(m)}) \delta_{s_{h+1}^{(m)}},
    \end{align}
      where $\Lambda_{i,h}^k\in\mathbb{R}^{d\times d}, \phi_i(\cdot,\cdot) \in\mathbb{R}^d, \widehat{\mu}_{i,h}^k\in\mathbb{R}^d\rightarrow\Delta(\mathcal{S})$. Here, we denote $\delta_s$ as a measure over state space $\cS$ so that the measure on state $s$ is $1$, and the measure on other states is $0$, for all $s\in\cS$. $\lambda$ is the regularization parameter in ridge regression satisfying $\lambda \geq 1$.
    For all $(s, a_i) \in \mathcal{S} \times \mathcal{A}_i$, we compute the estimated reward function and transition probabilities as:
    \begin{align}
      &r_{i,h}^k(s, a_i) = \phi_i(s, a_i)^\top \widehat{\theta}_{i,h}^k,\quad
      P_{i,h}^k(\cdot \mid s, a_i) = \phi_i(s, a_i)^\top\widehat{\mu}_{i,h}^k .
    \end{align}
    \end{subequations}

    \item {\bf Updating the robust-Q-function and policy.} Denoting the current joint policy by $\pi_h^k$, with the estimated model $\{r_{i,h}^k, P_{i,h}^k\}_{i \in [n]}$, we compute a robust $Q$-function at each iteration. Specifically, for all $(s,a_i) \in \mathcal{S} \times \mathcal{A}_i$, we compute $q_{i,h}^k(s,a_i)$ as follows:
    \begin{equation}    
    \label{eq:q_function_generative_model}
    \begin{aligned}
        & q_{i,h}^k(s, a_i) = r_{i,h}^k(s, a_i) \\
                &\quad  + \max_{\alpha \in \big[\min_s \widehat{V}_{i,h+1}(s), \max_s \widehat{V}_{i,h+1}(s)\big]} \big\{
        P_{i,h}^k(\cdot \mymid s,a_i)\big[\widehat{V}_{i,h+1}\big]_{\alpha}- \sigma_i\big(\alpha - \min_{s'}\big[\widehat{V}_{i,h+1}\big]_{\alpha}\left(s'\right)\big)\big\},
    \end{aligned}
  \end{equation}
  which is the dual form of the robust Bellman equation \citep[Lemma 1]{shi2023curious}
  \begin{align}
      q^k_{i,h}(s, a_i) = r_{i,h}^k(s, a_i) + \inf_{ \cP \in \unb^{\ror_i}(P_{i,h}^k(\cdot \mymid s,a_i))} \cP \widehat{V}_{i,h+1}
  \end{align}
  and is computationally tractable \citep{iyengar2005robust}. Note that even $\mathcal{S}$ may be infinite, the model estimate $P_{i,h}^k(\cdot\mymid s,a_i)$ obtained from ridge regression in \eqref{eq:ridge_regression_generative_model} is supported on the finite number of observed next states $\{s_{h+1}^{(m)}\}_{m\in\mathcal{D}_i}$ in dataset $\mathcal{D}_i$. Consequently, $P_{i,h}^k(\cdot\mymid s,a_i)[\widehat{V}_{i,h+1}]_\alpha$ is a finite weighted sum, and $\widehat{V}_{i,h+1}$ need only be evaluated at finitely many points, as in the lazy-evaluation implementation in standard linear MDP/MG algorithms \citep{jin2019provablyefficientreinforcementlearning,wang2023breaking}.
  After calculating the $q$-function, we employ the Follow-the-Regularized-Leader (FTRL) strategy \citep{shalev2012online,li2022minimax} to update the learner's policy towards $\pi_{i,h}^{k+1}$ via \eqref{eq:policy_update_generative_model}.

 \item {\bf Estimate robust value function.} 
  Subsequently, after completing the $K$ iterations at step $h$, we compute an optimistic estimate of the value function as follows: for all $(i, s) \in [n] \times \mathcal{S}$,
    \begin{equation}\label{eq:line-number-policy-update_generative_model}
    \begin{aligned}
        \widehat{V}_{i,h}(s)& = \min\big\{ \sum_{k=1}^{K} \frac{1}{K} \big<\pi_{i,h}^k(\cdot \mid s),q_{i,h}^k(s,\cdot)\big>
               + \beta_{i,h}(s),H-h+1\big\},
          \end{aligned}
        \end{equation}
  Here, $\beta_{i,h}(s)$ is a bonus term specified as follows:
        \begin{align}
        \label{eq:genetative_model_beta}
       \beta_{i,h}(s)=8\sqrt{\frac{d}{N}}\big(2H\sqrt{d\ln(NH+1)+2\ln\big(\frac{3KNHn}{\delta}\big)}+H\sqrt{d}\big)+2H\sqrt{\frac{\ln A_i}{K}}.
    \end{align}

\end{itemize}  
After completing the learning for all $h\in[H]$, each with $K$ iterations at every step, we output the joint correlated policy 
        $\big\{ \frac{1}{K} \sum_{k=1}^K\pi_h^k = \frac{1}{K} \sum_{k=1}^K \{ (\pi_{1,h}^k \times \cdots \times \pi_{n,h}^k) \}_{h\in[H]}\big\}$
        as the learned solution, which can be shown later as an $\varepsilon$-CCE with high probability.

\section{Proof of Theorem~\ref{eq:ne-existence}}\label{sec:proof-eq:ne-existence}

The proof of equilibrium existence often relies on fixed-point theorems. In this paper, to establish equilibrium existence for large-scale (possibly infinite) state spaces—including those that may be infinite or continuous—we appeal to the following fixed-point theorem.

\begin{definition}[Upper semi-continuous]\label{def:semi-continous}
A point-to-set mapping $x \in X \mapsto \phi(x) \in Y$ is upper semi-continuous if $\lim_{n\rightarrow\infty}x^n =x_0, y^n\in \phi(x^n), 
\lim_{n\rightarrow \infty} y^n = y_0$ imply that $y_0\in \phi(x_0)$.
\end{definition} 
\begin{theorem}[Glicksberg's fixed point theorem \citep{glicksberg1952further}]\label{thm:clicksberg}
If $X$ is a compact and convex subset in a Hausdorff locally convex topological vector space, and $\phi$ is an upper semi-continuous correspondence mapping $X$ into the family of all convex subsets of $X$, then there exists $x\in X$ so that $x\in\phi(x)$.
\end{theorem}

It is worth noting that for cases with finite state spaces, the existence of Nash equilibria for RMGs—a broader class of games that includes R-LMGs—has been established in \citet[Theorem1]{shi2024breakcursemultiagencyrobust}, where another theorem --- Kakutani's fixed-point theorem is used. Theorem~\ref{thm:clicksberg} relaxes the requirement in Kakutani's fixed-point theorem, where $X$ must belong to a finite-dimensional Euclidean space, extending it to a broader class of Hausdorff linear topological spaces that is crucial for us to handle infinite state spaces. 

\paragraph{Proof of Theorem~\ref{eq:ne-existence}.}
With Theorem~\ref{thm:clicksberg} in hand, we prove Theorem~\ref{eq:ne-existence} by following the proof outline of \citet[Theorem~1]{shi2024breakcursemultiagencyrobust}, with modifications to accommodate a large-scale (possibly infinite) state space. In the large-state setting, we work directly with general robust MGs and do not exploit the game’s linear structure to narrow down to the subclass of robust MGs (robust linear MGs). Following \citet[Section~B.1]{shi2024breakcursemultiagencyrobust}, we first establish the existence of a Nash equilibrium in an auxiliary one-step game; this result is then used recursively to verify equilibrium existence for the full sequential RMG.

Fix a step \(h\in[H]\) and consider the associated auxiliary one-step game. Let \(V_{i,h+1}\in\mathbb{R}^{\cS}\) be a fixed "value function" for agent \(i\), with \(0\le V_{i,h+1}(s)\le H\) for all \(s\in\cS\). Given any joint product policy \(\pi:\cS\to \prod_{i\in[n]}\Delta(\cA_i)\) at step \(h\), define agent \(i\)'s payoff at state \(s\) by
\begin{align}
\textstyle\forall s\in\mathcal{S}:\quad
f_{i,s}\big(\pi_i(s),\pi_{-i}(s);V_{i,h+1}\big)
&= \mathbb{E}_{\mathbf{a}\sim \pi(s)}\!\big[r_{i,h}(s,\mathbf{a})\big]
+ \mathbb{E}_{a_i\sim \pi_i(s)}\!\big[\textstyle\inf_{\mathcal{U}^{\sigma_i}\!\left(P^{\pi_{-i}}_{h,s,a_i}\right)} PV_{i,h+1}\big].
\label{eq:lemma-continuous-2-payoff}
\end{align}
Based on these payoffs, we can introduce the best-response correspondence mapping \(\phi\) as follows: for any \(\pi:\cS\to \prod_{i\in[n]}\Delta(\cA_i)\),
{
\begin{align}
\textstyle \phi(\pi) \defn \big\{ u :\cS \mapsto \prod_{i\in[n]} \Delta(\cA_i) \mymid u_i(s) \in \mathrm{argmax}_{\pi_i'(s)\in\Delta(\cA_i)} \;f_{i,s}(\pi_i'(s), \pi_{-i}(s); V_{i,h+1}), \forall (i,s)\in [n] \times \cS  \big\}. \label{eq:definiton-nash-phi}
\end{align}
}
For this one-step game, the (joint) product policy space is
\[
\textstyle X \defn \Big\{\pi:\cS\to \prod_{i\in[n]}\Delta(\cA_i)\Big\}.
\]
Since fixed points of \(\phi\) correspond to Nash equilibria (NE), Theorem~\ref{thm:clicksberg} yields the existence of NE once its conditions are verified. To show the three conditions, we begin by showing that \(X\) is a compact and convex subset of a convex Hausdorff linear topological space. First, $\{ \pi: \cS \mapsto \prod_{i\in[n]} \Delta(\cA_i)\}$ is convex since the finite product of simplex $X' \defn \prod_{i\in[n]} \Delta(\cA_i)$ is convex and $\{ \pi: \cS \mapsto \prod_{i\in[n]} \Delta(\cA_i)\}$ can be seen as a infinite product of a convex set, whose convexity is easily verified since product operations preserve convexity. Next, the compactness of $\{ \pi: \cS \mapsto \prod_{i\in[n]} \Delta(\cA_i)\}$ follows from the compactness of $\prod_{i\in[n]} \Delta(\cA_i)$ and Tychonoff's theorem \citep{wright1994tychonoff}, which guarantees that an arbitrary product of compact sets remains compact. Therefore, we verified that \(X\) is a compact and convex set.

Next, define
\[
\textstyle Y' \defn \prod_{i\in[n]} \mathbb{R}^{|\cA_i|},
\qquad
Y \defn (Y')^{\cS} = \Big\{\pi:\cS\to Y'\Big\},
\]
Since \(Y'\) is a finite-dimensional Euclidean space, it is a locally convex Hausdorff topological vector space. And arbitrary products of Hausdorff (resp.\ locally convex) topological vector spaces are again Hausdorff \citep[Theorem 5]{kelley2017general} (resp.\ locally convex \citep[Page 207]{kothe1969topological}). It follows that the finite product $Y'$ is a locally convex Hausdorff topological vector space, and then the arbitrary product \(Y = (Y')^{\cS}\) is a locally convex Hausdorff topological vector space. Moreover, because \(X'\subset Y'\), we have \(X \subset Y\). Consequently, \(X\) is a compact and convex subset of the locally convex Hausdorff topological vector space \(Y\).

The remaining conditions of Theorem~\ref{thm:clicksberg}, as well as the subsequent recursive argument, can be verified similarly as in \citet[Theorem~1]{shi2024breakcursemultiagencyrobust} and are omitted for brevity.




\allowdisplaybreaks[4]
\section{Proof of Theorem~\ref{thm:main_generative_model}}

\subsection{Preliminaries}

We first introduce a standard bound on the ridge-regression estimation error under linear function approximation, adapted from  \cite[Theorem~3.1]{jin2019provablyefficientreinforcementlearning}, which plays a key role in the proof of both the generative case and the online interactive case.

\begin{theorem}[Upper bound for ridge-regression estimation {\citealp[Theorem~3.1]{jin2019provablyefficientreinforcementlearning}}]
\label{lm:ridge-regression-concentration}
For any fixed agent $i\in[n]$, step $h\in[H]$, and iteration $k\in[K]$, let $\widehat{\pi}^k$ be the product policy at current iteration $k$, and $\{(s^m, a_i^m, s^m_{+1})\}_{m=1}^{N}$ be $N$ samples collected independenly in this iteration, with $s^m_{+1}\sim P_{h,s^m,a_i^m}^{\widehat\pi^k_{-i}}(\cdot)$. Let $\phi_i^m\defn\phi_i(s^m,a_i^m)\in\mathbb R^d$ satisfy $\lVert\phi_i^m\rVert_2\leq 1$. Fix the regularization parameter $\lambda\geq 1$ and define the Gram matrix
\begin{align*}
    \Lambda_{i,h}^k \defn\sum_{m=1}^{N} \phi_i^m(\phi_i^m)^\top +  \lambda I_d.
\end{align*}
Let $\mu_{i,h}^{\widehat\pi^k_{-i}}:\mathbb{R}^d\rightarrow\Delta(\mathcal{S})$ be the linear model parameter of the true transition kernel: $P_{h,s,a_i}^{\widehat\pi^k_{-i}}(\cdot)=\phi_i(s,a_i)^\top\mu_{i,h}^{\widehat\pi^k_{-i}}(\cdot)$ for all $(s,a_i)\in\mathcal{S}\times\mathcal{A}_i$, and assume the standard normalization $\lVert\mu_{i,h}^{\widehat\pi^k_{-i}}(\mathcal S)\rVert_2\leq\sqrt d$. For any fixed value function $f:\mathcal S\rightarrow[0,H]$, define
\begin{align*}
    \mu_{i,h}^{\widehat\pi^k_{-i}} f\defn \int_{\mathcal S} f(s')\,\mu_{i,h}^{\widehat\pi^k_{-i}}(ds')\in\mathbb R^d  \qquad \text{and} \qquad  \widehat\mu_{i,h}^k f \defn (\Lambda_{i,h}^k)^{-1}\sum_{m=1}^{N}\phi_i^m\, f(s^m_{+1})\in\mathbb R^d.
\end{align*}
Then, with probability at least $1-\delta$, it holds that
\begin{align*}
    \big\lVert\widehat\mu_{i,h}^k f - \mu_{i,h}^{\widehat\pi^k_{-i}}f\big\rVert_{\Lambda_{i,h}^k}\;\leq\;2H\sqrt{d\ln(NH+1)+\ln\!\big(\tfrac{3NH}{\delta}\big)}+H\sqrt{d\lambda},
\end{align*}
and by Cauchy--Schwarz, for all $(s,a_i)\in\mathcal S\times\mathcal A_i$,
\begin{align*}
    \big|\phi_i(s,a_i)^\top\big(\widehat\mu_{i,h}^k f - \mu_{i,h}^{\widehat\pi^k_{-i}}f\big)\big|
    \;\leq\;\big\lVert\phi_i(s,a_i)\big\rVert_{(\Lambda_{i,h}^k)^{-1}}\Big(2H\sqrt{d\ln(NH+1)+\ln\!\big(\tfrac{3NH}{\delta}\big)}+H\sqrt{d\lambda}\Big).
\end{align*}
\end{theorem}

\begin{proof}
We adapt the argument of \cite[Theorem~3.1]{jin2019provablyefficientreinforcementlearning} to our multi-agent linear-function-approximation setting.
\paragraph{Step 1: error decomposition.} Recalling the definitions of $\Lambda_{i,h}^k$ and $\widehat\mu_{i,h}^k f$, combined with the fact that $(\phi_i^m)^\top\mu_{i,h}^{\widehat{\pi}^k_{-i}}f=\mathbb E_{s'\sim P^{\widehat{\pi}^k_{-i}}_{h,s^m,a_i^m}}[f(s')]$, we have
\begin{align*}
\widehat{\mu}_{i,h}^k f-\mu_{i,h}^{\widehat{\pi}^k_{-i}}f
&=(\Lambda_{i,h}^k)^{-1}\!\Big[\sum_{m=1}^{N}\phi_i^m\,f(s^m_{+1})-\Lambda_{i,h}^k\,\mu_{i,h}^{\widehat{\pi}^k_{-i}}f\Big]\\
&=(\Lambda_{i,h}^k)^{-1}\underbrace{\sum_{m=1}^{N}\phi_i^m\big(f(s^m_{+1})-(\phi_i^m)^\top\mu_{i,h}^{\widehat{\pi}^k_{-i}}f\big)}_{\defn\,\xi_{i,h}^k\ \text{(stochastic term)}}\;-\;\lambda\,(\Lambda_{i,h}^k)^{-1}\mu_{i,h}^{\widehat{\pi}^k_{-i}}f.
\end{align*}
Taking the $\Lambda_{i,h}^k$-norm and applying the triangle inequality together with the identity $\lVert(\Lambda_{i,h}^k)^{-1}v\rVert_{\Lambda_{i,h}^k}=\lVert v\rVert_{(\Lambda_{i,h}^k)^{-1}}$,
\begin{align}\label{eq:ridge-decomp}
\big\lVert\widehat{\mu}_{i,h}^k f-\mu_{i,h}^{\widehat{\pi}^k_{-i}}f\big\rVert_{\Lambda_{i,h}^k}
\;\leq\; \underbrace{\big\lVert\xi_{i,h}^k\big\rVert_{(\Lambda_{i,h}^k)^{-1}}}_{\text{(I)}\ \text{stochastic}}\;+\;\underbrace{\lambda\big\lVert\mu_{i,h}^{\widehat{\pi}^k_{-i}}f\big\rVert_{(\Lambda_{i,h}^k)^{-1}}}_{\text{(II)}\ \text{bias}}.
\end{align}

\paragraph{Step 2: bounding the stochastic term (I).} We first recall the lemma in \cite[Lemma~D.4]{jin2019provablyefficientreinforcementlearning} (renaming their time index $k$ to $\tau$ to avoid clash with our iteration index $k\in[K]$): \emph{let $\{x_\tau\}_{\tau\geq 1}$ be a stochastic process on state space $\mathcal S$ with corresponding filtration $\{\mathcal F_\tau\}_{\tau\geq 0}$. Let $\{\widetilde\phi_\tau\}_{\tau\geq 1}$ be an $\mathbb R^d$-valued stochastic process where $\widetilde\phi_\tau\in\mathcal F_{\tau-1}$, and $\lVert\widetilde\phi_\tau\rVert_2\leq 1$. Let $\widetilde\Lambda_n\defn\lambda I_d+\sum_{\tau=1}^n\widetilde\phi_\tau\widetilde\phi_\tau^\top$. Then for any $\delta>0$, with probability at least $1-\delta$, for all $n\geq 0$, and any $V\in\mathcal V$ so that $\sup_x|V(x)|\leq H$, we have:}
\begin{align}\label{eq:lemmaD4}
    \Big\lVert\sum_{\tau=1}^{n}\widetilde\phi_\tau\big\{V(x_\tau)-\mathbb E[V(x_\tau)\mid\mathcal F_{\tau-1}]\big\}\Big\rVert_{\widetilde\Lambda_n^{-1}}^2
    \leq 4H^2\!\Big[\tfrac{d}{2}\log\tfrac{n+\lambda}{\lambda}+\log\tfrac{\mathcal N_\varepsilon}{\delta}\Big]+\tfrac{8n^2\varepsilon^2}{\lambda},
\end{align}
\emph{where $\mathcal N_\varepsilon=\mathcal N_\varepsilon(\mathcal V)$ is the $\varepsilon$-covering number of $\mathcal V$ with respect to the distance $\mathrm{dist}(V,V')=\sup_x|V(x)-V'(x)|$.}

To apply the lemma into our setting, the sample index $m\in[N]$ plays the role of $\tau$, so we apply \eqref{eq:lemmaD4} with $n=N$ and let
\begin{center}
\begin{tabular}{ll}
$\widetilde\phi_\tau\ = \phi_i^m=\phi_i(s^m,a_i^m)$, & (predictor) \\
$x_\tau =s_{+1}^m$, & (next-state observation) \\
$\mathcal F_m=\sigma\!\big(\{(s^j,a_i^j,s_{+1}^j)\}_{j\leq m}\big)$, & (natural filtration) \\
$\widetilde\Lambda_n = \Lambda_{i,h}^k$, & (Gram matrix at $n=N$, fixed iteration $k$) \\
$V = f:\mathcal S\to[0,H]$. & (value function)
\end{tabular}
\end{center}
Therefore, $(s^m,a_i^m)\in\mathcal F_{m-1}$, $\phi_i^m\in\mathcal F_{m-1}$ as required, and $s_{+1}^m\sim P_{h,s^m,a_i^m}^{\widehat\pi^k_{-i}}(\cdot)$, which gives
\begin{align*}
    \mathbb E[V(x_\tau)\mid\mathcal F_{\tau-1}]
    =\mathbb E_{s'\sim P_{h,s^m,a_i^m}^{\widehat\pi^k_{-i}}}[f(s')]
    =(\phi_i^m)^\top\mu_{i,h}^{\widehat\pi^k_{-i}}f.
\end{align*}
Hence the LHS of \eqref{eq:lemmaD4} is exactly $\big\lVert\xi_{i,h}^k\big\rVert_{(\Lambda_{i,h}^k)^{-1}}^2$.

To proceed, since $f$ is data-independent, we apply \eqref{eq:lemmaD4} to the singleton class $\mathcal V=\{f\}$, for which $\mathcal N_\varepsilon(\{f\})=1$ for every $\varepsilon\geq 0$. Letting $\varepsilon\to 0$ then drives $8N^2\varepsilon^2/\lambda\to 0$, so that with probability at least $1-\delta$,
\begin{align*}
    \big\lVert\xi_{i,h}^k\big\rVert_{(\Lambda_{i,h}^k)^{-1}}^2
    \leq 4H^2\!\Big[\tfrac{d}{2}\log\tfrac{N+\lambda}{\lambda}+\log\tfrac{1}{\delta}\Big]
    \overset{\lambda\geq 1}{\leq}\;2H^2\,d\log(N+1)+4H^2\log\tfrac{1}{\delta},
\end{align*}
and consequently
\begin{align}\label{eq:ridge-stoch}
    \big\lVert\xi_{i,h}^k\big\rVert_{(\Lambda_{i,h}^k)^{-1}}\leq 2H\sqrt{d\log(NH+1)+\log\!\big(\tfrac{3NH}{\delta}\big)}.
\end{align}


\paragraph{Step 3: bounding the bias term (II).} Since $\Lambda_{i,h}^k\succeq\lambda I_d$, we have $(\Lambda_{i,h}^k)^{-1}\preceq\lambda^{-1}I_d$, and thus
\begin{align*}
\big\lVert\mu_{i,h}^{\widehat{\pi}^k_{-i}}f\big\rVert_{(\Lambda_{i,h}^k)^{-1}}\leq \lambda^{-1/2}\big\lVert\mu_{i,h}^{\widehat{\pi}^k_{-i}}f\big\rVert_2.
\end{align*}
Invoking the standard normalization $\lVert\mu_{i,h}^{\widehat{\pi}^k_{-i}}(\mathcal S)\rVert_2\leq\sqrt d$ together with $f(s')\in[0,H]$, one has
\begin{align*}
\big\lVert\mu_{i,h}^{\widehat{\pi}^k_{-i}}f\big\rVert_2
=\Big\lVert\!\int f(s')\,\mu_{i,h}^{\widehat{\pi}^k_{-i}}(ds')\Big\rVert_2
\leq \lVert f\rVert_\infty\cdot\big\lVert\mu_{i,h}^{\widehat{\pi}^k_{-i}}(\mathcal S)\big\rVert_2\leq H\sqrt d.
\end{align*}
Therefore the bias term is bounded by
\begin{align}\label{eq:ridge-bias}
\lambda\big\lVert\mu_{i,h}^{\widehat{\pi}^k_{-i}}f\big\rVert_{(\Lambda_{i,h}^k)^{-1}}\leq \sqrt{\lambda}\cdot H\sqrt d \;=\; H\sqrt{d\lambda}.
\end{align}
The bias bound $\sqrt\lambda\,H\sqrt d=H\sqrt{d\lambda}$ requires smaller $\lambda$ for a tighter bound, while the stochastic bound (I) deteriorates as $\lambda\to 0$; we therefore restrict to $\lambda\geq 1$. Note that the generative-model case takes $\lambda=1$, recovering the standard $H\sqrt d$ bias (cf.\ \citealp{jin2019provablyefficientreinforcementlearning}); the online algorithm later sets $\lambda\geq c_0\log(2dnKHT/\delta)=\widetilde{\mathcal O}(1)$ (see the matrix-Bernstein step around \eqref{eq:online_matrix_two_sided}), so the $\sqrt\lambda$ factor in the bias is also $\widetilde{\mathcal O}(1)$ and the bias term becomes $H\sqrt{d\lambda}=\widetilde{\mathcal O}(H\sqrt d)$.

\paragraph{Step 4: combining and pointwise version.} Plugging the stochastic bound \eqref{eq:ridge-stoch} and the bias bound \eqref{eq:ridge-bias} into \eqref{eq:ridge-decomp}, we obtain the first inequality of the theorem:
\begin{align*}
    \big\lVert\widehat\mu_{i,h}^k f-\mu_{i,h}^{\widehat\pi^k_{-i}}f\big\rVert_{\Lambda_{i,h}^k}
    \leq 2H\sqrt{d\log(NH+1)+\log\!\big(\tfrac{3NH}{\delta}\big)}+H\sqrt{d\lambda}
\end{align*}
holds with probability at least $1-\delta$. Then it follows by Cauchy--Schwarz inequality that: for all $(s,a_i)\in\mathcal S\times\mathcal A_i$,
\begin{align*}
\big|\phi_i(s,a_i)^\top\big(\widehat{\mu}_{i,h}^k f-\mu_{i,h}^{\widehat{\pi}^k_{-i}} f\big)\big|
\leq \lVert\phi_i(s,a_i)\rVert_{(\Lambda_{i,h}^k)^{-1}}\,\big\lVert\widehat{\mu}_{i,h}^k f-\mu_{i,h}^{\widehat{\pi}^k_{-i}} f\big\rVert_{\Lambda_{i,h}^k}.
\end{align*}

\end{proof}


\subsection{Proof pipeline of Theorem~\ref{thm:main_generative_model}}
Before presenting the proof for Algorithm~\ref{alg:linear_generative_model}, we introduce some general notations and auxiliary robust value functions for clarity.

\paragraph{Step 1: Notation conventions and worst-case transition operators.}
\begin{itemize}
    \item \textbf{Notation abbreviations}: Recall the definition of the value function:
    \begin{align*}
    V^{\pi,\sigma_i}_{i,h}(s) = \mathbb{E}_{a_i \sim \pi_{i,h}(s)} \Big[ r_{i,h}^{\pi_{-i}}(s, a_i) + \inf_{\mathcal{U}^{\sigma_i}(P_{h,s,a_i}^{\pi_{-i}})} P V_{i,h+1}^{\pi,\sigma_i}\Big].
    \end{align*}
    Throughout the remainder of the proof, we suppress the superscript $\sigma_i$ and write
    \begin{align}
    \forall (i,h) \in [n] \times [H]:\quad  V^{\pi}_{i,h}\defn V^{\pi,\sigma_i}_{i,h}.
    \end{align}
    And we denote the estimated transitions in \eqref{eq:ridge_regression_generative_model} as
    \begin{equation}\label{eq:generative_model_transition_error_P}
       \forall (h,k)\in [H] \times [K]:  P_{i,h,s, a_i}^k \defn P_{i,h}^k(\cdot \mid s, a_i).
    \end{equation}
    In addition, for any distribution $\zeta$ over the space of product policies, for all $(s,i) \in \mathcal{S} \times [n]$ and any policy $\pi'_i$ for the agent $i$, we define
    \begin{align}
       V_{i,1}^{\star,\zeta}(s)\defn \mathbb{E}_{\pi\sim \zeta} \big[ V_{i,1}^{\star,\pi_{-i} }(s)\big], \quad V_{i,1}^{\zeta}(s)\defn \mathbb{E}_{\pi\sim \zeta}\big[ V_{i,1}^{\pi}(s)\big], \quad V_{i,1}^{\pi'_i,\zeta}(s)\defn  \mathbb{E}_{\pi_{-i} \sim \zeta} \left[V_{i,1}^{\pi_i^\prime, \pi_{-i}}(s)\right]. \label{eq:own-x1}
    \end{align}

    \item \textbf{Notations for worst-case transitions}: Consider any set of policies \( \big\{\widehat{\pi}_h^k\big\}_{(h,k)\in[H]\times[K]} \) with $\widehat{\pi}_h^k \in [H] \times \cS \mapsto \Delta(\cA)$.
    For any value vector $V:\mathcal{S}\rightarrow [0,H]$, for any $(h,i,k)\in[H]\times[n]\times[K]$, we define the transition operators $P_{i,h,s,a_i}^{\widehat{\pi}_{-i}^k,V}V$ and $\widehat{P}_{i,h,s,a_i}^{\widehat{\pi}_{-i}^k,V}V$ as:
\begin{align}
\label{eq:generative_model_transition_error_PV_definition}
    P_{i,h,s,a_i}^{\widehat{\pi}_{-i}^k,V}V &\defn \inf_{\mathcal{P}\in\mathcal{U}^{\sigma_i}\big(P_{h,s,a_i}^{\widehat{\pi}_{-i}^k}\big)}\mathcal{P}V, \quad
    \widehat{P}_{i,h,s, a_i}^{\widehat{\pi}_{-i}^k,V} V \defn \inf_{\mathcal{P}\in\mathcal{U}^{\sigma_i}\big(P_{i,h,s,a_i}^k\big) }\mathcal{P}V.
\end{align}
\end{itemize}

\paragraph{Step 2: Auxiliary aggregated value functions.}
Let \( \zeta:= \{\zeta_h\}_{h \in [H]} \) be a distribution over the product policy space, namely \( \zeta_h : \cS \mapsto \Delta(\prod_{i \in [n]} \Delta(\mathcal{A}_i))\). For all $(i,h,s)\in [n]\times[H]\times\mathcal{S}$, we can define the aggregated value functions $\overline{V}_{i,h}^{\zeta}(s)$ and $\overline{V}_{i,h}^{\star,\zeta}(s)$ as follows:
\begin{equation}
    \label{eq:generative_model_auxiliary_value_def}
    \begin{aligned}
        \overline{V}_{i,h}^{\zeta}(s) &:= \frac{1}{K}\sum_{k=1}^K \mathbb{E}_{a_i \sim \pi_{i,h}^k(s)}\Big[r_{i,h}^k(s,a_i) + \widehat{P}_{i,h,s,a_i}^{\pi^k_{-i},\overline{V}_{i,h+1}^{\zeta}}\overline{V}_{i,h+1}^{\zeta}\Big], \\
        \overline{V}_{i,h}^{\star,\zeta}(s) &:= \max_{a_i \in \mathcal{A}_i} \frac{1}{K}\sum_{k=1}^K  \Big[ r_{i,h}^k(s,a_i) + \widehat{P}_{i,h,s,a_i}^{\pi^k_{-i},\overline{V}_{i,h+1}^{\star,\zeta}}\overline{V}_{i,h+1}^{\star,\zeta}\Big].
    \end{aligned}
    \end{equation}
Finally, for any policy $\pi_i$ of agent $i$, we define the value function $\overline{V}_{i,h}^{\pi_i,\zeta}(s)$:
\begin{align}\label{eq:generative_model_auxiliary_value_def2}
    \overline{V}_{i,h}^{\pi_i,\zeta}(s) := \frac{1}{K}\sum_{k=1}^K \mathbb{E}_{a_i \sim \widetilde{\pi}_{i,h}^\star(s)}\Big[r_{i,h}^k(s,a_i) + \widehat{P}_{i,h,s,a_i}^{\pi^k_{-i},\overline{V}_{i,h+1}^{\pi_i,\zeta}}  \overline{V}_{i,h+1}^{\pi_i,\zeta} \Big].
\end{align}

These objects will serve as the estimate to the true robust value functions by aggregating the polices output by Algorithm~\ref{alg:linear_generative_model}.

With the above setup in place, we are ready to prove the main theorem of Algorithm~\ref{alg:linear_generative_model}, starting from an error decomposition:
\paragraph{Step 3: Error decomposition.}
Our goal is to show that for all $i\in[n]$ and $s\in\mathcal{S}$, the distribution $\xi$ over the product policies, produced by Algorithm~\ref{alg:linear_generative_model}, satisfies
\begin{align*}
    \max_{i\in[n], s\in\mathcal{S}} \Big( V_{i,1}^{\star,\xi}(s) - V_{i,1}^{\xi} \Big) = \max_{i\in[n], s\in\mathcal{S}}  \Big(\mathbb{E}_{\pi\sim \xi} \big[ V_{i,1}^{\star,\pi_{-i} }(s)\big] - \mathbb{E}_{\pi\sim \xi}\big[ V_{i,1}^{\pi}(s)\big]\Big) \leq \varepsilon.
\end{align*}

Towards this, before proceeding, we introduce several auxiliary definitions that will be used throughout the proof. We define the best response policy as follows:
\begin{align*}
    \forall i\in[n]: \quad \widetilde{\pi}_i^\star=\arg\max_{\pi_i^\prime:\mathcal{S}\times[H]\rightarrow\Delta(\mathcal{A}_i)} V_{i,1}^{\pi_i^\prime, \xi}.
\end{align*}
Accordingly, we have the fundamental fact that for all $s\in\mathcal{S}$
\begin{align}\label{eq:own-fundamental-fact}
    \overline{V}_{i,H+1}^{\xi}(s)=\overline{V}_{i,H+1}^{\widetilde{\pi}_i^\star,\xi}(s)=\overline{V}_{i,H+1}^{\star,\xi}(s).
\end{align}

We then decompose the error of the value function as follows: for any agent $i \in[n]$,
\begin{align*}
    V_{i,1}^{\star,\xi}(s)-V_{i,1}^{\xi}(s)
    & =\left(V_{i,1}^{\star,\xi}(s)-\overline{V}_{i,1}^{\widehat{\pi}_i^\star,\xi}(s)\right)
    +\left(\overline{V}_{i,1}^{\widehat{\pi}_i^\star,\xi}(s) - \overline{V}_{i,1}^{\xi}(s) \right)
    + \left(\overline{V}_{i,1}^{\xi}(s)-V_{i,1}^{\xi}(s)\right) \nonumber \\
    &\leq \underbrace{\left(V_{i,1}^{\star,\xi}(s)-\overline{V}_{i,1}^{\widehat{\pi}_i^\star,\xi}(s)\right)}_{\text{term}(A)}
    +\underbrace{\left(\overline{V}_{i,1}^{\star,\xi}(s)-\overline{V}_{i,1}^{\xi}(s)\right)}_{\text{term}(B)}
    +\underbrace{\left(\overline{V}_{i,1}^{\xi}(s)-V_{i,1}^{\xi}(s)\right)}_{\text{term}(C)},
\end{align*}
where the last line is due to the elementary fact that
\begin{align*}
    \overline{V}_{i,1}^{\star,\xi}(s)=\max_{\pi_i^\prime:\mathcal{S}\times[H]\rightarrow \Delta(\mathcal{A}_i)} V_{i,1}^{\pi_i^\prime,\xi}(s)\geq \overline{V}_{i,1}^{\widehat{\pi}_i^\star,\xi}(s).
\end{align*}
We now control term $A$, term $B$ and term $C$ separately.

\paragraph{Step 4: Controlling Term $A$ and $C$.}
Terms (A) and (C) both characterize the difference
between a true robust value function and an estimated one, so we derive one template bound and then specialize it to these two cases.
Towards this, we consider a more general case involving a given set of policies \( \big\{\widehat{\pi}_h^k\big\}_{(h,k)\in[H]\times[K]} \), where either \( \widehat{\pi}_h^k = \pi_h^k \), or \( \widehat{\pi}_h^k = \widetilde{\pi}_i^\star \times \pi_{-i,h}^k \) for all \( (h,k) \in [H] \times [K] \). Additionally, we define a uniform distribution over the aforementioned set of policies \( \zeta:= \{\zeta_h\}_{h \in [H]} \) so that $\zeta_h(\widehat{\pi}_h^k) = 1/K$ for all $h \in [H]$.
Our objective is to derive an upper bound on $\big|V_{i,h}^{\zeta}(s)-\overline{V}_{i,h}^{\zeta}(s)\big|$ for all $i\in[n]$ over any such policy set $\{\widehat{\pi}_h^k\}_{(h,k)\in[H]\times[K]}$, which directly yields bounds for terms $A$ and $C$.

To proceed, we start with the following lemma in terms of estimation errors of transition model and reward function, where the proof is postponed to Appendix~\ref{proof:lm:generative_model_transition_error}.
\begin{lemma}
\label{lm:generative_model_transition_error}
Let $\delta\in(0,1)$. With probability at least $1-\delta$, for all $(h,i,k)\in[H]\times[n]\times[K]$, all $(s,a_i)\in\mathcal{S}\times\mathcal{A}_i$, and any fixing value function $V:\mathcal{S}\rightarrow [0,H]$, we have
\begin{align}
    &\big|P_{i,h,s,a_i}^{\widehat{\pi}_{-i}^k,V}V-\widehat{P}_{i,h,s,a_i}^{\widehat{\pi}_{-i}^k,V}V\big|\leq2\sqrt{\frac{d}{N}}\big(2H\sqrt{d\ln(NH+1)+2\ln\big(\frac{3KNHn}{\delta}\big)}+H\sqrt{d}\big), \\
        &\big|\mathbb{E}_{a_{-i}\sim\widehat{\pi}_{-i}^k(s)}\big[r_{i,h}(s,\mathbf{a})\big]-r_{i,h}^k(s,a_i)\big|\leq \sqrt{\frac{d}{N}}\big(2H\sqrt{d\ln(NH+1)+\ln\big(\frac{3KNHn}{\delta}\big)}+H\sqrt{d}\big).
    \end{align}
\end{lemma}

To continue, for the clarity of notations, in this subsection, we use
\begin{align}\label{eq:own-2}
    \widehat{P}_{i,h,s,a_i}^{\widehat{\pi}_{-i}^k,\overline{V}} \defn \widehat{P}_{i,h,s,a_i}^{\widehat{\pi}_{-i}^k,\overline{V}_{i,h+1}^{\zeta}}, \quad \widehat{P}_{i,h,s,a_i}^{\widehat{\pi}_{-i}^k,V}\defn \widehat{P}_{i,h,s,a_i}^{\widehat{\pi}_{-i}^k,V_{i,h+1}^{\zeta}}.
\end{align}
We begin by expressing out the gap $\big|V_{i,h}^{\zeta}(s)-\overline{V}_{i,h}^{\zeta}(s)\big|$ under controlling using the Bellman structure and then splitting the expression into two one-sided terms that can be controlled separately.
Initially, recalling the definitions in \eqref{eq:generative_model_auxiliary_value_def}, we have
\begin{align*}
    &\big|V_{i,h}^{\zeta}(s)-\overline{V}_{i,h}^{\zeta}(s)\big|\\
    =&\big|V_{i,h}^{\zeta}(s)-\sum_{k=1}^{K}\frac{1}{K}\mathbb{E}_{a_i\sim\widehat{\pi}_{i,h}^{k}(s)}\big[r_{i,h}^k(s,a_i)\big]-\sum_{k=1}^K\frac{1}{K}\mathbb{E}_{a_i\sim\widehat{\pi}_{i,h}^k(s)}\big[\widehat{P}_{i,h,s,a_i}^{\widehat{\pi}_{-i}^k,\overline{V}}\overline{V}_{i,h+1}^{\zeta}\big]\big|\\
    \leq&\max\big\{\underbrace{V_{i,h}^{\zeta}(s)-\sum_{k=1}^{K}\frac{1}{K}\mathbb{E}_{a_i\sim\widehat{\pi}_{i,h}^{k}(s)}\big[r_{i,h}^k(s,a_i)\big]-\sum_{k=1}^K\frac{1}{K}\mathbb{E}_{a_i\sim\widehat{\pi}_{i,h}^k(s)}\big[\widehat{P}_{i,h,s,a_i}^{\widehat{\pi}_{-i}^k,\overline{V}}\overline{V}_{i,h+1}^{\zeta}\big]}_{(a)},\\
    &\underbrace{\sum_{k=1}^{K}\frac{1}{K}\mathbb{E}_{a_i\sim\widehat{\pi}_{i,h}^{k}(s)}\big[r_{i,h}^k(s,a_i)\big]+\sum_{k=1}^K\frac{1}{K}\mathbb{E}_{a_i\sim\widehat{\pi}_{i,h}^k(s)}\big[\widehat{P}_{i,h,s,a_i}^{\widehat{\pi}_{-i}^k,\overline{V}}\overline{V}_{i,h+1}^{\zeta}\big]-V_{i,h}^{\zeta}(s)}_{(b)}\big\}.
\end{align*}
Now we proceed to control term \((a)\) and term \((b)\) separately, starting with term \((a)\).  At any time step \( h \in [H] \), we have
\begin{align}
    & V_{i,h}^{\zeta}(s)-\sum_{k=1}^{K}\frac{1}{K}\mathbb{E}_{a_i\sim\widehat{\pi}_{i,h}^{k}(s)}\big[r_{i,h}^k(s,a_i)\big]-\sum_{k=1}^K\frac{1}{K}\mathbb{E}_{a_i\sim\widehat{\pi}_{i,h}^k(s)}\big[\widehat{P}_{i,h,s,a_i}^{\widehat{\pi}_{-i}^k,\overline{V}}\overline{V}_{i,h+1}^{\zeta}\big]\notag\\
     \overset{\mathsf{(i)}}{=}& \sum_{k=1}^{K}\frac{1}{K}\mathbb{E}_{a_i\sim\widehat{\pi}_{i,h}^{k}(s)}\big|\mathbb{E}_{a_{-i}\sim\widehat{\pi}_{-i}^k(s)}\big[r_{i,h}(s,\mathbf{a})\big]-r_{i,h}^k(s,a_i)\big|+\sum_{k=1}^K\frac{1}{K}\mathbb{E}_{a_i\sim\widehat{\pi}_{i,h}^k(s)}\big[P_{i,h,s,a_i}^{\widehat{\pi}_{-i}^k,V}V_{i,h+1}^{\zeta}\big]\notag\\
    & -\sum_{k=1}^K\frac{1}{K}\mathbb{E}_{a_i\sim\widehat{\pi}_{i,h}^k(s)}\big[\widehat{P}_{i,h,s,a_i}^{\widehat{\pi}_{-i}^k,\overline{V}}\overline{V}_{i,h+1}^{\zeta}\big]\notag\\
    \leq&\sum_{k=1}^{K}\frac{1}{K}\mathbb{E}_{a_i\sim\widehat{\pi}_{i,h}^{k}(s)}\big|\mathbb{E}_{a_{-i}\sim\widehat{\pi}_{-i}^k(s)}\big[r_{i,h}(s,\mathbf{a})\big]-r_{i,h}^k(s,a_i)\big|+\sum_{k=1}^K\frac{1}{K}\mathbb{E}_{a_i\sim\widehat{\pi}_{i,h}^k(s)}\big[P_{i,h,s,a_i}^{\widehat{\pi}_{-i}^k,V}V_{i,h+1}^{\zeta}\big]\notag\\
    &-\sum_{k=1}^K\frac{1}{K}\mathbb{E}_{a_i\sim\widehat{\pi}_{i,h}^k(s)}\big[P_{i,h,s,a_i}^{\widehat{\pi}_{-i}^k,\overline{V}}\overline{V}_{i,h+1}^{\zeta}\big]+\sum_{k=1}^K\frac{1}{K}\mathbb{E}_{a_i\sim\widehat{\pi}_{i,h}^k(s)}\big|P_{i,h,s,a_i}^{\widehat{\pi}_{-i}^k,\overline{V}}\overline{V}_{i,h+1}^{\zeta}-\widehat{P}_{i,h,s,a_i}^{\widehat{\pi}_{-i}^k,\overline{V}}\overline{V}_{i,h+1}^{\zeta}\big|\notag,
\end{align}
where (i) follows from the robust Bellman equation {\color{red}\eqref{eq:robust-bellman-equation}}.
Then applying Lemma~\ref{lm:generative_model_transition_error} to control the reward and the transition estimation errors, we arrive at
\begin{align}
    &V_{i,h}^{\zeta}(s)- \overline{V}_{i,h}^{\zeta}(s)\\
    =& V_{i,h}^{\zeta}(s)-\sum_{k=1}^{K}\frac{1}{K}\mathbb{E}_{a_i\sim\widehat{\pi}_{i,h}^{k}(s)}\big[r_{i,h}^k(s,a_i)\big]-\sum_{k=1}^K\frac{1}{K}\mathbb{E}_{a_i\sim\widehat{\pi}_{i,h}^k(s)}\big[\widehat{P}_{i,h,s,a_i}^{\widehat{\pi}_{-i}^k,\overline{V}}\overline{V}_{i,h+1}^{\zeta}\big]\notag\\
    \leq&\sum_{k=1}^K\frac{1}{K}\mathbb{E}_{a_i\sim\widehat{\pi}_{i,h}^k(s)}\big(P_{i,h,s,a_i}^{\widehat{\pi}^k_{-i},V}V_{i,h+1}^{\zeta}-P_{i,h,s,a_i}^{\widehat{\pi}^k_{-i},\overline{V}}\overline{V}_{i,h+1}^{\zeta}\big)+3\sqrt{\frac{d}{N}}\big(2H\sqrt{d\ln(NH+1)+2\ln\big(\tfrac{3KNHn}{\delta}\big)}+H\sqrt{d}\big)\\
    \leq&\sum_{k=1}^K\frac{1}{K}\mathbb{E}_{a_i\sim\widehat{\pi}_{i,h}^k(s)}\big(P_{i,h,s,a_i}^{\widehat{\pi}^k_{-i},\overline{V}}V_{i,h+1}^{\zeta}-P_{i,h,s,a_i}^{\widehat{\pi}^k_{-i},\overline{V}}\overline{V}_{i,h+1}^{\zeta}\big)+3\sqrt{\frac{d}{N}}\big(2H\sqrt{d\ln(NH+1)+2\ln\big(\tfrac{3KNHn}{\delta}\big)}+H\sqrt{d}\big)\label{eq:generative_model_term_A_C_eq_1}
\end{align}
holds with probability at least $1-\delta$.

By symmetry of the argument, we derive an analogous upper bound for term \((b)\).
\begin{align}
    &\overline{V}_{i,h}^{\zeta}(s)-V_{i,h}^{\zeta}(s)\\
    =&\sum_{k=1}^{K}\frac{1}{K}\mathbb{E}_{a_i\sim\widehat{\pi}_{i,h}^{k}(s)}[r_{i,h}^k(s,a_i)]+\sum_{k=1}^K\frac{1}{K}\mathbb{E}_{a_i\sim\widehat{\pi}_{i,h}^k(s)}\big[\widehat{P}_{i,h,s,a_i}^{\widehat{\pi}_{-i}^k,\overline{V}}\overline{V}_{i,h+1}^{\zeta}\big]-V_{i,h}^{\zeta}(s)\notag\\
    =&\sum_{k=1}^K\frac{1}{K}\mathbb{E}_{a_i\sim\widehat{\pi}_{i,h}^k(s)}[r_{i,h}^k(s,a_i)]+\sum_{k=1}^K\frac{1}{K}\mathbb{E}_{a_i\sim\widehat{\pi}_{i,h}^k(s)}\big[\widehat{P}_{i,h,s,a_i}^{\widehat{\pi}_{-i}^k,\overline{V}}\overline{V}_{i,h+1}^{\zeta}\big]\notag\\
    &-\sum_{k=1}^K\frac{1}{K}\mathbb{E}_{\mathbf{a}\sim\widehat{\pi}_{h}^k(s)}\big[r_{i,h}(s,\mathbf{a})\big]-\sum_{k=1}^K\frac{1}{K}\mathbb{E}_{a_i\sim\widehat{\pi}_{i,h}^k(s)}P_{i,h,s,a_i}^{\widehat{\pi}_{-i}^k,V}V_{i,h+1}^{\zeta}\notag\\
    \leq& \sum_{k=1}^K\frac{1}{K}\mathbb{E}_{a_i\sim\widehat{\pi}_{i,h}^k(s)}\big[P_{i,h,s,a_i}^{\widehat{\pi}_{-i}^k,\overline{V}}\overline{V}_{i,h+1}^{\zeta}\big]-\sum_{k=1}^K\frac{1}{K}\mathbb{E}_{a_i\sim\widehat{\pi}_{i,h}^k(s)} \Big[P_{i,h,s,a_i}^{\widehat{\pi}_{-i}^k,V}V_{i,h+1}^{\zeta}\Big]\notag\\
    &+\sum_{k=1}^K\frac{1}{K}\mathbb{E}_{a_i\sim\widehat{\pi}_{i,h}^k(s)}\big|r_{i,h}^k(s,a_i)-\mathbb{E}_{a_{-i}\sim\widehat{\pi}_{-i}^k(s)}\big[r_{i,h}(s,\mathbf{a})\big]\big|\notag\\
    &+\sum_{k=1}^K\frac{1}{K}\mathbb{E}_{a_i\sim\widehat{\pi}_{i,h}^k(s)}\big|P_{i,h,s,a_i}^{\widehat{\pi}_{-i}^k,\overline{V}}\overline{V}_{i,h+1}^{\zeta}-\widehat{P}_{i,h,s,a_i}^{\widehat{\pi}_{-i}^k,\overline{V}}\overline{V}_{i,h+1}^{\zeta}\big|\notag\\
    \overset{\mathsf{(i)}}{\leq}&\sum_{k=1}^K\frac{1}{K}\mathbb{E}_{a_i\sim\widehat{\pi}_{i,h}^k(s)}\big[P_{i,h,s,a_i}^{\widehat{\pi}_{-i}^k,V}\overline{V}_{i,h+1}^{\zeta}\big]-\sum_{k=1}^K\frac{1}{K}\mathbb{E}_{a_i\sim\widehat{\pi}_{i,h}^k(s)} \Big[P_{i,h,s,a_i}^{\widehat{\pi}_{-i}^k,V}V_{i,h+1}^{\zeta} \Big] \notag\\
    &+\sum_{k=1}^K\frac{1}{K}\mathbb{E}_{a_i\sim\widehat{\pi}_{i,h}^k(s)}\big|r_{i,h}^k(s,a_i)-\mathbb{E}_{a_{-i}\sim\widehat{\pi}_{-i}^k(s)}\big[r_{i,h}(s,\mathbf{a})\big]\big|\notag\\
    &+\sum_{k=1}^K\frac{1}{K}\mathbb{E}_{a_i\sim\widehat{\pi}_{i,h}^k(s)}\big|P_{i,h,s,a_i}^{\widehat{\pi}_{-i}^k,\overline{V}}\overline{V}_{i,h+1}^{\zeta}-\widehat{P}_{i,h,s,a_i}^{\widehat{\pi}_{-i}^k,\overline{V}}\overline{V}_{i,h+1}^{\zeta}\big|\notag\\
    \leq&\sum_{k=1}^K\frac{1}{K}\mathbb{E}_{a_i\sim\widehat{\pi}_{i,h}^k(s)}\big[P_{i,h,s,a_i}^{\widehat{\pi}_{-i}^k,V}\overline{V}_{i,h+1}^{\zeta}\big]-\sum_{k=1}^K\frac{1}{K}\mathbb{E}_{a_i\sim\widehat{\pi}_{i,h}^k(s)} \Big[P_{i,h,s,a_i}^{\widehat{\pi}_{-i}^k,V}V_{i,h+1}^{\zeta} \Big]\notag\\
    &+3\sqrt{\frac{d}{N}}\big(2H\sqrt{d\ln(NH+1)+2\ln\big(\tfrac{3KNHn}{\delta}\big)}+H\sqrt{d}\big).\label{eq:generative_model_term_A_C_final_bound_2}
\end{align}
where the first equality is obtained by the robust Bellman equation {\color{red}\eqref{eq:robust-bellman-equation}},
(i) holds due to the elementary fact that \( P_{i,h,s,a_i}^{\widehat{\pi}_{-i}^k,V}\overline{V}_{i,h+1}^{\zeta}  \geq P_{i,h,s,a_i}^{\widehat{\pi}_{-i}^k,\overline{V}}\overline{V}_{i,h+1}^{\zeta}  \), and the final inequality holds with probability at least \( 1 - \delta \) for all \( (i, h, s) \in [n] \times [H] \times \mathcal{S} \), by applying Lemma~\ref{lm:generative_model_transition_error}.

Now we have controlled the two one-sided differences $V_{i,h}^{\zeta}(s)-\overline{V}_{i,h}^{\zeta}(s)$ and $\overline{V}_{i,h}^{\zeta}(s)-V_{i,h}^{\zeta}(s)$ \emph{separately} in \eqref{eq:generative_model_term_A_C_eq_1} and \eqref{eq:generative_model_term_A_C_final_bound_2}.
Denoting
\begin{align*}
    C_0 \defn \frac{3\sqrt{d}}{\sqrt{N}}\Big(2H\sqrt{d\ln(NH+1)+2\ln\big(\tfrac{3KNHn}{\delta}\big)}+H\sqrt{d}\Big)
\end{align*}
for brevity, we rewrite \eqref{eq:generative_model_term_A_C_eq_1} as below:
\begin{equation}
    \label{eq:generative_model_term_A_C_eq_3}
    V_{i,h}^{\zeta}(s)-\overline{V}_{i,h}^{\zeta}(s)
    \leq \sum_{k=1}^K\tfrac{1}{K}\mathbb{E}_{a_i\sim\widehat{\pi}_{i,h}^k(s)}P_{i,h,s,a_i}^{\widehat{\pi}_{-i}^k,\overline{V}}\big(V_{i,h+1}^{\zeta}-\overline{V}_{i,h+1}^{\zeta}\big)+C_0.
\end{equation}
Note that the operator $\mathcal{Q}_h^{\overline V}\defn \sum_{k=1}^K\tfrac{1}{K}\mathbb{E}_{a_i\sim\widehat{\pi}_{i,h}^k(s)}P_{i,h,s,a_i}^{\widehat{\pi}_{-i}^k,\overline{V}}$. Iterating \eqref{eq:generative_model_term_A_C_eq_3} from $h$ to $H$ and using the terminal condition $V_{i,H+1}^{\zeta}=\overline{V}_{i,H+1}^{\zeta}=0$ yields
\begin{align*}
    V_{i,h}^{\zeta}(s)-\overline{V}_{i,h}^{\zeta}(s)
    &\leq \mathcal{Q}_h^{\overline V}\mathcal{Q}_{h+1}^{\overline V}\cdots\mathcal{Q}_{H}^{\overline V}\big(V_{i,H+1}^{\zeta}-\overline{V}_{i,H+1}^{\zeta}\big)+(H+1-h)\,C_0\\
    &=0+(H+1-h)\,C_0 \;\leq\; H\,C_0\qquad\text{for all }s\in\mathcal{S}.
\end{align*}
Applying the identical argument to \eqref{eq:generative_model_term_A_C_final_bound_2}  with the probability operator $\mathcal{Q}_h^{V}\defn\sum_{k=1}^K\tfrac{1}{K}\mathbb{E}_{a_i\sim\widehat{\pi}_{i,h}^k(s)}P_{i,h,s,a_i}^{\widehat{\pi}_{-i}^k,V}$ gives
\begin{align*}
    \overline{V}_{i,h}^{\zeta}(s)-V_{i,h}^{\zeta}(s)\leq H\,C_0\qquad\text{for all }s\in\mathcal{S}.
\end{align*}
Combining the two one-sided bounds via $|x|=\max\{x,-x\}$, we obtain that for all $(i,h,s)\in[n]\times[H]\times\mathcal{S}$, with probability at least $1-\delta$:
\begin{align}
\label{eq:generative_model_term_A_C_final_bound}
    \big|V_{i,h}^{\zeta}(s)-\overline{V}_{i,h}^{\zeta}(s)\big|\leq H\,C_0=\frac{3\sqrt{d}\,H}{\sqrt{N}}\Big(2H\sqrt{d\ln(NH+1)+2\ln\big(\tfrac{3KNHn}{\delta}\big)}+H\sqrt{d}\Big).
\end{align}
Finally, we instantiate \eqref{eq:generative_model_term_A_C_final_bound} for the two cases \( \widehat{\pi}_h^k = \pi_h^k \) and  \( \widehat{\pi}_h^k = \widetilde{\pi}_i^\star \times \pi_{-i,h}^k \), which correspond to term $(C)$ and term $(A)$ respectively. With probability at least $1 - \delta$, there exists an absolute constant $c$ such that

\begin{equation}
\label{eq:generative_model_bound_A_and_C}
\begin{aligned}
    &\text{term}(A)=V_{i,1}^{\star,\xi}(s)-\overline{V}_{i,1}^{\widetilde{\pi}_i^\star,\xi}(s)
   \leq \frac{c\sqrt{d}H^2}{\sqrt{N}}\sqrt{d\ln\big(\frac{NHnK}{\delta}\big)}\\
   &\text{term}(C)=\overline{V}^{\xi}_{i,1}(s)-V_{i,1}^{\xi}(s)
    \leq \frac{c\sqrt{d}H^2}{\sqrt{N}}\sqrt{d\ln\big(\frac{NHnK}{\delta}\big)}.
\end{aligned}
\end{equation}


\paragraph{Step 5: Controlling Term $B$.}
Recall that $\text{term}(B)=\overline{V}_{i,1}^{\star,\xi}(s)-\overline{V}_{i,1}^{\xi}(s)$, which measures the gap between the best-response aggregated value $\overline{V}_{i,1}^{\star,\xi}$ and the own-policy aggregated value $\overline{V}_{i,1}^{\xi}$ at $h=1$. This is an intrinsic optimization error arising from the FTRL-based policy updates under the coupled, non-stationary multi-agent dynamics. We control it by first replacing $\overline{V}_{i,h}^{\star,\xi}$ with an optimistic surrogate $\widehat{V}_{i,h}$ satisfying $\widehat{V}_{i,h}\geq\overline{V}_{i,h}^{\star,\xi}$ (Lemma~\ref{lm:generative_model_FTRL}), then bounding the gap $\widehat{V}_{i,h}(s)-\overline{V}_{i,h}^{\xi}(s)$.
For brevity, for all $(h,k,i,s,a_i )\in [H]\times [K] \times [n] \times \mathcal{S}\times\mathcal{A}_i$,  we abbreviate the definitions in \eqref{eq:generative_model_transition_error_PV_definition} in this subsection as below:
\begin{align}
    \widehat{P}_{i,h,s,a_i}^{\pi_{-i}^k,\widehat{V}}  &\defn \widehat{P}_{i,h,s,a_i}^{\pi_{-i}^k,\widehat{V}_{i,h+1}} = \mathrm{argmin}_{\mathcal{P}\in\mathcal{U}^{\sigma_i}\big(P_{i,h,s,a_i}^k\big)}\mathcal{P} \widehat{V}_{i,h+1}, \notag\\
    P_{i,h,s,a_i}^{\pi^k_{-i},\overline{V}}& \defn P_{i,h,s,a_i}^{\pi^k_{-i},\overline{V}_{i,h+1}^{\xi}} = \mathrm{argmin}_{\mathcal{P}\in\mathcal{U}^{\sigma_i}\big(P_{h,s,a_i}^{\pi_{-i}^k}\big)}\mathcal{P}\overline{V}_{i,h+1}^{\xi}.
\end{align}

To continue, we first introduce the following lemma, which establishes that the value $\widehat{V}_{i,h}$ is an optimistic version of the best-response surrogate $\overline{V}_{i,h}^{\star,\xi}$; its proof is deferred to Appendix~\ref{proof:lm:generative_model_FTRL}:
\begin{lemma}
\label{lm:generative_model_FTRL}
With probability at least $1-\delta$, for all $(i,h,s)\in[n]\times[H] \times \mathcal{S}$, with the bonus factors $\beta_{i,h}(s)$ defined in \eqref{eq:genetative_model_beta}, it holds that
    \begin{align*}
        \quad \widehat{V}_{i,h}(s)\geq\overline{V}_{i,h}^{\star,\xi}.
    \end{align*}
\end{lemma}
Applying Lemma~\ref{lm:generative_model_FTRL}, we first replace $\overline{V}_{i,h}^{\star,\xi}$ by its optimistic quantity $\widehat{V}_{i,h}$ and choose to control the corresponding gap as below:
\begin{align}
    &\overline{V}_{i,h}^{\star,\xi}(s)-\overline{V}_{i,h}^{\xi}(s)
    \leq \widehat{V}_{i,h}(s)-\overline{V}_{i,h}^{\xi}(s).
\end{align}

Using techniques and arguments analogous to those in the proof of Lemma~\ref{lm:generative_model_FTRL}, we obtain the following lemma and corrolary:
\begin{lemma}\label{lm:optimistic_estimate_2}

    With probability at least $1-\delta$, it holds that
    \begin{align*}
        \widehat{V}_{i,h}(s)\geq\overline{V}_{i,h}^{\xi}(s)
        ,\quad\text{for all }(i,h,s)\in[n]\times[H]\times\mathcal{S}.
    \end{align*}
\end{lemma}

\begin{corollary}
\label{cor:lm:optimistic_estimate_3}
    With probability at least $1-\delta$, it holds that for all $(i,h,s)\in[n]\times[H]\times\mathcal{S}$,
    \begin{align*}
        &\sum_{k=1}^K\frac{1}{K}\mathbb{E}_{a_i\sim\pi_{i,h}^k(s)}\big[r_{i,h}^k(s,a_i)+\widehat{P}_{i,h,s,a_i}^{\pi_{-i}^k,\widehat{V}}\widehat{V}_{i,h+1}\big]+\beta_{i,h}(s)\\
        \geq&\sum_{k=1}^K\frac{1}{K}\mathbb{E}_{a_i\sim\pi_{i,h}^k(s)}\big[r_{i,h}^k(s,a_i)+\widehat{P}_{i,h,s,a_i}^{\pi_{-i}^k,\overline{V}}\overline{V}_{i,h+1}^{\xi}\big]
    \end{align*}
\end{corollary}

Combining Lemma~\ref{lm:optimistic_estimate_2} and Corollary~\ref{cor:lm:optimistic_estimate_3},

 recalling the definition of $\widehat{V}_{i,h}$ in \eqref{eq:line-number-policy-update_generative_model} gives that
    \begin{align*}
          \widehat{V}_{i,h}(s)& = \min\left\{ \sum_{k=1}^{K} \frac{1}{K} \big<\pi_{i,h}^k(\cdot \mid s),q_{i,h}^k(s,\cdot)\big>
               + \beta_{i,h}(s),H-h+1\right\},\\
               &=\min\big\{\sum_{k=1}^K\frac{1}{K}\mathbb{E}_{a_i\sim\pi_{i,h}^k(s)}\big[r_{i,h}^k(s,a_i)+\widehat{P}_{i,h,s,a_i}^{\pi_{-i}^k,\widehat{V}}\widehat{V}_{i,h+1}\big]+\beta_{i,h}(s),H-h+1\big\}
    \end{align*}
Therefore, term $(B)$ can be controlled by
\begin{align}
    &\overline{V}_{i,h}^{\star,\xi}(s)-\overline{V}_{i,h}^{\xi}(s)
    \leq \widehat{V}_{i,h}(s)-\overline{V}_{i,h}^{\xi}(s) \notag\\
    \overset{\mathsf{(i)}}{\leq}& \sum_{k=1}^K\frac{1}{K}\mathbb{E}_{a_i\sim\pi_{i,h}^k(s)}\big[r_{i,h}^k(s,a_i)+\widehat{P}_{i,h,s,a_i}^{\pi_{-i}^k,\widehat{V}}\widehat{V}_{i,h+1}\big]+\beta_{i,h}(s)  - \frac{1}{K}\sum_{k=1}^K \mathbb{E}_{a_i \sim \pi_{i,h}^k(s)}\Big[r_{i,h}^k(s,a_i) + \widehat{P}_{i,h,s,a_i}^{\pi^k_{-i},\overline{V}_{i,h+1}^{\zeta}}\overline{V}_{i,h+1}^{\zeta}\Big]\notag \\
    =&\sum_{k=1}^K\frac{1}{K}\mathbb{E}_{a_i\sim\pi_{i,h}^k(s)}\big[\widehat{P}_{i,h,s,a_i}^{\pi^k_{-i},\widehat{V}}\widehat{V}_{i,h+1}\big]+\beta_{i,h}(s)- \sum_{k=1}^K\frac{1}{K}\mathbb{E}_{a_i\sim\pi_{i,h}^k(s)}\big[\widehat{P}_{i,h,s,a_i}^{\pi^k_{-i},\overline{V}}\overline{V}_{i,h+1}^{\xi}\big]\notag\\
    \leq&\sum_{k=1}^K\frac{1}{K}\mathbb{E}_{a_i\sim\pi_{i,h}^k(s)}\big[P_{i,h,s,a_i}^{\pi^k_{-i},\widehat{V}}\widehat{V}_{i,h+1}\big]+\beta_{i,h}(s)-\sum_{k=1}^K\frac{1}{K}\mathbb{E}_{a_i\sim\pi_{i,h}^k(s)}\big[P_{i,h,s,a_i}^{\pi^k_{-i},\overline{V}}\overline{V}_{i,h+1}^{\xi}\big]\notag\\
    &+\sum_{k=1}^K\frac{1}{K}\mathbb{E}_{a_i\sim\pi_{i,h}^k(s)}\Big|\widehat{P}_{i,h,s,a_i}^{\pi^k_{-i},\widehat{V}}\widehat{V}_{i,h+1}-P_{i,h,s,a_i}^{\pi^k_{-i},\widehat{V}}\widehat{V}_{i,h+1}\Big|\notag\\
    &+\sum_{k=1}^K\frac{1}{K}\mathbb{E}_{a_i\sim\pi_{i,h}^k(s)} \Big|P_{i,h,s,a_i}^{\pi^k_{-i},\overline{V}}\overline{V}_{i,h+1}^{\xi}-\widehat{P}_{i,h,s,a_i}^{\pi^k_{-i},\overline{V}}\overline{V}_{i,h+1}^{\xi}\Big|\label{eq:generative_model_term_B_error_non_recursion}.
\end{align}
Here, (i) holds by recalling the definition of $\widehat{V}_{i,h}$ in \eqref{eq:line-number-policy-update_generative_model} and the definition of $\overline{V}_{i,h}^{\xi}(s)$ in \eqref{eq:generative_model_auxiliary_value_def}. Furthermore, we apply Lemma~\ref{lm:generative_model_transition_error} and arrive at: with probability at least $1 - \delta$,
\begin{align}
    &\overline{V}_{i,h}^{\star,\xi}(s)-\overline{V}_{i,h}^{\xi}(s)  \leq \sum_{k=1}^K\frac{1}{K}\mathbb{E}_{a_i\sim\pi_{i,h}^k(s)}\big[P_{i,h,s,a_i}^{\pi^k_{-i},\widehat{V}}\widehat{V}_{i,h+1}\big]+\beta_{i,h}(s)-\sum_{k=1}^K\frac{1}{K}\mathbb{E}_{a_i\sim\pi_{i,h}^k(s)}\big[P_{i,h,s,a_i}^{\pi^k_{-i},\overline{V}}\overline{V}_{i,h+1}^{\xi}\big] \notag \\
    &\quad + \frac{4\sqrt{d}}{\sqrt{N}}\big(2H\sqrt{d\ln(NH+1)+2\ln\big(\tfrac{3KNHn}{\delta}\big)}+H\sqrt{d}\big) \notag \\
    &\leq \sum_{k=1}^K\frac{1}{K}\mathbb{E}_{a_i\sim\pi_{i,h}^k(s)}\big[P_{i,h,s,a_i}^{\pi^k_{-i},\overline{V}}\widehat{V}_{i,h+1}\big]+\beta_{i,h}(s)-\sum_{k=1}^K\frac{1}{K}\mathbb{E}_{a_i\sim\pi_{i,h}^k(s)}\big[P_{i,h,s,a_i}^{\pi^k_{-i},\overline{V}}\overline{V}_{i,h+1}^{\xi}\big] \notag \\
    &\quad + \frac{4\sqrt{d}}{\sqrt{N}}\big(2H\sqrt{d\ln(NH+1)+2\ln\left(\tfrac{3KNHn}{\delta}\right)}+H\sqrt{d}\big) \notag\\
    & = \sum_{k=1}^K\frac{1}{K}\mathbb{E}_{a_i\sim\pi_{i,h}^k(s)}\Big[  P_{i,h,s,a_i}^{\pi^k_{-i},\overline{V}} \Big(\widehat{V}_{i,h+1}- \overline{V}_{i,h+1}^{\xi} \Big)\Big] \notag \\
    &\quad + \beta_{i,h}(s) + \frac{4\sqrt{d}}{\sqrt{N}}\big(2H\sqrt{d\ln(NH+1)+2\ln\left(\tfrac{3KNHn}{\delta}\right)}+H\sqrt{d}\big) \label{eq:own-generative-model-term-B-bound}
\end{align}
where the penultimate inequality holds due to the elementary fact that \(P_{i,h,s,a_i}^{\pi^k_{-i},\widehat{V}}\widehat{V}_{i,h+1} \leq P_{i,h,s,a_i}^{\pi^k_{-i},\overline{V}}\widehat{V}_{i,h+1} \).
Finally, applying the recursive bound in \eqref{eq:own-generative-model-term-B-bound} recursively for $h=1,2,\ldots,H$ yields
\begin{align}
    \text{term}(B) = \overline{V}_{1,h}^{\star,\xi}(s)-\overline{V}_{1,h}^{\xi}(s) \leq & \sum_{i=1}^{H}\beta_{i,h}(s)+H\cdot\frac{4\sqrt{d}}{\sqrt{N}}\big(2H\sqrt{d\ln(NH+1)+2\ln\big(\tfrac{3KNHn}{\delta}\big)}+H\sqrt{d}\big)\notag\\
    \leq&\frac{12\sqrt{d}H}{\sqrt{N}}\big(2H\sqrt{d\ln(NH+1)+2\ln\big(\tfrac{3KNHn}{\delta}\big)}+H\sqrt{d}\big)+2H^2\sqrt{\frac{\ln A_i}{K}}.\label{eq:generative_model_bound_B}
\end{align}
where the first inequality holds by $\widehat{V}_{i,H+1}(s) = 0$ and the fundamental fact $\overline{V}_{i,H+1}^{\xi}(s)=0$ in \eqref{eq:own-fundamental-fact} for all $(i,h)\in[n] \times [H]$.


\paragraph{Step 6: Summing up the results.}
To complete the proof, combining the bounds for terms $(A)$ and $(C)$ in \eqref{eq:generative_model_bound_A_and_C} with that for term $(B)$ in \eqref{eq:generative_model_bound_B}, we obtain
\begin{align*}
 V_{i,1}^{\star,\xi}(s) -V_{i,1}^{\xi}(s)\lesssim \frac{dH^2}{\sqrt{N}} \sqrt{\ln\big(\frac{KNnH}{\delta}\big)}+H^2\sqrt{\frac{\ln A_i}{K}}.
\end{align*}
Consequently, Algorithm~\ref{alg:linear_generative_model} outputs an $\varepsilon$-robust CCE $\xi$, i.e.,
\begin{align*}
    \max_{i\in[n], s\in\mathcal{S}}  \Big(\mathbb{E}_{\pi\sim \xi} \big[ V_{i,1}^{\star,\pi_{-i} }(s)\big] - \mathbb{E}_{\pi\sim \xi}\big[ V_{i,1}^{\pi}(s)\big]\Big) \leq \varepsilon
\end{align*}
with probability at least $1 - \delta$, provided that $N \geq CH^4d^3/\varepsilon^2$ and $K \geq CH^4/\varepsilon^2$, where $C$ is an absolute constant. Specifically, if the total number of samples acquired during the learning process satisfies
\begin{align*}
  N_{\mathsf{all}} \geq HK(N+d(d+1)/2) \geq O\left(H^9d^3/\varepsilon^4\right).
\end{align*}
Thus, we finish the proof of Theorem~\ref{thm:main_generative_model}.

\subsection{Proof of auxiliary results}

\subsubsection{Proof of Lemma~\ref{lm:generative_model_transition_error}}\label{proof:lm:generative_model_transition_error}

\paragraph{Step 1: Proof of the transition estimation gap.} First, let's consider any fixed triple $(h,i,k)\in[H]\times[n]\times[K]$.
Applying strong duality property \citep{iyengar2005robust}, we can rewrite the transition estimation error as
    \begin{align}
        &\max_{(s,a_i)\in\mathcal{S}\times\mathcal{A}_i} \Big|P_{i,h,s,a_i}^{\widehat{\pi}_{-i}^k,V}V-\widehat{P}_{i,h,s,a_i}^{\widehat{\pi}_{-i}^k,V}V\Big|\notag\\
        \overset{\mathsf{(i)}}{=}&\max_{(s,a_i)\in\mathcal{S}\times\mathcal{A}_i}\Big|\max_{\alpha\in\big[\min_s V(s),\max_s V(s)\big]}\big[P_{h,s,a_i}^{\widehat{\pi}_{-i}^k}\big[V\big]_\alpha-\sigma_i\big(\alpha-\min_{s^\prime}\big[V\big]_\alpha\big(s^\prime\big)\big)\big]\notag\\
        &-\max_{\alpha\in\big[\min_s V(s),\max_s V(s)\big]}\big[P_{i,h,s,a_i}^{k}\big[V\big]_\alpha-\sigma_i\big(\alpha-\min_{s^\prime}\big[V\big]_\alpha\big(s^\prime\big)\big)\big]\Big|\notag\\
        \leq&\max_{(s,a_i)\in\mathcal{S}\times\mathcal{A}_i}\max_{\alpha\in\big[\min_s V(s),\max_s V(s)\big]}\big|P_{h,s,a_i}^{\widehat{\pi}_{-i}^k}\big[V\big]_\alpha-P_{i,h,s,a_i}^{k}\big[V\big]_\alpha\big|,\notag
    \end{align}
    where (i) holds by refering to the definitions in \eqref{eq:generative_model_transition_error_PV_definition} and then apply the dual form of robust Bellman operator \citep[Lemma 1]{shi2023curious} due to strong duality property \citep{iyengar2005robust}. See the definiton of $P_{i,h,s,a_i}^{k}$ in \eqref{eq:generative_model_transition_error_P}.
    For a fixed $\alpha$, we have
    \begin{align}
        &\max_{(s,a_i)\in\mathcal{S}\times\mathcal{A}_i}\big|P_{h,s,a_i}^{\widehat{\pi}_{-i}^k}\big[V\big]_\alpha-P_{i,h,s,a_i}^{k}\big[V\big]_\alpha\big|\notag\\
        \overset{\mathsf{(i)}}{=}&\max_{(s,a_i)\in\mathcal{S}\times\mathcal{A}_i}\big|\phi_i(s,a_i)^\top\mu_{i,h}^{\widehat{\pi}_{-i}^k}\big[V\big]_\alpha-\phi_i(s,a_i)^\top\widehat{\mu}_{i,h}^k\big[V\big]_\alpha\big|\notag\\
        \overset{\mathsf{(ii)}}{\leq}&\max_{(s,a_i)\in\mathcal{S}\times\mathcal{A}_i}\big\lVert\phi_i(s,a_i)\big\rVert_{(\Lambda_{i,h}^k)^{-1}}\big(2H\sqrt{d\ln(NH+1)+\ln\big(\frac{3NH}{\delta}\big)}+H\sqrt{d}\big),\notag
    \end{align}
    where (i) holds due to the linear assumption of Markov games in \eqref{eq:linaer-mg-assumption} and the definition of the estimated transition kernel in \eqref{eq:ridge_regression_generative_model}, and (ii) follows by applying Theorem~\ref{lm:ridge-regression-concentration} with $f=[V]_\alpha$ (which is a fixed function with range $[0,H]$ for the given $\alpha$) and $\lambda =1$, which holds with probability at least $1 - \delta$.

  Furthermore, we apply Lemma~\ref{lm:generative_model_supporting_set} with $\Sigma = \frac{\Lambda_{i,h}^k}{|\cD_i|}$ and $x = \phi_i(s, a_i)$ for all $(s,a_i)$ and arrive at
    \begin{align}\label{eq:own-1}
        &\max\limits_{(s,a_i)\in\mathcal{S}\times\mathcal{A}_i}\big|P_{h,s,a_i}^{\widehat{\pi}_{-i}^k}\big[V\big]_\alpha-P_{i,h,s,a_i}^{k}\big[V\big]_\alpha\big|
        \leq \sqrt{\frac{d}{N}}\big(2H\sqrt{d\ln(NH+1)+\ln\big(\frac{3NH}{\delta}\big)}+H\sqrt{d}\big),
    \end{align}
    with probability at least $1-\delta$.
    To extend the pointwise bound towards a union bound for all $\alpha\in[0,H]$, we define a $\varepsilon_1$-net denoted as $N_{\varepsilon_1}$ for $\alpha$. To cover $\alpha\in[0,H]$, the net size bounded by $|N_{\varepsilon_1}|\leq 3H/\varepsilon_1$ \citep{vershynin2018high} allows us to apply the union bound:
    \begin{align}
        &\max_{(s,a_i)\in\mathcal{S}\times\mathcal{A}_i}\max_{\alpha\in\big[\min_s V(s),\max_s V(s)\big]}\big|P_{h,s,a_i}^{\widehat{\pi}_{-i}^k}\big[V\big]_\alpha-P_{i,h,s,a_i}^{k}\big[V\big]_\alpha\big| \nonumber \\
        \leq& \max_{(s,a_i)\in\mathcal{S}\times\mathcal{A}_i}\max_{\alpha\in N_{\varepsilon_1}}\big|P_{h,s,a_i}^{\widehat{\pi}_{-i}^k}\big[V\big]_\alpha-P_{i,h,s,a_i}^{k}\big[V\big]_\alpha\big|+\varepsilon_1 \nonumber \\
        \leq&\sqrt{\frac{d}{N}} \big(2H\sqrt{d\ln(NH+1)+\ln\big(\frac{3NH |N_{\varepsilon_1}|}{\delta}\big)}+H\sqrt{d}\big)+\varepsilon_1 \notag\\
        \leq& 2\sqrt{\frac{d}{N}}\big(2H\sqrt{d\ln(NH+1)+\ln\big(\frac{3NH}{\delta}\big)}+H\sqrt{d}\big), \label{eq:generative_model_transition_estimation_error_mediate_eq_final}
    \end{align}
    where the first inequality follows because the chosen $\alpha$ lies within an $\varepsilon$-ball centered at some point in $N_{\varepsilon_1}$, and the function $\big|P_{h,s,a_i}^{\widehat{\pi}_{-i}^k}\big[V\big]_\alpha-P_{i,h,s,a_i}^{k}\big[V\big]_\alpha\big|$ is $1$-Lipschitz in $\alpha$; the penultimate line holds by appying \eqref{eq:own-1}; and the last line is induced by letting $\varepsilon_1=1/N$ and $|N_{\varepsilon_1}|\leq 3HN$.

 Note that \eqref{eq:generative_model_transition_estimation_error_mediate_eq_final} is established for a \emph{fixed} triple $(h,i,k)\in[H]\times[n]\times[K]$, with probability at least $1-\delta$.
To upgrade this pointwise bound to a union bound over all $(h,i,k)\in[H]\times[n]\times[K]$, we apply for all $nKH$ triples and arrive at: with probability at least $1-\delta$, for all $(h,i,k)\in[H]\times[n]\times[K]$, and any fixed value function $V:\mathcal{S}\rightarrow[0,H]$, it obeys
    \begin{align*}
        \max_{(s,a_i)\in \cS\times\cA_i}\big|P_{i,h,s,a_i}^{\widehat{\pi}_{-i}^k,V}V-\widehat{P}_{i,h,s,a_i}^{\widehat{\pi}_{-i}^k,V}V\big|\leq 2\sqrt{\frac{d}{N}}\big(2H\sqrt{d\ln(NH+1)+2\ln\big(\frac{3KNHn}{\delta}\big)}+H\sqrt{d}\big).
    \end{align*}



\paragraph{Proof of the reward estimation gap.}

The reward bound follows the same proof structure as that used for the transition kernel, but with a simpler one-dimensional parameter vector. Based on the linear assumption of the reward function $r_{i,h}$ in \eqref{eq:linaer-mg-assumption} and the construction of the estimated reward function $r_{i,h}^k$ in \eqref{eq:ridge_regression_generative_model}, we derive that
    \begin{align*}
        &\max_{(s,a_i)\in\mathcal{S}\times\mathcal{A}_i}\big|\mathbb{E}_{a_{-i}\sim\widehat{\pi}_{-i}^k(s)}r_{i,h}(s,\mathbf{a})-r_{i,h}^k(s,a_i)\big|
         = \max_{(s,a_i)\in\mathcal{S}\times\mathcal{A}_i}\big|\phi_i(s,a_i)^\top\theta_{i,h}^{\widehat{\pi}_{-i}^k}-\phi_i(s,a_i)^\top\widehat{\theta}_{i,h}^k\big|.
    \end{align*}
    Applying Theorem~\ref{lm:ridge-regression-concentration} with the scalar target $f=\mathbb{E}_{a_{-i}\sim\widehat{\pi}_{-i}^k(s)}[r_{i,h}(s,\mathbf a)]\in[0,H]$ and $\lambda=1$, we instantiate the ridge-regression concentration bound for the one-dimensional reward parameter and
    obtain that with probability at least $1 - \delta$, for a fixed $(h, i, k)$:
    \begin{align}
        &\max_{(s,a_i)\in\mathcal{S}\times\mathcal{A}_i}\big|\mathbb{E}_{a_{-i}\sim\widehat{\pi}_{-i}^k(s)}r_{i,h}(s,\mathbf{a})-r_{i,h}^k(s,a_i)\big|\notag\\
        \leq&\max_{(s,a_i)\in\mathcal{S}\times\mathcal{A}_i}\big\lVert\phi_i(s,a_i)\big\rVert_{\big(\Lambda_{i,h}^k\big)^{-1}}\cdot\big(2H\sqrt{d\ln(NH+1)+\ln\big(\frac{1}{\delta}\big)}+H\sqrt{d}\big)\notag\\
        \leq&\frac{\sqrt{d}}{\sqrt{N}}\cdot\big(2H\sqrt{d\ln(NH+1)+\ln\big(\frac{1}{\delta}\big)}+H\sqrt{d}\big), \label{eq:generative_model_reward_estimation}
    \end{align}
    where the last inequality holds by applying Lemma~\ref{lm:generative_model_supporting_set} with $\Sigma = \frac{\Lambda_{i,h}^k}{|\cD_i|}$ and $x = \phi_i(s, a_i)$ for all $(s,a_i)$.

As in the transition case, we will transfer this pointwise bound \eqref{eq:generative_model_reward_estimation} the high-probability event underlying the ridge-regression bound above is defined relative to the samples used to estimate $\widehat{\theta}_{i,h}^k$ for a \emph{fixed} triple $(h,i,k)\in[H]\times[n]\times[K]$ for any fixed$(h,i,k)$ to a union bound over the $nKH$ triples of the space $[H]\times[n]\times[K]$.
Consequently, it obeys that with probability at least $1-\delta$, for all $(h,i,k) \in [H]\times[n]\times[K]$,
    \begin{align*}
        \big|\mathbb{E}_{a_{-i}\sim\widehat{\pi}_{-i}^k(s)}r_{i,h}(s,\mathbf{a})-r_{i,h}^k(s,a_i)\big|\leq \frac{\sqrt{d}}{\sqrt{N}}\big(2H\sqrt{d\ln(NH+1)+\ln\big(\tfrac{3NHnK}{\delta}\big)}+H\sqrt{d}\big).
    \end{align*}



\subsubsection{Proof of Lemma~\ref{lm:generative_model_FTRL} }\label{proof:lm:generative_model_FTRL}

Before proceeding, we first introduce an exsiting result for Follow-the-Regularized-Leader (FTRL) (see \citet[Theorem 6]{shalev2007online} and also \citet[Theorem 3]{li2022minimax}) that we will resort to: for all $(i,h) \in [n] \times [H]$:
    \begin{equation}
    \label{eq:generative_model_FTRL_condition}
    \begin{aligned}
        &\max_{a_i\in\mathcal{A}_i}\sum_{k=1}^K\frac{1}{K}\big[r_{i,h}^k(s,a_i)+\widehat{P}_{i,h,s,a_i}^{\pi_{-i}^k,\widehat{V}}\widehat{V}_{i,h+1}\big]\\
        \leq&\sum_{k=1}^K\frac{1}{K}\mathbb{E}_{a_i\sim\pi_{i,h}^k(s)}\big[r_{i,h}^k(s,a_i)+\widehat{P}_{i,h,s,a_i}^{\pi_{-i}^k,\widehat{V}}\widehat{V}_{i,h+1}\big]+2H\sqrt{\frac{\ln A_i}{K}},
    \end{aligned}
    \end{equation}
which can be obtained byc instantiating the standard FTRL regret bound in our setting as follows. For any fixed time step $h$, agent $i$, and state $s\in\cS$, view the $K$ iterations as an online-learning process for the agent. At each iteration $k$, the per-round reward is defnied as $\ell_{i,h}^k(a_i)\defn r_{i,h}^k(s,a_i)+\widehat{P}_{i,h,s,a_i}^{\pi_{-i}^k,\widehat{V}}\widehat{V}_{i,h+1}$ for all $a_i\in\cA_i$ and $\widehat{V}_{i,h+1} \in [0,H]$. During the prpocess, the agent $i$ plays the mixed strategy $\pi_{i,h}^k(\cdot\mid s)\in\Delta(\mathcal A_i)$ produced by FTRL with entropy regularizer. Applying \citet[Theorem 6]{shalev2007online} (equivalently \citet[Theorem 3]{li2022minimax}) with per-round reward range size $H$ and action-space size $|\mathcal A_i|=A_i$ yields the regret bound $\max_{a_i}\sum_{k=1}^K \ell_{i,h}^k(a_i)-\sum_{k=1}^K\mathbb E_{a_i\sim\pi_{i,h}^k(s)}\big[\ell_{i,h}^k(a_i)\big]\leq 2H\sqrt{K\ln A_i}$; dividing by $K$ gives \eqref{eq:generative_model_FTRL_condition}.x

Armed with the key results, we shall prove Lemma~\ref{lm:generative_model_FTRL} using induction arguments. Initially, for the base case at the final time step $H+1$, we directly have
    \begin{align*}
        \text{for all $(i,s)\in [n]\times\mathcal{S}$}: \quad  \widehat{V}_{i,H+1}(s)=0=\overline{V}_{i,H+1}^{\star,\xi}(s),
    \end{align*}
due to the definition of $\widehat{V}_{i,H+1}$ and the fundamental fact in \eqref{eq:own-fundamental-fact}.
    We proceed by backward induction from the last step  $h=H+1$ to $h=1$.  Assume that $ \widehat{V}_{i,h+1}(s) \geq \overline{V}_{i,h+1}^{\star,\xi}(s) $ holds for all $(i,s) \in [n] \times \mathcal{S} $. Our goal is to show that $ \widehat{V}_{i,h}(s) \geq \overline{V}_{i,h}^{\star,\xi}(s) $ for all $(i,s) \in [n] \times \mathcal{S}$.

Before proceeding, to invoke the FTRL regret bound in \eqref{eq:generative_model_FTRL_condition}, we must bound the range of the per-round gain $\ell_{i,h}^k(a_i)\defn r_{i,h}^k(s,a_i)+\widehat{P}_{i,h,s,a_i}^{\pi{-i}^k,\widehat{V}}\widehat{V}_{i,h+1}$. By the definition of $\widehat V_{i,h}$ in \eqref{eq:line-number-policy-update_generative_model}, we have $\widehat V_{i,h}(s)=\min\big\{\sum_{k=1}^K\frac{1}{K}\mathbb E_{a_i\sim\pi_{i,h}^k(s)}\big[r_{i,h}^k(s,a_i)+\widehat{P}_{i,h,s,a_i}^{\pi_{-i}^k,\widehat V}\widehat V_{i,h+1}\big]+\beta_{i,h}(s),\,H-h+1\big\}$. It immediately yields $\widehat V_{i,h}(s) \leq H-h+1 \leq H$. For the lower bound, we proceed by backward induction on $h$: with the base case $\widehat V_{i,H+1}= 0$, and given $\widehat V_{i,h+1}\geq 0$, each summand of the first argument is non-negative—$r_{i,h}^k(s,a_i)\geq 0$ by assumption, $\widehat{P}_{i,h,s,a_i}^{\pi{-i}^k,\widehat V}\widehat V_{i,h+1}\geq 0$ as a probability kernel applied to a non-negative function, and $\beta_{i,h}(s)\geq 0$. Hence $\widehat V_{i,h}(s)\in[0,H]$ for all $(i,h)\in[n]\times[H]$. Now we are ready to verify the induction hypothesis for step $h$.
 Towards this, recalling the definition of $\overline{V}_{i,h}^{\star,\xi}(s)$ (see \eqref{eq:generative_model_auxiliary_value_def}) gives

    \begin{align}
        &\overline{V}_{i,h}^{\star,\xi}(s) = \max_{a_i\in\mathcal{A}_i}\sum_{k=1}^K\frac{1}{K}\big[r_{i,h}^k(s,a_i)+\widehat{P}_{i,h,s,a_i}^{\pi_{-i}^k,\overline{V}_{i,h+1}^{\star,\xi}}\overline{V}_{i,h+1}^{\star,\xi}\big]\notag\\
        \overset{\mathsf{(i)}}{\leq}&\max_{a_i\in\mathcal{A}_i}\sum_{k=1}^K\frac{1}{K}\big[r_{i,h}^k(s,a_i)+P_{i,h,s,a_i}^{\pi_{-i}^k,\overline{V}_{i,h+1}^{\star,\xi}}\overline{V}_{i,h+1}^{\star,\xi}\big]\notag\\
        &+\frac{2\sqrt{d}}{\sqrt{N}}\big(2H\sqrt{d\ln(NH+1)+2\ln\big(\frac{3KNHn}{\delta}\big)}+H\sqrt{d}\big)\notag\\
        \overset{\mathsf{(ii)}}{\leq}&\max_{a_i\in\mathcal{A}_i}\sum_{k=1}^K\frac{1}{K}\big[r_{i,h}^k(s,a_i)+P_{i,h,s,a_i}^{\pi_{-i}^k,\widehat{V}}\widehat{V}_{i,h+1} \big]+\frac{2\sqrt{d}}{\sqrt{N}}\big(2H\sqrt{d\ln(NH+1)+2\ln\big(\frac{3KNHn}{\delta}\big)}+H\sqrt{d}\big)\notag\\
        \overset{\mathsf{(iii)}}{\leq}&\max_{a_i\in\mathcal{A}_i}\sum_{k=1}^K\frac{1}{K}\big[r_{i,h}^k(s,a_i)+\widehat{P}_{i,h,s,a_i}^{\pi_{-i}^k,\widehat{V}} \widehat{V}_{i,h+1} \big]+\frac{4\sqrt{d}}{\sqrt{N}}\big(2H\sqrt{d\ln(NH+1)+2\ln\big(\frac{3KNHn}{\delta}\big)}+H\sqrt{d}\big)\notag\\
        \overset{\mathsf{(iv)}}{\leq}&\sum_{k=1}^K\frac{1}{K}\mathbb{E}_{a_i\sim\pi_{i,h}^k(s)}\big[r_{i,h}^k(s,a_i)+\widehat{P}_{i,h,s,a_i}^{\pi_{-i}^k,\widehat{V}}\widehat{V}_{i,h+1}\big]+\beta_{i,h}(s),\label{eq:generative_model_optimistic_estimation_eq_final}
    \end{align}
    Here,  (i) and (iii) follows from applying Lemma~\ref{lm:generative_model_transition_error}, (ii) follows from the induction assumption,  and (iv) is based on the results of FTRL in \eqref{eq:generative_model_FTRL_condition} and recalling $\beta_{i,h}(s)$ in \eqref{eq:genetative_model_beta} for all $(i,s) \in [n]\times \mathcal{S}$:
    \begin{align*}
        \beta_{i,h}(s)=8\sqrt{\frac{d}{N}}\big(2H\sqrt{d\ln(NH+1)+2\ln\big(\frac{3KNHn}{\delta}\big)}+H\sqrt{d}\big)+2H\sqrt{\frac{\ln A_i}{K}}.
    \end{align*}
    On the other hand, recalling the definition of $\widehat{V}_{i,h}$ in \eqref{eq:line-number-policy-update_generative_model} gives that
    \begin{align*}
          \widehat{V}_{i,h}(s)& = \min\left\{ \sum_{k=1}^{K} \frac{1}{K} \big<\pi_{i,h}^k(\cdot \mid s),q_{i,h}^k(s,\cdot)\big>
               + \beta_{i,h}(s),H-h+1\right\},\\
               &=\min\big\{\sum_{k=1}^K\frac{1}{K}\mathbb{E}_{a_i\sim\pi_{i,h}^k(s)}\big[r_{i,h}^k(s,a_i)+\widehat{P}_{i,h,s,a_i}^{\pi_{-i}^k,\widehat{V}}\widehat{V}_{i,h+1}\big]+\beta_{i,h}(s),H-h+1\big\}
    \end{align*}
    where the last equality holds by the definitions of $q_{i,h}^k$ in \eqref{eq:q_function_generative_model}. Therefore, with the fact that $\overline{V}_{i,h}^{\star,\xi }(s)\leq H-h+1$ for all $(i,s) \in [n] \times \mathcal{S}$, we obtain that
    \begin{align*}
       \forall (i,s) \in [n] \times \mathcal{S}: \quad \overline{V}_{i,h}^{\star,\xi }(s)\leq \widehat{V}_{i,h}(s).
    \end{align*}
    Therefore, we have finished the proof of the lemma by induction arguments.

\section{Proof of Theorem~\ref{thm:main_online}}
We now turn to the proof for the online interactive setting. The overall structure parallels the generative case: we begin with estimation lemmas, then construct optimistic and pessimistic surrogates, and finally combine them to derive a recursive relationship, which leads to the regret bound.

\paragraph{Step 1: error decomposition.}
First, recall that the objective in the online interactive setting is to upper bound the following regret term induced by the process of Algoarithm~\ref{alg:lin_robust_Q_FTRL}:
\begin{align*}
    \max_{i\in[n]}\sum_{t=1}^T\mathbb{E}_{\pi\sim\xi^t}\big[V_{i,1}^{\star,\pi_{-i},\sigma_i}(s)\big]-\mathbb{E}_{\pi\sim\xi^t}\big[V_{i,1}^{\pi,\sigma_i}(s)\big]=\max_{i\in[n]}\sum_{t=1}^TV_{i,1}^{\star,\xi^t}(s)-V_{i,1}^{\xi^t}(s).
\end{align*}

Before continuning, for any fixed set of policies $\big\{\widehat{\pi}^{k.t}\big\}_{(k,t)\times[K]\times[T]}$ with $\widehat{\pi}^{k.t} \in [H] \times \cS \mapsto \prod_{i=1}^n \Delta(\cA_i)$, we define the following useful auxiliary notations of transition kernels:
\begin{align*}
    & P_{i,h,s,a_i}^{\widehat{\pi}_{-i}^{k,t},V}V \defn \inf_{\mathcal{P}\in\mathcal{U}^{\sigma_i}\big(P_{h,s,a_i}^{\widehat{\pi}_{-i}^{k,t}}\big)}\mathcal{P}V, \qquad \widehat{P}_{i,h,s,a_i}^{\widehat{\pi}_{-i}^{k,t},V}V \defn \inf_{\mathcal{P}\in\mathcal{U}^{\sigma_i}\big(P_{i,h,s,a_i}^{k,t})}\mathcal{P}V,
\end{align*}
where $P_{i,h,s,a_i}^{k,t}$ is the estimated transition kernel output by Algoarithm~\ref{alg:lin_robust_Q_FTRL}.

Armed with the notations, we first introduce the following lemma, which bounds the estimation errors of the transition probability and the reward function. The proof, which is analogous to that of Lemma~\ref{lm:generative_model_transition_error} with the added union bound over $t\in[T]$, is provided in Appendix~\ref{proof:lm:online_transition_estimation_error}.
\begin{lemma}
\label{lm:online_transition_estimation_error}
    Consider any fixed $(h,i,k,t)\in[H]\times[n]\times[K]\times[T]$, fixed value function $V:\mathcal{S}\rightarrow [0,H]$ and $\delta\in(0,1)$. With probability at least $1-2\delta$, the following hold simultaneously for all $(s,a_i)\in\mathcal{S}\times\mathcal{A}_i$:
        \begin{align}
            &\big|P_{i,h,s,a_i}^{\widehat{\pi}_{-i}^{k,t},V}V-\widehat{P}_{i,h,s,a_i}^{\widehat{\pi}_{-i}^{k,t},V}V\big|
            \leq\big\lVert\phi(s,a_i)\big\rVert_{\big(\Lambda_{i,h}^{k,t}\big)^{-1}} \big(2H\sqrt{d\ln(NH+1)+2\ln\big(\tfrac{3TKNHn}{\delta}\big)}+H\sqrt{d\lambda}\big)+\frac{1}{N},\nonumber \\
            &\big|\mathbb{E}_{a_{-i}\sim\widehat{\pi}_{-i}^{k,t}(s)}\big[r_{i,h}(s,\mathbf{a})\big]-r_{i,h}^{k,t}(s,a_i)\big|
            \leq \big\lVert\phi(s,a_i)\big\rVert_{\big(\Lambda_{i,h}^{k,t}\big)^{-1}}\big(2H\sqrt{d\ln(NH+1)+\ln\big(\tfrac{3NTHnK}{\delta}\big)}+H\sqrt{d\lambda}\big),
        \end{align}
        where $\Lambda_{i,h}^{k,t}$ was defined in \eqref{eq:ridge_regression_online} with regularization $\lambda\geq 1$.
\end{lemma}

To continue, recalling the definition of the estimated value function $\overline{V}_{i,h}^t(s)$ in \eqref{eq:estimate_value_function}, analogous to the generative case, we will show that $\overline{V}_{i,h}^t(s)$ is an optimistic estimation of the best-response value function $V_{i,h}^{\star,\xi^t}(s)$ simultaneously for all $h\in[H]$.

\begin{lemma}
\label{lm:online_optimistic_value_function}
    For all $(t,i,h,s)\in[T]\times[n]\times[H]\times\mathcal{S}$, invoking the optimistic bonus term defined in  \eqref{eq:online_optimistic_bonus}, with probability at least $1-\delta$, the following inequality holds true:
    \begin{align*}
       \forall s\in\cS:\quad \overline{V}_{i,h}^t(s)\geq V_{i,h}^{\star,\xi^t}(s).
    \end{align*}


\end{lemma}

Unlike the generative case, we also maintain a collection of pessimistic estimates of the value function, and will show that $\underline{V}_{i,h}^t(s)$ serves as a lower bound for $V_{i,h}^{\xi^t}(s)$.
\begin{lemma}
\label{lm:online_pessimistic_value_function}
    Invoking the pessimistic bonus in \eqref{eq:online_pessimistic_bonus}, with probability at least $1-\delta$,
    \begin{align*}
        \underline{V}_{i,h}^t(s)\leq V_{i,h}^{\xi^t}(s),
    \end{align*}
    holds true for all $(i,h,t,s)\in[n]\times[H]\times[T] \times \cS$.
\end{lemma}
The proof of Lemma~\ref{lm:online_optimistic_value_function} and Lemma~\ref{lm:online_pessimistic_value_function} are postponed to Appendix~\ref{proof:online-auxiliary}. Now we are now ready to prove Theorem~\ref{thm:main_online}. Applying Lemma~\ref{lm:online_optimistic_value_function} and Lemma~\ref{lm:online_pessimistic_value_function}, we have
\begin{align}
    &V_{i,h}^{\star,\xi^t}(s)-V_{i,h}^{\xi^t}(s)
    \leq\overline{V}_{i,h}^t(s)-\underline{V}_{i,h}^t(s)\notag\\
    \overset{\mathsf{(i)}}{\leq}&\sum_{k=1}^K\frac{1}{K}\mathbb{E}_{a_i\sim\pi_{i,h}^{k,t}(s)}\big[r_{i,h}^{k,t}(s,a_i)+\widehat{P}_{i,h,s,a_i}^{\pi_{-i}^{k,t},\overline{V}}\overline{V}_{i,h+1}^t\big]+\beta_{i,h,1}^t(s)+\beta_{i,h,2}^t(s)\notag\\
    &-\sum_{k=1}^K\frac{1}{K}\mathbb{E}_{a_i\sim\pi_{i,h}^{k,t}(s)}\big[r_{i,h}^{k,t}(s,a_i)+\widehat{P}_{i,h,s,a_i}^{\pi_{-i}^{k,t},\underline{V}}\underline{V}_{i,h+1}^t\big]\notag\\
    \leq& \beta_{i,h,1}^t(s) +\beta_{i,h,2}^t(s) + \sum_{k=1}^K\frac{1}{K}\mathbb{E}_{a_i\sim\pi_{i,h}^{k,t}(s)}\big[P_{i,h,s,a_i}^{\pi_{-i}^{k,t},\overline{V}}\overline{V}_{i,h+1}^t\big] \notag \\
    &+\sum_{k=1}^K\frac{1}{K}\mathbb{E}_{a_i\sim\pi_{i,h}^{k,t}(s)}\big|\widehat{P}_{i,h,s,a_i}^{\pi_{-i}^{k,t},\overline{V}}\overline{V}_{i,h+1}^t-P_{i,h,s,a_i}^{\pi_{-i}^{k,t},\overline{V}}\overline{V}_{i,h+1}^t\big| -\sum_{k=1}^K\frac{1}{K}\mathbb{E}_{a_i\sim\pi_{i,h}^{k,t}(s)}\big[P_{i,h,s,a_i}^{\pi_{-i}^{k,t},\underline{V}}\underline{V}_{i,h+1}^t\big]\notag\\
&+\sum_{k=1}^K\frac{1}{K}\mathbb{E}_{a_i\sim\pi_{i,h}^{k,t}(s)}\big|P_{i,h,s,a_i}^{\pi_{-i}^{k,t},\underline{V}}\underline{V}_{i,h+1}^t-\widehat{P}_{i,h,s,a_i}^{\pi_{-i}^{k,t},\underline{V}}\underline{V}_{i,h+1}^t\big| \notag \\
     \overset{\mathsf{(ii)}}{\leq}&\sum_{k=1}^K\frac{1}{K}\mathbb{E}_{a_i\sim\pi_{i,h}^{k,t}(s)}\big[P_{i,h,s,a_i}^{\pi_{-i}^{k,t},\overline{V}}\overline{V}_{i,h+1}^t-P_{i,h,s,a_i}^{\pi_{-i}^{k,t},\underline{V}}\underline{V}_{i,h+1}^t\big] + \beta_{i,h,1}^t(s) +\beta_{i,h,2}^t(s) \notag\\
    &+\sum_{k=1}^K\frac{1}{K}\max_{a_i\in\mathcal{A}_i}\big\lVert\phi_i(s,a_i)\big\rVert_{\big(\Lambda_{i,h}^{k,t}\big)^{-1}}\big(4H\sqrt{d\ln(NH+1)+\ln\big(\tfrac{3NTHnK}{\delta}\big)}+ 2H\sqrt{d\lambda}\big)+\frac{1}{N} \notag\\
    \leq&\sum_{k=1}^K\frac{1}{K}\mathbb{E}_{a_i\sim\pi_{i,h}^{k,t}(s)}\big[P_{i,h,s,a_i}^{\pi_{-i}^{k,t}  \underline{V}} \big(\overline{V}_{i,h+1}^t - \underline{V}_{i,h+1}^t\big)\big] + \beta_{i,h,1}^t(s) +\beta_{i,h,2}^t(s)\notag\\
    &+\sum_{k=1}^K\frac{1}{K}\max_{a_i\in\mathcal{A}_i}\big\lVert\phi_i(s,a_i)\big\rVert_{\big(\Lambda_{i,h}^{k,t}\big)^{-1}}\big(4H\sqrt{d\ln(NH+1)+\ln\big(\tfrac{3NTHnK}{\delta}\big)}+ 2H\sqrt{d\lambda}\big)+\frac{1}{N}.
\label{eq:online_sum_up_error_decompose_eq1}
\end{align}
where (i) is obtained by recalling the definition of $\overline{V}_{i,h}^t(s)$ and $\underline{V}_{i,h}^t(s)$ in \eqref{eq:estimate_value_function}, (ii) is obtained by applying Lemma~\ref{lm:online_transition_estimation_error} to replace the empirical estimate $\widehat{P}$ by their ground truth,
and the last inequality holds with probability at least $1 - \delta$.

\paragraph{Step 2: controlling the second term in \eqref{eq:online_sum_up_error_decompose_eq1}.}
To continue, we focus on the second term in \eqref{eq:online_sum_up_error_decompose_eq1}. We first show that, for $\lambda\geq c_0\log(2dnKHT/\delta)=\widetilde{\mathcal O}(1)$, the empirical Gram matrix $\Lambda_{i,h}^{k,t}$ concentrates multiplicatively around its expectation.

Recall the empirical Gram matrix from \eqref{eq:ridge_regression_online}:
\begin{align*}
    \Lambda_{i,h}^{k,t}=\sum_{m=1}^{\lvert\cD_i\rvert}\phi_i^{(m)}(\phi_i^{(m)})^\top+\lambda I_d,\qquad \text{where} \qquad  \phi_i^{(m)}\defn\phi_i\big(s_h^{(m)},a_{i,h}^{(m)}\big).
\end{align*}
which is a positive definite matrix satisfying $\Lambda_{i,h}^{k,t}\succeq\lambda I_d$. By invoking the Hybrid-Sampling procedure in Algorithm~\ref{alg:multi-sampling}, the dataset $\cD_i$ at round $t$ contains $\lvert\cD_i\rvert\approx N$ samples. Specifically, for each $l \in [t-1]$, we independently collect $\lceil N/t \rceil$ samples conditioned on $\{\pi^{k,l}\}_{k\in[K]}$, which together form $\{(s_h^{(m)},a_{i,h}^{(m)})\}_{m=1}^{\lvert\cD_i\rvert}$. Recall the hybrid sampling procedure in Algorithm~\ref{alg:multi-sampling} for any round $t\in[T]$, time step $h\in[H]$, and agent $i$: for each $j\in[h-1]$,
\begin{align}
    s_{j+1}\sim\sum_{k=1}^K\frac{1}{K}\mathbb{E}_{a_{i,j}\sim\pi_{i,j}^{k,t}(s_j)}P_{i,j,s_j,a_{i,j}}^{\pi_{-i,j}^{k,t},\underline{V}}(\cdot),
\end{align}
where $P_{i,j,s_j,a_{i,j}}^{\pi_{-i,j}^{k,t},\underline{V}}$ denotes the pessimistic worst-case transition kernel defined in \eqref{eq:online_pessimistic_value_transition_def}. Accordingly, we denote the resulting distribution of $s_h$ by
\begin{align*}
    s_h\sim \bigg[\sum_{k=1}^K\frac{1}{K}P_{i,1}^{\pi_{1}^{k,t},\underline{V}}\bigg]\cdots\bigg[\sum_{k=1}^K\frac{1}{K}P_{i,h-1}^{\pi_{h-1}^{k,t},\underline{V}}\bigg].
\end{align*}
For all $(t,h)\in[T]\times[H]$, we define
\begin{align*}
    & X_{t,h}=\mathbb{E}_{s_h, a_i}\big[\phi_i(s_h,a_i)\phi_i(s_h,a_i)^\top\ \big|\ s_h\sim \Big[\sum_{k=1}^K\frac{1}{K}P_{i,1}^{\pi_{1}^{k,t},\underline{V}}\Big]\cdots\Big[\sum_{k=1}^K\frac{1}{K}P_{i,h-1}^{\pi_{h-1}^{k,t},\underline{V}}\Big],\ a_{i}\sim\text{Unif}(\mathcal{A}_i)\big].
\end{align*}

Therefore, denoting the round index of any $m$-th sample as $\tau(m)\in[t-1]$, one has
\begin{align*}
    \mathbb E\!\big[\phi_i^{(m)}(\phi_i^{(m)})^\top\big]=X_{\tau(m),h}
\end{align*}
so that
\begin{align}\label{eq:lambda-fact1}
    \bar\Lambda\defn\mathbb E[\Lambda_{i,h}^{k,t}]=\sum_{m=1}^{\lvert\cD_i\rvert}X_{\tau(m),h}+\lambda I_d \qquad \text{and} \qquad \bar\Lambda\succeq\lambda I_d.
\end{align}

Then we introduce the following lemma whose proof is postponed to Appendix~\ref{proof:lemma:lambda-expectation}.
\begin{lemma}\label{lemma:lambda-expectation}
Let $\lambda\geq c_0\log(2dnKHT/\delta)$ for some large enough constant $c_0$. With probability at least $1-\delta$, for all $(k,t,h,i)\in\mathcal[K]\times[T]\times[H]\times[n]$, one has
\begin{align}\label{eq:online_matrix_two_sided}
    \tfrac{1}{2}\,\bar\Lambda\;\preceq\;\Lambda_{i,h}^{k,t}\;\preceq\;\tfrac{3}{2}\,\bar\Lambda.
\end{align}
\end{lemma}

With the above lemma in hand, we introduce an auxiliary term
\begin{align*}
S_{t,h}=\sum_{\tau=1}^{t-1}X_{\tau,h}+\lambda I_{d}
\end{align*}
As each round $l\in[t-1]$ contributes at least one sample to $\cD_i$, we have $\sum_{m=1}^{\lvert\cD_i\rvert}X_{\tau(m),h}\succeq\sum_{l=1}^{t-1}X_{l,h}$, since each population covariance $X_{l,h}$ is positive semi-definite. Therefore, one has
\begin{align}\label{eq:bar-Lambda-vs-S-0}
\bar\Lambda\succeq S_{t,h}\succeq\lambda I_d,
\end{align}
and consequently
\begin{align}\label{eq:bar-Lambda-vs-S}
   \forall\phi\in\mathbb R^d: \quad \bar\Lambda^{-1}\preceq S_{t,h}^{-1},\qquad
    \big\lVert\phi\big\rVert_{\bar\Lambda^{-1}}\leq\big\lVert\phi\big\rVert_{S_{t,h}^{-1}}.
\end{align}
Inverting \eqref{eq:online_matrix_two_sided} in  Lemma~\ref{lemma:lambda-expectation} yields the operator inequality $(\Lambda_{i,h}^{k,t})^{-1}\preceq 2\bar\Lambda^{-1}$. Combining with \eqref{eq:bar-Lambda-vs-S} gives that with probability at least $1-\delta$, simultaneously for all $(s,k,t,h,i)\in\mathcal S\times[K]\times[T]\times[H]\times[n]$,
\begin{align}\label{eq:online_matrix_Bernstein_final}
    \max_{a_i\in\mathcal{A}_i}\big\lVert\phi_i(s,a_i)\big\rVert_{\big(\Lambda_{i,h}^{k,t}\big)^{-1}}\lesssim\max_{a_i\in\mathcal{A}_i}\big\lVert\phi_i(s,a_i)\big\rVert_{\big(\mathbb{E}\Lambda_{i,h}^{k,t}\big)^{-1}}\leq\max_{a_i\in\mathcal{A}_i}\big\lVert\phi_i(s,a_i)\big\rVert_{S_{t,h}^{-1}}.
\end{align}

Inserting \eqref{eq:online_matrix_Bernstein_final} back into \eqref{eq:online_sum_up_error_decompose_eq1} gives
\begin{align}
     &V_{i,h}^{\star,\xi^t}(s)-V_{i,h}^{\xi^t}(s)
    \leq\overline{V}_{i,h}^t(s)-\underline{V}_{i,h}^t(s)\notag\\
     \leq&\sum_{k=1}^K\frac{1}{K}\mathbb{E}_{a_i\sim\pi_{i,h}^{k,t}(s)}\big[P_{i,h,s,a_i}^{\pi_{-i}^{k,t}  \underline{V}} \big(\overline{V}_{i,h+1}^t - \underline{V}_{i,h+1}^t\big)\big] + \beta_{i,h,1}^t(s) +\beta_{i,h,2}^t(s)\notag\\
    &+\sum_{k=1}^K\frac{1}{K}\max_{a_i\in\mathcal{A}_i}\big\lVert\phi_i(s,a_i)\big\rVert_{S_{t,h}^{-1}}\big(4H\sqrt{d\ln(NH+1)+\ln\big(\tfrac{3NTHnK}{\delta}\big)}+ 2H\sqrt{d\lambda}\big)+\frac{1}{N} \notag \\
     \leq& \frac{1}{K}\sum_{k=1}^{K}\mathbb E_{a_i\sim\pi_{i,h}^{k,t}(s)}P_{i,h,s,a_i}^{\pi_{-i}^{k,t},\underline V}\big(\overline V_{i,h+1}^t-\underline V_{i,h+1}^t\big)+cH\sqrt{\frac{\ln A_i}{K}}+\frac{c}{N} \notag \\
    &\quad +c\sum_{k=1}^{K}\frac{1}{K}\max_{a_i\in\mathcal A_i}\big\lVert\phi_i(s,a_i)\big\rVert_{S_{t,h}^{-1}}\cdot H\big(\sqrt{2d\ln(NH+1)+\ln(3TNHnK/\delta)}+\sqrt{d\lambda}\big), \label{eq:online_sum_up_error_decompose_eq2}
\end{align}
which holds with probability at least $1-\delta$. Here the last inequality holds by invoking the definitions of $\beta_{i,h,1}^t(s),\beta_{i,h,2}^t(s)$ in \eqref{eq:online_optimistic_bonus} and \eqref{eq:online_pessimistic_bonus}.

\paragraph{Step 3: recursion arguments.}
Consequently, applying \eqref{eq:online_sum_up_error_decompose_eq2} recursively across layers $h=1,\ldots,H$ arrives at
\begin{equation}
\label{eq:online_error_bound_2}
\begin{aligned}
& V_{i,1}^{\star,\xi^t}(s_1)-V_{i,1}^{\xi^t}(s_1)
    \leq\overline{V}_{i,1}^t(s_1)-\underline{V}_{i,1}^t(s_1)\\
    \leq&\sum_{h=1}^H c\big[H\sqrt{\frac{\ln A_i}{K}}+\frac{1}{N} + \mathbb{E}\big[\lVert\phi_i(s_h,a_i)\rVert_{S_{t,h}^{-1}} \ \big| \ s_h\sim \big[\sum_{k=1}^K\frac{1}{K}P_{i,1}^{\pi_{1}^{k,t},\underline{V}}\big]\cdots\big[\sum_{k=1}^K\frac{1}{K}P_{i,h-1}^{\pi_{h-1}^{k,t},\underline{V}}\big],\ a_i\sim\text{Unif}(\mathcal{A}_i)\big]\\
    &\cdot A_iH\big(\sqrt{2d\ln(NH+1)+\ln\big(\frac{3TNHnK}{\delta}\big)}+\sqrt{d\lambda}\big)\big]
\end{aligned}
\end{equation}
with probability at least $1-\delta$.
Therefore, we can control the regret by applying \eqref{eq:online_error_bound_2} for all $t \in [T]$:
\begin{align}
    &\sum_{t=1}^TV_{i,1}^{\star,\xi^t}(s_1)-V_{i,1}^{\xi^t}(s_1)\notag\\
    \leq&\sum_{h=1}^H c\big[\sum_{t=1}^T\mathbb{E}\big[\lVert\phi_i(s_h,a_i)\rVert_{S_{t,h}^{-1}} \ \big| \ s_h\sim \big[\sum_{k=1}^K\frac{1}{K}P_{i,1}^{\pi_{1}^{k,t},\underline{V}}\big]\cdots\big[\sum_{k=1}^K\frac{1}{K}P_{i,h-1}^{\pi_{h-1}^{k,t},\underline{V}}\big],\ a_i\sim\text{Unif}(\mathcal{A}_i)\big]\notag\\
    &\cdot A_iH\big(\sqrt{2d\ln(NH+1)+\ln\big(\frac{3TNHnK}{\delta}\big)}+\sqrt{d\lambda}\big)\big]+\sum_{h=1}^H\sum_{t=1}^Tc\big(\sqrt{\frac{\ln A_i}{K}}+\frac{1}{N}\big) \notag\\
    \leq&\sum_{h=1}^H c\bigg[\sqrt{T\sum_{t=1}^T\mathbb{E}\big[\lVert\phi_i(s_h,a_i)\rVert^2_{S_{t,h}^{-1}} \ \Big| \ s_h\sim \big[\sum_{k=1}^K\frac{1}{K}P_{i,1}^{\pi_{1}^{k,t},\underline{V}}\big]\cdots\big[\sum_{k=1}^K\frac{1}{K}P_{i,h-1}^{\pi_{h-1}^{k,t},\underline{V}}\big],\ a_i\sim\text{Unif}(\mathcal{A}_i)\big]}\notag\\
    &\cdot A_iH\big(\sqrt{2d\ln(NH+1)+\ln\big(\frac{3TNHnK}{\delta}\big)}+\sqrt{d\lambda}\big)\bigg]+\sum_{h=1}^H\sum_{t=1}^Tc\big(\sqrt{\frac{\ln  A_i}{K}}+\frac{1}{N}\big)\notag\\
    \leq&\sum_{h=1}^H c\sqrt{T\sum_{t=1}^T\mathsf{tr}\big(S_{t,h}^{-1/2}X_{t,h}S_{t,h}^{-1/2}\big) }\cdot A_iH\big(\sqrt{2d\ln(NH+1)+\ln\big(\frac{3TNHnK}{\delta}\big)}+\sqrt{d\lambda}\big) \notag\\
    &\qquad +\sum_{h=1}^H\sum_{t=1}^Tc\big(\sqrt{\frac{\ln  A_i}{K}}+\frac{1}{N}\big) \label{eq:online_regret_bound}
\end{align}
where the second inequality holds by the Cauchy--Schwarz inequality applied to the sum over $t$ (i.e., $\sum_{t=1}^T x_t \leq \sqrt{T\,\sum_{t=1}^T x_t^2}$ with $x_t=\mathbb{E}[\lVert\phi_i(s,a_i)\rVert_{S_{t,h}^{-1}}\mid \cdots]$, after also using Jensen's inequality, which yields $(\mathbb{E}[\lVert\phi_i(s,a_i)\rVert_{S_{t,h}^{-1}}])^2\leq\mathbb{E}[\lVert\phi_i(s,a_i)\rVert_{S_{t,h}^{-1}}^2]$).

The final inequality holds by the cyclic property of the trace: for any vector $\phi\in\mathbb R^d$ and any positive definite matrix $M\succ 0$,
\begin{align*}
    \lVert\phi\rVert_{M^{-1}}^2=\phi^\top M^{-1}\phi=\mathsf{tr}\big(M^{-1}\phi\phi^\top\big).
\end{align*}

Let $M=S_{t,h}$ and noting that $S_{t,h}=\sum_{\tau=1}^{t-1}X_{\tau,h}+\lambda I_d$ depends only on iterations $\tau<t$ and is therefore measurable with respect to (and constant given) the conditioning event $\{s_h\sim[\sum_k\frac{1}{K}P_{i,1}^{\pi_1^{k,t},\underline{V}}]\cdots[\sum_k\frac{1}{K}P_{i,h-1}^{\pi_{h-1}^{k,t},\underline{V}}],\ a_i\sim\mathrm{Unif}(\mathcal{A}_i)\}$ used at iteration $t$, we may interchange the conditional expectation with the trace by linearity:
\begin{align*}
    &\mathbb{E}\big[\lVert\phi_i(s,a_i)\rVert^2_{S_{t,h}^{-1}} \ \big| \ s_h\sim \big[\sum_{k=1}^K\frac{1}{K}P_{i,1}^{\pi_{1}^{k,t},\underline{V}}\big]\cdots\big[\sum_{k=1}^K\frac{1}{K}P_{i,h-1}^{\pi_{h-1}^{k,t},\underline{V}}\big],\ a_i\sim\text{Unif}(\mathcal{A}_i)\big] \notag\\
   & =\mathsf{tr}\big(S_{t,h}^{-1}\,\mathbb{E}[\phi_i(s_h,a_i)\phi_i(s_h,a_i)^\top \big| \ s_h\sim \big[\sum_{k=1}^K\frac{1}{K}P_{i,1}^{\pi_{1}^{k,t},\underline{V}}\big]\cdots\big[\sum_{k=1}^K\frac{1}{K}P_{i,h-1}^{\pi_{h-1}^{k,t},\underline{V}}\big],\ a_i\sim\text{Unif}(\mathcal{A}_i)\big]  \big) \notag \\
    &=\mathsf{tr}\big(S_{t,h}^{-1}X_{t,h}\big) = \mathsf{tr}\big(S_{t,h}^{-1/2}X_{t,h}S_{t,h}^{-1/2}\big)
\end{align*}
where the last line holds by invoking the definition of $X_{t,h}$ and the cyclic invariance of the trace.

We introduce the following lemme that controls $\sum_{t=1}^T\mathsf{tr}\big(S_{t,h}^{-1/2}X_{t,h}S_{t,h}^{-1/2}\big)$. The proof is postponed to Appendix~\ref{proof:lemma;trace-X}.
\begin{lemma}\label{lemma;trace-X}
    \begin{align}
\sum_{t=1}^T
\mathsf{tr}\!\left(S_{t,h}^{-1/2}X_{t,h}S_{t,h}^{-1/2}\right)
&=
2d\log\!\left(1+\frac{T}{d\lambda}\right)
=
\widetilde{O}(d).
\end{align}
\end{lemma}

Substituting this lemma into \eqref{eq:online_regret_bound} gives
\begin{align*}
    &\sum_{t=1}^TV_{i,1}^{\star,\xi^t}(s_1)-V_{i,1}^{\xi^t}(s_1) \notag \\
    &\leq\sum_{h=1}^{H}c\sqrt{T\sum_{t=1}^T
\mathsf{tr}\!\left(S_{t,h}^{-1/2}X_{t,h}S_{t,h}^{-1/2}\right)}\cdot A_iH\cdot\!\Big(\sqrt{2d\ln(NH+1)+\ln(\tfrac{3TNHnK}{\delta})}+\sqrt{d\lambda}\Big) \notag \\
&\qquad  +\sum_{h=1}^H\sum_{t=1}^Tc\big(\sqrt{\frac{\ln  A_i}{K}}+\frac{1}{N}\big)\notag \\
    &\leq\sum_{h=1}^{H}c\widetilde{\mathcal O}(\sqrt{Td})\cdot A_iH\cdot\!\Big(\sqrt{2d\ln(NH+1)+\ln(\tfrac{3TNHnK}{\delta})}+\sqrt{d\lambda}\Big) +\sum_{h=1}^H\sum_{t=1}^Tc\big(\sqrt{\frac{\ln  A_i}{K}}+\frac{1}{N}\big)\notag \\
    &\overset{\mathsf{(i)}}{\leq} \sum_{h=1}^{H}c\widetilde{\mathcal O}(\sqrt{Td})\cdot A_iH\cdot\widetilde{\mathcal O}(\sqrt d) +\sum_{h=1}^H\sum_{t=1}^Tc\big(\sqrt{\frac{\ln  A_i}{K}}+\frac{1}{N}\big) \notag \\
    &\overset{\mathsf{(ii)}}{\leq} \widetilde{\mathcal O}\big(A_iH^2 d\sqrt{T}\big) + \widetilde{\mathcal O}(H\sqrt{T\ln A_i}+H) \;\overset{\mathsf{(iii)}}{\leq}\; \widetilde{\mathcal O}\big(A_iH^2 d\sqrt{T}\big)
\end{align*}
where (i) bounds $\sqrt{2d\ln(NH+1)+\ln(\tfrac{3TNHnK}{\delta})}+\sqrt{d\lambda}$ by $\widetilde{\mathcal O}(\sqrt d)$ due to the requirement of $\lambda\geq c_0\log(2dnKHT/\delta)=\widetilde{\mathcal O}(1)$ in Lemma~\ref{lemma:lambda-expectation}, (ii) holds by setting $T=N=K$, and (iii) drops the residual because it is dominated by the leading $\widetilde{\mathcal O}(A_iH^2 d\sqrt T)$.





Finally, we conclude that with probability at least $1-\delta$, Algorithm~\ref{alg:lin_robust_Q_FTRL} guarantees that the regret is bounded by
\begin{align}
    \text{Regret}(T)=&\max_{i\in[n]}\sum_{t=1}^T\mathbb{E}_{\pi\sim\xi^t}\big[V_{i,1}^{\star,\pi_{-i}}(s_1)\big]-\mathbb{E}_{\pi\sim\xi^t}\big[V_{i,1}^{\pi}(s_1)\big] \leq \widetilde{\mathcal{O}}\big(d H^2\sqrt{T}\cdot\max_{i\in[n]}A_i\big). \label{eq:final-regret}
\end{align}
Thus, we have finished the proof of the regret bound in Theorem~\ref{thm:main_online}.

\paragraph{Proof of the $\varepsilon$-CCE guarantee.} It remains to show that, when $T>H^4 d^2(\max_{i\in[n]}A_i)^2/\varepsilon^2$, Algorithm~\ref{alg:lin_robust_Q_FTRL} outputs an $\varepsilon$-robust CCE specialized to any fixed initial state $s_1$ (cf.~\eqref{eq:appximate_CCE}). Define the uniform mixture of the output policy distributions of Algorithm~\ref{alg:lin_robust_Q_FTRL} over the $T$ iterates as
\begin{align}\label{eq:online_CCE_mixture}
    \widehat\xi\defn \frac{1}{T}\sum_{t=1}^{T}\xi^t.
\end{align}
 For any fixed agent $i\in[n]$, by linearity of expectation, one has
\begin{align*}
    &\mathbb{E}_{\pi\sim\widehat\xi}\big[V_{i,1}^{\star,\pi_{-i},\sigma_i}(s_1)\big]-\mathbb{E}_{\pi\sim\widehat\xi}\big[V_{i,1}^{\pi,\sigma_i}(s_1)\big]\\
    &=\frac{1}{T}\sum_{t=1}^{T}\Big(\mathbb{E}_{\pi\sim\xi^t}\big[V_{i,1}^{\star,\pi_{-i},\sigma_i}(s_1)\big]-\mathbb{E}_{\pi\sim\xi^t}\big[V_{i,1}^{\pi,\sigma_i}(s_1)\big]\Big)\\
    &\leq \frac{1}{T}\max_{j\in[n]}\sum_{t=1}^{T}\Big(\mathbb{E}_{\pi\sim\xi^t}\big[V_{j,1}^{\star,\pi_{-j},\sigma_j}(s_1)\big]-\mathbb{E}_{\pi\sim\xi^t}\big[V_{j,1}^{\pi,\sigma_j}(s_1)\big]\Big)\;=\;\frac{1}{T}\,\mathrm{Regret}(T),
\end{align*}
where the last equality follows from the definition of $\mathrm{Regret}(T)$ in \eqref{eq:regret-def}. Plugging in the regret bound in \eqref{eq:final-regret} gives
\begin{align*}
    \max_{i\in[n]}\,\Big\{\mathbb{E}_{\pi\sim\widehat\xi}\big[V_{i,1}^{\star,\pi_{-i},\sigma_i}(s_1)\big]-\mathbb{E}_{\pi\sim\widehat\xi}\big[V_{i,1}^{\pi,\sigma_i}(s_1)\big]\Big\}
    \leq \frac{\mathrm{Regret}(T)}{T}
    \leq \widetilde{\mathcal O}\!\left(\frac{d H^2\,\max_{i\in[n]}A_i}{\sqrt{T}}\right) \leq \varepsilon
\end{align*}
with probability at least $1-\delta$, as long as
\begin{align*}
    T\;>\;\frac{H^4 d^2\,(\max_{i\in[n]}A_i)^2}{\varepsilon^2},
\end{align*}
This is exactly the $\varepsilon$-robust CCE condition in \eqref{eq:appximate_CCE}, which completes the proof of Theorem~\ref{thm:main_online}.

\subsection{Proof of auxiliary results.}\label{proof:online-auxiliary}

\subsubsection{Proof of Lemma~\ref{lm:online_transition_estimation_error}}\label{proof:lm:online_transition_estimation_error}

The proof follows the pipeline of that for Lemma~\ref{lm:generative_model_transition_error} in Appendix~\ref{proof:lm:generative_model_transition_error}.  Throughout the proof, $V:\cS\to[0,H]$ is a fixed value function.

\paragraph{Step 1: Proof of the transition estimation gap.} First, consider any fixed $(h,i,k,t)\in[H]\times[n]\times[K]\times[T]$. Recalling the auxiliary defintions in  \eqref{eq:generative_model_transition_error_PV_definition}, and then applying strong-duality formulation \citep{iyengar2005robust} for the dual form of the robust Bellman operator \citep[Lemma 1]{shi2023curious} yields
\begin{align}
    \big|P_{i,h,s,a_i}^{\widehat{\pi}_{-i}^{k,t},V}V-\widehat{P}_{i,h,s,a_i}^{\widehat{\pi}_{-i}^{k,t},V}V\big|
    \leq \max_{\alpha\in[\min_s V(s),\,\max_s V(s)]}\big|P_{h,s,a_i}^{\widehat{\pi}_{-i}^{k,t}}[V]_\alpha-P_{i,h,s,a_i}^{k,t}[V]_\alpha\big|,\notag
\end{align}
where $P_{i,h,s,a_i}^{k,t}(\cdot)\defn\phi_i(s,a_i)^\top\widehat\mu_{i,h}^{k,t}(\cdot)$ is the estimated transition kernel defined in \eqref{eq:ridge_regression_online}.

For any fixed $\alpha$, the linear-MG assumption in \eqref{eq:linaer-mg-assumption} gives $P_{h,s,a_i}^{\widehat{\pi}_{-i}^{k,t}}[V]_\alpha=\phi_i(s,a_i)^\top\mu_{i,h}^{\widehat{\pi}_{-i}^{k,t}}[V]_\alpha$. Therefore, applying Theorem~\ref{lm:ridge-regression-concentration} to the online Gram matrix $\Lambda_{i,h}^{k,t}$ defined in \eqref{eq:ridge_regression_online} with $f=[V]_\alpha\in[0,H]$ yields that with probability at least $1-\delta$,
\begin{align}\label{eq:online_transition_pointwise}
    \big|P_{h,s,a_i}^{\widehat{\pi}_{-i}^{k,t}}[V]_\alpha-P_{i,h,s,a_i}^{k,t}[V]_\alpha\big|
    \leq \big\lVert\phi_i(s,a_i)\big\rVert_{(\Lambda_{i,h}^{k,t})^{-1}}\big(2H\sqrt{d\ln(NH+1)+\ln\!\big(\tfrac{3NH}{\delta}\big)}+H\sqrt{d\lambda}\big),
\end{align}
for all $(s,a_i)\in\cS\times\cA_i$.

To derive the union bound over $\alpha\in[0,H]$, we apply the same $\varepsilon_1$-net argument as in \eqref{eq:generative_model_transition_estimation_error_mediate_eq_final}: for $\varepsilon_1=1/N$, we have $|N_{\varepsilon_1}|\leq 3HN$. Observing the fact that the function $\alpha\mapsto|P_{h,s,a_i}^{\widehat{\pi}_{-i}^{k,t}}[V]_\alpha-P_{i,h,s,a_i}^{k,t}[V]_\alpha|$ is $1$-Lipschitz, combined with the pointwise bound in \eqref{eq:online_transition_pointwise} gives, with probability at least $1-\delta$, for all $(s,a_i)$ and $\alpha\in[0,H]$,
\begin{align*}
    \big|P_{h,s,a_i}^{\widehat{\pi}_{-i}^{k,t}}[V]_\alpha-P_{i,h,s,a_i}^{k,t}[V]_\alpha\big|
    \leq \big\lVert\phi_i(s,a_i)\big\rVert_{(\Lambda_{i,h}^{k,t})^{-1}}\big(2H\sqrt{d\ln(NH+1)+\ln\!\big(\tfrac{3NH\cdot 3HN}{\delta}\big)}+H\sqrt{d\lambda}\big)+\frac{1}{N}.
\end{align*}

With above pointwise bound for any fixed $(h,i,k,t)\in[H]\times[n]\times[K]\times[T]$, we replace $\delta$ by $\delta/(nKHT)$ to derive the union bound over $[H]\times[n]\times[K]\times[T]$: with probability at least $1-\delta$, for all $(h,i,k,t)$ and all $(s,a_i)\in\cS\times\cA_i$,
\begin{align*}
    \big|P_{i,h,s,a_i}^{\widehat{\pi}_{-i}^{k,t},V}V-\widehat{P}_{i,h,s,a_i}^{\widehat{\pi}_{-i}^{k,t},V}V\big|
    \leq \big\lVert\phi_i(s,a_i)\big\rVert_{(\Lambda_{i,h}^{k,t})^{-1}}\big(2H\sqrt{d\ln(NH+1)+2\ln\!\big(\tfrac{3TKNHn}{\delta}\big)}+H\sqrt{d\lambda}\big)+\frac{1}{N}.
\end{align*}

\paragraph{Step 2: Proof of the reward estimation gap.} The argument is identical to the generative case, except: (i) the Gram matrix is now $\Lambda_{i,h}^{k,t}$ from \eqref{eq:ridge_regression_online} with regularization $\lambda\geq 1$, so the ridge-regression bias term becomes $H\sqrt{d\lambda}$ (rather than $H\sqrt d$ in the generative case due to the $\lambda=1$ specialization); (ii) we retain $\lVert\phi_i(s,a_i)\rVert_{(\Lambda_{i,h}^{k,t})^{-1}}$ explicitly; (iii) the union bound is over $(h,i,k,t)\in[H]\times[n]\times[K]\times[T]$, contributing an extra $\ln T$ factor inside the logarithmic term. Concretely, applying Theorem~\ref{lm:ridge-regression-concentration} to the scalar target $f=\mathbb E_{a_{-i}\sim\widehat\pi_{-i}^{k,t}(s)}[r_{i,h}(s,\mathbf a)]\in[0,H]$ with the online Gram matrix and replacing $\delta$ by $\delta/(nKHT)$ in the resulting union bound, we have with probability at least $1-\delta$,
\begin{align*}
    \big|\mathbb E_{a_{-i}\sim\widehat\pi_{-i}^{k,t}(s)}[r_{i,h}(s,\mathbf a)]-r_{i,h}^{k,t}(s,a_i)\big|
    \leq \big\lVert\phi_i(s,a_i)\big\rVert_{(\Lambda_{i,h}^{k,t})^{-1}}\big(2H\sqrt{d\ln(NH+1)+\ln\!\big(\tfrac{3NTHnK}{\delta}\big)}+H\sqrt{d\lambda}\big),
\end{align*}
holds simultaneously for all $(h,i,k,t)$ and $(s,a_i)$.


\subsubsection{Proof of Lemma~\ref{lm:online_optimistic_value_function}}\label{proof:lm:online_optimistic_value_function}

We will prove Lemma~\ref{lm:online_optimistic_value_function} through induction arguments. For the base case $h = H+1$, by definition,
\begin{align}
\overline{V}_{i,H+1}^t(s) = V_{i,H+1}^{\star,\xi^t}(s) = 0
\end{align}
for all $s \in \mathcal{S}$. Thus, assuming that the desired fact is satisfied at time step $h+1$, namely $\overline{V}_{i,h+1}^t(s) \geq V_{i,h+1}^{\star,\xi^t}(s)$ holds for all $s \in \mathcal{S}$, we will prove $\overline{V}_{i,h}^t(s) \geq V_{i,h}^{\star,\xi^t}(s)$ for all $s\in\cS$.

Towards this, we first define several auxiliary notations as follows. For all agents $i \in [n]$ and $(t, k, s, h, a_i) \in [T] \times [K] \times \mathcal{S} \times [H] \times \mathcal{A}_i$, let $P_{i,h,s,a_i}^{k,t}(\cdot)\defn\phi_i(s,a_i)^\top\widehat\mu_{i,h}^{k,t}(\cdot)$ denote the online estimated transition kernel from \eqref{eq:ridge_regression_online}. Define the worst-case kernels under the true and the estimated nominal transitions, respectively, as
\begin{align*}
    & P_{i,h,s,a_i}^{\pi_{-i}^{k,t},\overline{V}} \defn \arg\min_{\mathcal{P}\in\mathcal{U}^{\sigma_i}\big(P_{h,s,a_i}^{\pi_{-i}^{k,t}}\big)}\mathcal{P}\overline{V}_{i,h+1}^t \qquad \text{and} \qquad
    \widehat{P}_{i,h,s,a_i}^{\pi_{-i}^{k,t},\overline{V}} \defn \arg\min_{\mathcal{P}\in\mathcal{U}^{\sigma_i}\big(P_{i,h,s,a_i}^{k,t}\big)}\mathcal{P}\overline{V}_{i,h+1}^t.
\end{align*}
To continue, let $\xi^t$ be the uniform distribution over the $K$ product policies $\{\pi^{k,t}\}_{k\in[K]}$ output by Algorithm~\ref{alg:lin_robust_Q_FTRL} at iteration $t$, where each $\pi^{k,t}=\{\pi_h^{k,t}\}_{h\in[H]}$ with $\pi_h^{k,t}:\cS\to\prod_{i\in[n]}\Delta(\cA_i)$.
Recalling the definition of the expected best-response value function $V_{i,h}^{\star,\xi^t}$ in \eqref{eq:own-x1} and \eqref{eq:defn-optimal-V}, by the robust Bellman equation \eqref{eq:robust-bellman-equation}, we have
\begin{align}
     V_{i,h}^{\star,\xi^t}(s)
    =&\max_{a_i\in\mathcal{A}_i}\sum_{k=1}^K\frac{1}{K}\big[\mathbb{E}_{a_{-i}\sim\pi_{-i,h}^{k,t}(s)}\big[r_{i,h}(s,\mathbf{a})\big]+\inf_{\mathcal{P}\in\mathcal{U}^{\sigma_i}\Big(P_{h,s,a_i}^{\pi_{-i}^{k,t}}\Big)}\mathcal{P}V_{i,h+1}^{\star,\xi^t}\big]\notag\\
    \overset{\mathsf{(i)}}{\leq}&\max_{a_i\in\mathcal{A}_i}\sum_{k=1}^K\frac{1}{K}\big[\mathbb{E}_{a_{-i}\sim\pi_{-i,h}^{k,t}(s)}\big[r_{i,h}(s,\mathbf{a})\big]+\inf_{\mathcal{P}\in\mathcal{U}^{\sigma_i}\Big(P_{h,s,a_i}^{\pi_{-i}^{k,t}}\Big)}\mathcal{P}\overline{V}_{i,h+1}^t\big]\notag\\
    =&\max_{a_i\in\mathcal{A}_i}\sum_{k=1}^K\frac{1}{K}\big[ \mathbb{E}_{a_{-i}\sim\pi_{-i,h}^{k,t}(s)}\big[r_{i,h}(s,\mathbf{a})\big]+P_{i,h,s,a_i}^{\pi_{-i}^{k,t},\overline{V}}\overline{V}_{i,h+1}^t \big]\notag\\
    \leq&\max_{a_i\in\mathcal{A}_i}\sum_{k=1}^K\frac{1}{K}\big[r_{i,h}^{k,t}(s,a_i)+\widehat{P}_{i,h,s,a_i}^{\pi_{-i}^{k,t},\overline{V}}\overline{V}^t_{i,h+1}+\big|r_{i,h}^{k,t}(s,a_i)-\mathbb{E}_{a_{-i}\sim\pi_{-i,h}^{k,t}}[r_{i,h}(s,\mathbf{a})]\big|\notag\\
    &\qquad+\big|\widehat{P}_{i,h,s,a_i}^{\pi_{-i}^{k,t},\overline{V}}\overline{V}^t_{i,h+1}-P_{i,h,s,a_i}^{\pi_{-i}^{k,t},\overline{V}}\overline{V}^t_{i,h+1}\big|\big],\notag
\end{align}
where (i) holds by the induction assumption $V_{i,h+1}^{\star,\xi^t}(s)\leq \overline{V}_{i,h+1}^t(s)$ for all $s\in\mathcal{S}$.
In the next step, we apply a similar analysis as in Lemma~\ref{lm:generative_model_transition_error}, from which we obtain
\begin{align}
    & V_{i,h}^{\star,\xi^t}(s)\notag\\
    \leq&\max_{a_i\in\mathcal{A}_i}\sum_{k=1}^K\frac{1}{K}\big[r_{i,h}^{k,t}(s,a_i)+\widehat{P}_{i,h,s,a_i}^{\pi_{-i}^{k,t},\overline{V}}\overline{V}^t_{i,h+1}+\big|r_{i,h}^{k,t}(s,a_i)-\mathbb{E}_{a_{-i}\sim\pi_{-i,h}^{k,t}}[r_{i,h}(s,\mathbf{a})]\big|\notag\\
    &\qquad+\big|\widehat{P}_{i,h,s,a_i}^{\pi_{-i}^{k,t},\overline{V}}\overline{V}^t_{i,h+1}-P_{i,h,s,a_i}^{\pi_{-i}^{k,t},\overline{V}}\overline{V}^t_{i,h+1}\big|\big].\notag\\
    \overset{\mathsf{(i)}}{\leq}&\max_{a_i\in\mathcal{A}_i}\sum_{k=1}^K\frac{1}{K}\big\lVert\phi(s,a_i)\big\rVert_{\big(\Lambda_{i,h}^{k,t}\big)^{-1}}\big(2H\sqrt{d\ln(NH+1)+\ln\big(\frac{3TNHnK}{\delta}\big)}+H\sqrt{d}\big)+\frac{1}{N}\notag\\
    &+\max_{a_i\in\mathcal{A}_i}\sum_{k=1}^K\big[r_{i,h}^{k,t}(s,a_i)+\widehat{P}_{i,h,s,a_i}^{\pi_{-i}^{k,t},\overline{V}}\overline{V}^t_{i,h+1}\big]\notag\\
    \overset{\mathsf{(ii)}}{\leq}&\max_{a_i\in\mathcal{A}_i}\sum_{k=1}^K\frac{1}{K}\big\lVert\phi(s,a_i)\big\rVert_{\big(\Lambda_{i,h}^{k,t}\big)^{-1}}\big(4H\sqrt{d\ln(NH+1)+\ln\big(\frac{3TNHnK}{\delta}\big)}+2H\sqrt{d}\big)\notag\\
    &+\frac{1}{N}+2H\sqrt{\frac{\ln A_i}{K}}+\sum_{k=1}^K\frac{1}{K}\mathbb{E}_{a_i\sim\pi_{i,h}^{k,t}(s)}\big[r_{i,h}^{k,t}(s,a_i)+\widehat{P}_{i,h,s,a_i}^{\pi_{-i}^{k,t},\overline{V}}\overline{V}^t_{i,h+1}\big] \nonumber \\
    & = \beta_{i,h,1}^t(s) +  \sum_{k=1}^K\frac{1}{K}\mathbb{E}_{a_i\sim\pi_{i,h}^{k,t}(s)}\big[r_{i,h}^{k,t}(s,a_i)+\widehat{P}_{i,h,s,a_i}^{\pi_{-i}^{k,t},\overline{V}}\overline{V}^t_{i,h+1}\big] \label{eq:online_optimistic_estimation}.
\end{align}
Here, (i) follows from Lemma~\ref{lm:online_transition_estimation_error}, (ii) holds by the bound for standard FTRL in Theorem~\ref{lm:ridge-regression-concentration}, and th final line follows from the definition of $\beta_{i,h,1}^t(s)$ in \eqref{eq:online_optimistic_bonus} as follows:
\begin{align*}
    \beta_{i,h,1}^t(s)=&\max_{a_i\in\mathcal{A}_i}\sum_{k=1}^K\frac{1}{K}\big\lVert\phi(s,a_i)\big\rVert_{\big(\Lambda_{i,h}^{k,t}\big)^{-1}}\big(4H\sqrt{d\ln(NH+1)+\ln\big(\frac{3TNHnK}{\delta}\big)}+2H\sqrt{d}\big)\\
    &+\frac{1}{N}+2H\sqrt{\frac{\ln A_i}{K}}.
\end{align*}

Finally, recalling the definition of the optimistic estimation $\overline{V}_{i,h}^t(s)$ in \eqref{eq:estimate_value_function}
      \begin{align}
          \overline{V}_{i,h}^t(s)& = \textstyle\min\big\{ \sum_{k=1}^{K} \frac{1}{K}\big<\pi_{i,h}^{k,t}(\cdot\mid s),\overline{q}_{i,h}^{k,t}(s,\cdot)\big>
          + \beta_{i,h,1}^{t}(s),~H-h+1\big\},
      \end{align}
combining\eqref{eq:online_optimistic_estimation} with the elementary fact that $V_{i,h}^{\star,\xi^t}(s)\leq H-h+1$ gives that
\begin{align*}
   \forall (h,i,t,s)\in[H]\times[n]\times[T] \times\cS: \quad V_{i,h}^{\star,\xi^t}(s)\leq\overline{V}_{i,h}^t(s)
\end{align*}
holds with probability at least $1-\delta$.

\subsubsection{Proof of Lemma~\ref{lm:online_pessimistic_value_function}}\label{proof:lm:online_pessimistic_value_function}

We will prove Lemma~\ref{lm:online_pessimistic_value_function} through induction arguments. For the base case $h = H+1$, by definition,
\begin{align*}
  \forall (i,s)\in[n]\times\mathcal{S}: \quad \underline{V}_{i,H+1}^t(s)=0=V_{i,H+1}^{\xi^t}(s).
\end{align*}
 Thus, assuming that the desired fact is satisfied at time step $h+1$, namely $\underline{V}_{i,h+1}^t(s)\leq V_{i,h+1}^{\xi^t}(s)$ holds for all $s \in \mathcal{S}$, we will prove $\underline{V}_{i,h}^t(s) \leq V_{i,h}^{\xi^t}(s)$ for all $s\in\cS$.

Towards this, we first define several auxiliary notations as follows. For all agents $i \in [n]$ and $(t, k, s, h, a_i) \in [T] \times [K] \times \mathcal{S} \times [H] \times \mathcal{A}_i$, with $P_{i,h,s,a_i}^{k,t}$ the online estimated transition kernel from \eqref{eq:ridge_regression_online}, we define the worst-case kernels (for the pessimistic value) under the true and the estimated nominal transitions, respectively, as
\begin{equation}
\label{eq:online_pessimistic_value_transition_def}
\begin{aligned}
    & P_{i,h,s,a_i}^{\pi_{-i}^{k,t},\underline{V}}\defn\arg\min_{\mathcal{P}\in\mathcal{U}^{\sigma_i}\big(P_{h,s,a_i}^{\pi_{-i}^{k,t}}\big)}\mathcal{P}\underline{V}_{i,h+1}^t \qquad \text{and} \qquad \widehat{P}_{i,h,s,a_i}^{\pi_{-i}^{k,t},\underline{V}}\defn\arg\min_{\mathcal{P}\in\mathcal{U}^{\sigma_i}\big(P_{i,h,s,a_i}^{k,t}\big)}\mathcal{P}\underline{V}_{i,h+1}^t.
\end{aligned}
\end{equation}
To continue, recall that $\xi^t$ is the uniform distribution over product policies $\{\pi_h^{k,t}\}_{(h,k)\in[H]\times[K]}$ output by Algorithm~\ref{alg:lin_robust_Q_FTRL}. Recalling the definition of the expected version of the value function $V_{i,h}^{\xi^t}$ in \eqref{eq:own-x1}, and referring to the robust Bellman equation in \eqref{eq:robust-bellman-equation}, we have
\begin{align}
    V_{i,h}^{\xi^t}(s)
    =&\sum_{k=1}^K\frac{1}{K}\mathbb{E}_{a_i\sim\pi_{i,h}^{k,t}(s)}\big[\mathbb{E}_{a_{-i}\sim\pi_{-i,h}^{k,t}(s)}\big[r_{i,h}(s,\mathbf{a})\big]+\inf_{\mathcal{P}\in\mathcal{U}^{\sigma_i}\big(P_{h,s,a_i}^{\pi_{-i}^{k,t}}\big)}\mathcal{P}V_{i,h+1}^{\xi^t}\big] \nonumber\\
    \overset{\mathsf{(i)}}{\geq}&\sum_{k=1}^K\frac{1}{K}\mathbb{E}_{a_i\sim\pi_{i,h}^{k,t}(s)}\big[\mathbb{E}_{a_{-i}\sim\pi_{-i,h}^{k,t}(s)}\big[r_{i,h}(s,\mathbf{a})\big]+\inf_{\mathcal{P}\in\mathcal{U}^{\sigma_i}\big(P_{h,s,a_i}^{\pi_{-i}^{k,t}}\big)}\mathcal{P}\underline{V}_{i,h+1}^t\big] \nonumber\\
    \overset{\mathsf{(ii)}}{=}&\sum_{k=1}^K\frac{1}{K}\mathbb{E}_{a_i\sim\pi_{i,h}^{k,t}(s)}\big[\mathbb{E}_{a_{-i}\sim\pi_{-i,h}^{k,t}(s)}\big[r_{i,h}(s,\mathbf{a})\big]+P_{i,h,s,a_i}^{\pi_{-i}^{k,t},\underline{V}}\underline{V}_{i,h+1}^t\big] \nonumber\\
    \overset{\mathsf{(iii)}}{\geq}&\sum_{k=1}^K\frac{1}{K}\mathbb{E}_{a_i\sim\pi_{i,h}^{k,t}(s)}\big[r_{i,h}^{k,t}(s,a_i)+\widehat{P}_{i,h,s,a_i}^{\pi_{-i}^{k,t},\underline{V}}\underline{V}_{i,h+1}^t\big] \nonumber\\
    &-\sum_{k=1}^K\frac{1}{K}\mathbb{E}_{a_i\sim\pi_{i,h}^{k,t}(s)}\big|r_{i,h}^{k,t}(s,a_i)-\mathbb{E}_{a_{-i}\sim\pi_{-i,h}^{k,t}(s)}\big[r_{i,h}(s,\mathbf{a})\big]\big| \nonumber \\
    &-\sum_{k=1}^K\frac{1}{K}\mathbb{E}_{a_i\sim\pi_{i,h}^{k,t}(s)}\big|\widehat{P}_{i,h,s,a_i}^{\pi_{-i}^{k,t},\underline{V}}\underline{V}_{i,h+1}^t-P_{i,h,s,a_i}^{\pi_{-i}^{k,t},\underline{V}}\underline{V}_{i,h+1}^t\big|, \label{eq:online_pessimistic_value_mediate_eq1}
\end{align}
where (i) holds by the induction hypothesis $\underline{V}_{i,h+1}^t(s) \leq V_{i,h+1}^{\xi^t}(s)$ for all $s\in\mathcal{S}$, (ii) uses the notations defined in \eqref{eq:online_pessimistic_value_transition_def}, and (iii) follows from the elementary inequality $x \geq y - |x-y|$ applied separately to the reward term $\mathbb{E}_{a_{-i}\sim\pi_{-i,h}^{k,t}(s)}\big[r_{i,h}(s,\mathbf{a})\big]$ against its linear estimate $r_{i,h}^{k,t}(s,a_i)$ and to the transition kernel term $P_{i,h,s,a_i}^{\pi_{-i}^{k,t},\underline{V}}\underline{V}_{i,h+1}^t$ against its empirical counterpart $\widehat{P}_{i,h,s,a_i}^{\pi_{-i}^{k,t},\underline{V}}\underline{V}_{i,h+1}^t$.

In the next step, we apply Lemma~\ref{lm:online_transition_estimation_error}, which simultaneously controls the reward and transition estimation errors, from which we obtain that, with probability at least $1-\delta$:
\begin{align}
    V_{i,h}^{\xi^t}(s)
    \geq&\sum_{k=1}^K\frac{1}{K}\mathbb{E}_{a_i\sim\pi_{i,h}^{k,t}(s)}\big[r_{i,h}^{k,t}(s,a_i)+\widehat{P}_{i,h,s,a_i}^{\pi_{-i}^{k,t},\underline{V}}\underline{V}_{i,h+1}^t\big]\notag\\
    &-\sum_{k=1}^K\frac{1}{K}\mathbb{E}_{a_i\sim\pi_{i,h}^{k,t}(s)}\big|r_{i,h}^{k,t}(s,a_i)-\mathbb{E}_{a_{-i}\sim\pi_{-i,h}^{k,t}(s)}\big[r_{i,h}(s,\mathbf{a})\big]\big|\notag\\
    &-\sum_{k=1}^K\frac{1}{K}\mathbb{E}_{a_i\sim\pi_{i,h}^{k,t}(s)}\big|\widehat{P}_{i,h,s,a_i}^{\pi_{-i}^{k,t},\underline{V}}\underline{V}_{i,h+1}^t-P_{i,h,s,a_i}^{\pi_{-i}^{k,t},\underline{V}}\underline{V}_{i,h+1}^t\big|\notag\\
    \overset{\mathsf{(i)}}{\geq}&\sum_{k=1}^K\frac{1}{K}\mathbb{E}_{a_i\sim\pi_{i,h}^{k,t}(s)}\big[r_{i,h}^{k,t}(s,a_i)+\widehat{P}_{i,h,s,a_i}^{\pi_{-i}^{k,t},\underline{V}}\underline{V}_{i,h+1}^t\big]\notag\\
    &-\sum_{k=1}^K\frac{1}{K}\mathbb{E}_{a_i\sim\pi_{i,h}^{k,t}(s)}\big\lVert\phi(s,a_i)\big\rVert_{\big(\Lambda_{i,h}^{k,t}\big)^{-1}}\big(4H\sqrt{d\ln(NH+1)+\ln\big(\frac{3TNHnK}{\delta}\big)}+2H\sqrt{d}\big)-\frac{1}{N}\notag\\
    =&\sum_{k=1}^K\frac{1}{K}\mathbb{E}_{a_i\sim\pi_{i,h}^{k,t}(s)}\big[r_{i,h}^{k,t}(s,a_i)+\widehat{P}_{i,h,s,a_i}^{\pi_{-i}^{k,t},\underline{V}}\underline{V}_{i,h+1}^t\big]-\beta_{i,h,2}^t(s),\label{eq:online_pessimistic_estimate_eq_final}
\end{align}
where (i) follows from Lemma~\ref{lm:online_transition_estimation_error},
 and the final line follows from the definition of the pessimistic bonus $\beta_{i,h,2}^t(s)$ in \eqref{eq:online_pessimistic_bonus}, which we recall below:
\begin{align*}
    \beta_{i,h,2}^t(s)=\sum_{k=1}^K\frac{1}{K}\mathbb{E}_{a_i\sim\pi_{i,h}^{k,t}(s)}\big\lVert\phi(s,a_i)\big\rVert_{\big(\Lambda_{i,h}^{k,t}\big)^{-1}}\big(4H\sqrt{d\ln(NH+1)+\ln\big(\frac{3TNHnK}{\delta}\big)}+2H\sqrt{d}\big)+\frac{1}{N}.
\end{align*}

Finally, recalling the definition of the pessimistic estimation $\underline{V}_{i,h}^t(s)$ in \eqref{eq:estimate_value_function}
\begin{align}
    \underline{V}_{i,h}^t(s)=\textstyle\max\big\{ \sum_{k=1}^{K} \frac{1}{K}\big<\pi_{i,h}^{k,t}(\cdot\mid s),\underline{q}_{i,h}^{k,t}(s,\cdot)\big>
    - \beta_{i,h,2}^{t}(s),~0\big\},
\end{align}
combining \eqref{eq:online_pessimistic_estimate_eq_final} with the elementary fact that $V_{i,h}^{\xi^t}(s)\geq 0$ gives that
\begin{align*}
   \forall (h,i,t,s)\in[H]\times[n]\times[T] \times\cS: \quad \underline{V}_{i,h}^t(s)\leq V_{i,h}^{\xi^t}(s)
\end{align*}
holds with probability at least $1-\delta$.

\subsubsection{Proof of Lemma~\ref{lemma:lambda-expectation}}\label{proof:lemma:lambda-expectation}


we define the centered symmetric matrices
\begin{align*}
    W_m\defn\phi_i^{(m)}(\phi_i^{(m)})^\top-X_{\tau(m),h}, \quad \text{so that} \quad \sum_{m=1}^{\lvert\cD_i\rvert}W_m=\Lambda_{i,h}^{k,t}-\bar\Lambda,
\end{align*}
where each $W_m$ are independent conditional on the past $t-1$ ronuds. Before continuing, note the basic fact
\begin{align}\label{eq:phi-fact}
     \big\lVert\phi_i^{(m)}\big\rVert_2\leq 1 \qquad \text{and} \qquad  \lVert\phi_i^{(m)}(\phi_i^{(m)})^\top\rVert_{\mathrm{op}}=\lVert\phi_i^{(m)}\rVert_2^2\leq 1.
\end{align}
Therefore, we have
\begin{align}\label{eq:Wm-bound}
    \lVert W_m\rVert_{\mathrm{op}}&\leq\big\lVert\phi_i^{(m)}(\phi_i^{(m)})^\top\big\rVert_{\mathrm{op}}+\big\lVert X_{\tau(m),h}\big\rVert_{\mathrm{op}} \leq 2
\end{align}
where the first inequality follows the triangle inequality, and the last inequality holds by \eqref{eq:phi-fact} and \begin{align}
\lVert X_{\tau(m),h}\rVert_{\mathrm{op}}\leq\mathbb E[\lVert\phi_i^{(m)}(\phi_i^{(m)})^\top\rVert_{\mathrm{op}}]\leq 1,
\end{align}
which is achieved by applying Jensen's inequality.

In addition, for the variance term,
\begin{align*}
    \mathbb E[W_m^2] &=\mathbb E[(\phi_i^{(m)}(\phi_i^{(m)})^\top)^2]-X_{\tau(m),h}^2 \preceq\mathbb E\!\big[(\phi_i^{(m)}(\phi_i^{(m)})^\top)^2\big]\preceq\mathbb E\!\big[\phi_i^{(m)}(\phi_i^{(m)})^\top\big]=X_{\tau(m),h},
\end{align*}
where the first inequality holds by the fact that $X_{\tau(m),h}^2\succeq 0$, and the second inequality holds by $(\phi_i^{(m)}(\phi_i^{(m)})^\top)^2=\lVert\phi_i^{(m)}\rVert_2^2\,\phi_i^{(m)}(\phi_i^{(m)})^\top\preceq\phi_i^{(m)}(\phi_i^{(m)})^\top$, since $\lVert\phi_i^{(m)}\rVert_2\leq 1$ by \eqref{eq:phi-fact}.

Consequently,
\begin{align}\label{eq:matrix-variance}
    V\defn\sum_{m=1}^{\lvert\cD_i\rvert}\mathbb E[W_m^2]\preceq\sum_{m=1}^{\lvert\cD_i\rvert}X_{\tau(m),h}=\bar\Lambda-\lambda I_d\preceq\bar\Lambda.
\end{align}


With above results for $W_n$ and the variance of the sum in hand, we define the conjugated increments
\begin{align*}
    Y_m\defn\bar\Lambda^{-1/2}W_m\bar\Lambda^{-1/2}, \qquad \text{so that} \qquad \sum_{m=1}^{\lvert\cD_i\rvert}Y_m=\bar\Lambda^{-1/2}\big(\Lambda_{i,h}^{k,t}-\bar\Lambda\big)\bar\Lambda^{-1/2}.
\end{align*}
Notice that
\begin{align}\label{eq:Ym-bound}
    \lVert Y_m\rVert_{\mathrm{op}}\leq\lVert\bar\Lambda^{-1/2}\rVert_{\mathrm{op}}^2\,\lVert W_m\rVert_{\mathrm{op}}\leq\tfrac{2}{\lambda}
\end{align}
where the last inequality holds by \eqref{eq:Wm-bound} and the fact $\lVert\bar\Lambda^{-1/2}\rVert_{\mathrm{op}}^2=\lVert\bar\Lambda^{-1}\rVert_{\mathrm{op}}\leq\lambda^{-1}$ since $\bar\Lambda\succeq\lambda I_d$. For the variance term,
\begin{align*}
    \sum_{m=1}^{\lvert\cD_i\rvert}\mathbb E[Y_m^2]
    &=\bar\Lambda^{-1/2}\!\bigg(\sum_{m=1}^{\lvert\cD_i\rvert}\mathbb E\!\big[W_m\bar\Lambda^{-1}W_m\big]\bigg)\bar\Lambda^{-1/2} \notag \\
    &\overset{\mathsf{(i)}}{\preceq} \tfrac{1}{\lambda}\,\bar\Lambda^{-1/2}V\bar\Lambda^{-1/2} \overset{\mathsf{(ii)}}{\preceq}
    \tfrac{1}{\lambda}\,\bar\Lambda^{-1/2}\bar\Lambda\bar\Lambda^{-1/2}=\tfrac{1}{\lambda}I_d.
\end{align*}
where (i) holds by $\bar\Lambda^{-1}\preceq\lambda^{-1}I_d$, and (ii) follows from \eqref{eq:matrix-variance}.
So we have
\begin{align}\label{eq:Ym-variance}
    \sigma'^2\defn\Big\lVert\sum_{m=1}^{\lvert\cD_i\rvert}\mathbb E[Y_m^2]\Big\rVert_{\mathrm{op}}\leq\tfrac{1}{\lambda}.
\end{align}
Applying matrix Bernstein \citep[Theorem~6.1.1]{tropp2015introduction} to the independent, mean-zero, symmetric matrices $\{Y_m\}_{m=1}^{\lvert\cD_i\rvert}$ with $R' = \frac{2}{\lambda}$ and $\sigma'^2 = \frac{1}{\lambda}$ as in \eqref{eq:Ym-bound} and \eqref{eq:Ym-variance}, we obtain, for any $\theta>0$,
\begin{align*}
    \mathbb{P}\Big[\Big\lVert\!\sum_{m=1}^{\lvert\cD_i\rvert}Y_m\Big\rVert_{\mathrm{op}}\geq\theta\Big]\leq 2d\exp\!\Big(-\frac{\theta^2/2}{\sigma'^2+R'\theta/3}\Big).
\end{align*}
Setting the RHS to $\delta'\defn\delta/(nKHT)$ and inverting yields that with probability at least $1-\delta'$,
\begin{align}\label{eq:bernstein-conjugated}
    \left\|\sum_{m=1}^{\lvert\cD_i\rvert}Y_m \right\|_{\mathrm{op}} &=  \Big\lVert\bar\Lambda^{-1/2}\big(\Lambda_{i,h}^{k,t}-\bar\Lambda\big)\bar\Lambda^{-1/2}\Big\rVert_{\mathrm{op}} \leq \sqrt{2\sigma'^2\log(2d/\delta')}+\tfrac{2R'\log(2d/\delta')}{3} \notag \\
&\leq  \sqrt{\tfrac{2\log(2d/\delta')}{\lambda}}+\tfrac{4\log(2d/\delta')}{3\lambda} \leq \frac{1}{2}
\end{align}
where the last inequality holds by setting $\lambda\geq c_0\log(2dnKHT/\delta)$ for a sufficiently large universal constant $c_0$, so that both terms on the right-hand side of \eqref{eq:bernstein-conjugated} are at most $1/4$.

Regarding \eqref{eq:bernstein-conjugated}, we have $-\tfrac{1}{2}I_d\preceq\bar\Lambda^{-1/2}(\Lambda_{i,h}^{k,t}-\bar\Lambda)\bar\Lambda^{-1/2}\preceq\tfrac{1}{2}I_d$, which directly gives $-\tfrac{1}{2}\bar\Lambda\preceq\Lambda_{i,h}^{k,t}-\bar\Lambda\preceq\tfrac{1}{2}\bar\Lambda$, yielding
\begin{align}
    \tfrac{1}{2}\,\bar\Lambda\;\preceq\;\Lambda_{i,h}^{k,t}\;\preceq\;\tfrac{3}{2}\,\bar\Lambda.
\end{align}

\subsubsection{Proof of Lemma~\ref{lemma;trace-X}}\label{proof:lemma;trace-X}

Denote $\mu_1,\ldots,\mu_d \geq0$ denote the eigenvalues of the positive definite matrix $S_{t,h}^{-1/2}X_{t,h}S_{t,h}^{-1/2}$, since $X_{t,h}$ is positive semi-definite.
We have
\begin{align}
\|S_{t,h}^{-1/2}X_{t,h}S_{t,h}^{-1/2}\|_{\mathrm{op}} \leq\lVert S_{t,h}^{-1}\rVert_{\mathrm{op}}\lVert X_{t,h}\rVert_{\mathrm{op}}\leq 1
\end{align}
since the fact $S_{t,h}\succeq\lambda I_d$ in \eqref{eq:bar-Lambda-vs-S-0} with $\lambda\geq 1$ gives $\lVert S_{t,h}^{-1}\rVert_{\mathrm{op}}\leq 1$, and $\lVert X_{t,h}\rVert_{\mathrm{op}}\leq\mathsf{tr}(X_{t,h})=\mathbb E_{s_h,a_i}[\lVert\phi_i(s_h,a_i)\rVert_2^2]\leq 1$.
Therefore, we have $0\leq \mu_i \leq 1$ and thus
\begin{align}
\log(1+\mu_i) \geq \frac{\mu_i}{2}
\end{align}
due to $\log(1+x)\geq x/2$ for $x\in[0,1]$.
Applying the fact to each eigenvalue, we arrive at
\begin{align}
\mathsf{tr}\big(S_{t,h}^{-1/2}X_{t,h}S_{t,h}^{-1/2}\big) = \sum_{i=1}^d \mu_i  \leq 2 \sum_{i=1}^d \log\det(1 +\mu_i) = 2\log\det(I_d+ S_{t,h}^{-1/2}X_{t,h}S_{t,h}^{-1/2}).
\end{align}

Since $S_{t,h}\succ 0$, both $S_{t,h}^{1/2}$ and $S_{t,h}^{-1/2}$ are well-defined. Therefore,
we have
\begin{align}
\det(S_{t+1,h})
&= \det(S_{t,h}+X_{t,h}) \notag \\
&= \det\!\left(
S_{t,h}^{1/2}
\left(I_d+S_{t,h}^{-1/2}X_{t,h}S_{t,h}^{-1/2}\right)
S_{t,h}^{1/2}
\right) \notag \\
&= \det(S_{t,h}^{1/2})\det(I_d+ S_{t,h}^{-1/2}X_{t,h}S_{t,h}^{-1/2})\det(S_{t,h}^{1/2}) \notag \\
&= \det(S_{t,h})\det(I_d+S_{t,h}^{-1/2}X_{t,h}S_{t,h}^{-1/2}),
\end{align}
where the first equality uses the update rule $S_{t+1,h}=S_{t,h}+X_{t,h}$, and the last two equalities use the multiplicativity of the determinant and
$\det(S_{t,h}^{1/2})^2=\det(S_{t,h})$.

Summing over $t=1,\ldots,T$, we obtain
\begin{align}
\sum_{t=1}^T
\mathsf{tr}\!\left(S_{t,h}^{-1/2}X_{t,h}S_{t,h}^{-1/2}\right)
&\le
2\sum_{t=1}^T
\log\det\!\left(
I_d+S_{t,h}^{-1/2}X_{t,h}S_{t,h}^{-1/2}
\right) \notag \\
&=
2\sum_{t=1}^T
\left[
\log\det(S_{t+1,h})-\log\det(S_{t,h})
\right] \notag \\
&=
2\left[
\log\det(S_{T+1,h})-\log\det(S_{1,h})
\right] \notag \\
&=
2\left[
\log\det(S_{T+1,h})-d\log\lambda
\right],
\end{align}
where the last equality uses $S_{1,h}=\lambda I_d$.
It remains to upper bound $\log\det(S_{T+1,h})$. Let $\nu_1,\ldots,\nu_d>0$ be the eigenvalues of $S_{T+1,h}$. By he arithmetic mean–geometric mean inequality (AM--GM) and the update $S_{T+1,h}=\lambda I_d+\sum_{t=1}^T X_{t,h}$,
\begin{align}
\log\det(S_{T+1,h})
&=
\log\prod_{i=1}^d \nu_i  \le
d\log\!\left(\frac{1}{d}\sum_{i=1}^d \nu_i\right) \notag \\
&=
d\log\!\left(\frac{\mathsf{tr}(S_{T+1,h})}{d}\right) \notag \\
&=
d\log\!\left(
\frac{\mathsf{tr}(\lambda I_d)+\sum_{t=1}^T \mathsf{tr}(X_{t,h})}{d}
\right) \notag \\
&\le
d\log\!\left(\frac{d\lambda+T}{d}\right),
\end{align}
where the last inequality uses $\mathsf{tr}(X_{t,h})\le 1$ for all $t \in T$. Therefore,
\begin{align}
\sum_{t=1}^T
\mathsf{tr}\!\left(S_{t,h}^{-1/2}X_{t,h}S_{t,h}^{-1/2}\right)
&\le
2\left[
d\log\!\left(\frac{d\lambda+T}{d}\right)-d\log\lambda
\right] \notag \\
&=
2d\log\!\left(1+\frac{T}{d\lambda}\right)
=
\widetilde{O}(d).
\end{align}

\end{document}